\documentclass[11pt]{article}
\usepackage{authblk}
\makeatletter
\def\input@path{{sections/}{appendices/}}
\makeatother
\usepackage{amsmath}
\usepackage{amssymb}
\usepackage{amsfonts}
\usepackage{amsthm}
\usepackage{booktabs}
\usepackage{bbm}
\usepackage{graphicx}
\usepackage{subcaption}
\usepackage{adjustbox}
\usepackage{multirow }
\usepackage{natbib}
\bibpunct[, ]{(}{)}{,}{a}{}{,}

\usepackage{xcolor}
\usepackage{float}
\usepackage{hyperref}       % hyperlinks
\hypersetup{
    colorlinks = true,
    linkcolor = {red!50!black},
    anchorcolor = blue,
    citecolor = {blue!50!black},
    filecolor = blue,
    urlcolor = {blue!80!black}
    }
\numberwithin{equation}{section}

\floatstyle{ruled}
\newfloat{algorithm}{htbp}{loa}
\floatname{algorithm}{Algorithm}
\newenvironment{algorithmic}
  {\begin{list}{}{\setlength{\leftmargin}{1.25em}\setlength{\itemsep}{0.4em}%
  \setlength{\parsep}{0pt}\setlength{\topsep}{0.25em}}}
  {\end{list}}
\newcommand{\STATE}{\item[]}

\theoremstyle{plain}
\newtheorem{theorem}{Theorem}
\newtheorem{lemma}{Lemma}
\newtheorem{proposition}{Proposition}
\newtheorem{assumption}{Assumption}
\newtheorem{corollary}{Corollary}

\theoremstyle{definition}
\newtheorem{definition}{Definition}

\theoremstyle{remark}
\newtheorem*{remark}{Remark}

\newcommand{\set}[1]{\left\{#1\right\}}

\newcommand{\linftynorm}[1]{\left\|#1\right\|_{\infty}}

\newcommand{\bracket}[1]{\left( #1 \right)}

\newcommand{\cA}{\mathcal{A}}

\newcommand{\cS}{\mathcal{S}}
\newcommand{\cH}{\mathcal{H}}
\newcommand{\cP}{\mathcal{P}}
\newcommand{\cM}{\mathcal{M}}
\newcommand{\cT}{\mathcal{T}}

\allowdisplaybreaks
\title{Distributionally Robust Average-Reward Reinforcement Learning: Finite-Sample Guarantees under Weak Communication}
\author[1]{Chenyu Lu}
\author[1]{Zijun Chen}
\author[1]{Nian Si}
\affil[1]{The Hong Kong University of Science and Technology}
\date{}

\begin{document}
\maketitle

\begin{abstract}
We study distributionally robust reinforcement learning (DR-RL) in the average-reward setting under weak communication. Our main result provides finite-sample guarantees for estimating the robust optimal average reward and learning a near-optimal policy, covering both SA-rectangular and S-rectangular structures with divergence-based and distance-based uncertainty sets. Specifically, for Kullback--Leibler and $f_k$-divergence balls, we establish explicit radius conditions under which the robust average-reward Bellman equation admits a constant-gain solution, while for total variation and Wasserstein balls, any positive radius suffices without requiring the nominal MDP is weakly communicating. Our algorithm is prior-knowledge-free and achieves sample complexities of $\widetilde O(|\mathcal{S}||\mathcal{A}|p_{\wedge}^{-1}\operatorname{Span}^{2}(u_{\delta}^{\ast})\epsilon^{-2})$ for estimating the robust optimal average reward and $\widetilde O(|\mathcal{S}||\mathcal{A}|p_{\wedge}^{-2}\operatorname{Span}^{2}(u_{\delta}^{\ast})\epsilon^{-2})$ for learning an $\epsilon$-optimal policy. Here, $p_{\wedge}$ is the smallest positive nominal transition probability and $u_{\delta}^{\ast}$ is a robust optimal bias function. We further provide an almost-tight explicit upper bound on \(\operatorname{Span}(u_{\delta}^{\ast})\).  Finally, we validate the predicted $n^{-1/2}$ convergence rate through numerical experiments.
\end{abstract}

\section{Introduction}\label{sec:introduction}

Reinforcement learning (RL)~\citep{sutton2018reinforcement} is a core machine learning framework in which agents learn to maximize long-term rewards by interacting with their environments. RL has achieved substantial success across a wide range of applications, from robotics and control~\citep{Kober2013RLinRoboticsSurvey, Sergey2017DeepRLinRobotics} to recent advances in reasoning with large language models (LLMs)~\citep{wei2022chain, guo2025deepseek}.

Despite these successes, RL faces a fundamental challenge when the transition dynamics encountered during deployment differ from those observed during training.
Such discrepancies can arise from changes in operating conditions, unmodeled disturbances, or adversarial perturbations, potentially degrading the performance of policies learned under a nominal model.

Distributionally robust reinforcement learning (DR-RL) addresses this challenge by explicitly accounting for uncertainty in transition dynamics.
Built on the robust Markov decision process framework~\citep{iyengar2005robust, nilim2005robust, wiesemann2013robust}, DR-RL seeks policies that perform well under the worst-case transition model within a prescribed uncertainty set.
Subsequent work has developed learning algorithms and theoretical guarantees in both model-based~\citep{panaganti_sample_2022, xu2023improved, NEURIPS2024_0346c8a5, shi2024curiouspricedistributionalrobustness} and model-free~\citep{liu22DRQ, wang_finite_2023, wang_sample_2024} settings.

Two widely studied uncertainty structures are SA-rectangularity and S-rectangularity, which specify how the adversary's choices are coupled across states and actions~\citep{le2007robust, wiesemann2013robust}. Under SA-rectangularity, the adversary selects transition distributions independently for each state-action pair, and SA-rectangular models always admit a deterministic and stationary optimal policy. Under $S$-rectangularity, these choices remain independent across states but may be coupled across actions within each state, the optimal policy may be history-dependent and randomized~\citep{grand-clement_beyond_2025}.

Much of the existing finite-sample distributionally robust reinforcement learning theory concerns discounted problems~\citep{li_wang_si_2026_s_rectangular, shi2024curiouspricedistributionalrobustness, NEURIPS2024_0346c8a5, clavier2024towards, panaganti_sample_2022, wang_sample_2024, wang_finite_2023}, in which future rewards are geometrically discounted. For continuing tasks evaluated by long-run performance, the average-reward criterion provides a natural alternative.
Existing sample complexity guarantees for average-reward DR-RL primarily concern SA-rectangular uncertainty under strong ergodic assumptions~\citep{chen2025sample, roch2026modelfreerobustaveragerewardreinforcement, xu2026efficientqlearningactorcriticmethods, xu2025finitesampleanalysispolicyevaluation}.
Despite progress on structural results under weaker assumptions~\citep{wang2025bellmanoptimalityaveragerewardrobust}, finite-sample guarantees under weak communication assumptions remain, to the best of our knowledge, unavailable for both SA- and S-rectangular uncertainty.

In this work, our main contribution are summarized as follows.
\begin{itemize}
    \item \textbf{Structural guarantees under weak communication.} We establish sufficient conditions for the existence of a constant-gain robust average-reward Bellman solution under both SA- and S-rectangular uncertainty. For $\mathrm{KL}$- and $f_k$-divergence balls, nominal weak communication and suitable radius restrictions (Assumption~\ref{ass:divergence-uncertainty}) preserve the nominal transition support, ensuring that every feasible adversarial kernel remains weakly communicating (Proposition~\ref{prop:divergence-constant-gain}). For TV and Wasserstein balls, any positive radius suffices to guarantee a constant-gain Bellman solution without requiring nominal weak communication (Proposition~\ref{prop:metric-constant-gain}). Under the same conditions, a robust average-reward optimal policy can be chosen stationary: deterministic under SA-rectangularity and randomized under S-rectangularity (Corollary~\ref{cor:stationary-average-optimal-policy}).
    \item \textbf{Unified finite-sample guarantees for gain estimation and policy learning.} We develop a prior knowledge free algorithm via discounted reduction (Algorithm~\ref{alg:dr-amdp-reduction}) covering all eight combinations of rectangularity structure and uncertainty families (KL, $f_k$, TV, and Wasserstein). Our analysis establishes Bellman-operator perturbation bounds that hold uniformly over stationary randomized policies and value functions, and controls the discounted-to-average-reward approximation through the robust optimal bias span. We obtain an $\varepsilon$-accurate estimate of the robust optimal average reward and optimal policy using
    \[
    \widetilde O(|\cS||\cA|p_\wedge^{-1}\operatorname{Span}(u_\delta^*)^2
        \varepsilon^{-2}),\quad \widetilde O(|\cS||\cA|p_\wedge^{-2}\operatorname{Span}(u_\delta^*)^2
        \varepsilon^{-2})
    \]
    samples respectively (Theorem~\ref{thm:robust-average-reward-error}), where $p_{\wedge}$ is the smallest positive nominal transition probability and $u_{\delta}^{\ast}$ is a robust optimal bias function, and
$
\operatorname{Span}(u_{\delta}^{\ast})
:= \max_{s\in\mathcal S} u_{\delta}^{\ast}(s)
-\min_{s\in\mathcal S} u_{\delta}^{\ast}(s).
$
    \item \textbf{Explicit bias-span bounds and sharpness.} We derive explicit bounds on $\operatorname{Span}(u_\delta^*)$ for KL, $f_k$, TV, and Wasserstein uncertainty under both rectangularity structures, making our span-dependent policy guarantee explicit (Theorem~\ref{thm:divergence-bias-span-bounds}). The TV and Wasserstein bounds are attained by two-state instances; for KL and $f_k$, our construction shows that the exponent of $p_\wedge^{-1}$ cannot be improved (Theorem~\ref{thm:divergence-bias-span-lower-bound}).
\end{itemize}

The remainder of this paper is organized as follows.
Section~\ref{sec:literature_review} reviews related work, with an emphasis on average-reward DR-RL.
Section~\ref{sec:preliminaries} introduces the notation and problem formulation.
Section~\ref{sec:structural-properties-learning-guarantees} presents our main theoretical results, including the algorithm and its sample complexity analysis.
Section~\ref{sec:numberical_experiments} presents numerical experiments illustrating our theoretical findings.

\section{Literature Review}\label{sec:literature_review}

\textbf{Average-reward RL.}
Average-reward RL provides a classical framework for sequential decision-making under long-run performance criteria~\citep{puterman_markov_2009}.
In the non-robust setting, model-based approaches have been extensively studied from the perspectives of regret minimization and exploration~\citep{JMLR:v11:jaksch10a, pmlr-v80-fruit18a}, as well as through structural complexity measures such as mixing time, minorization time, and bias span~\citep{jin_towards_2021, wang_optimal_2024, zurek_span-based_2024, zurek_plug-approach_2024, zurek2025span}.
Beyond model-based approaches, a growing body of work has investigated model-free methods, including Q-learning and its variants~\citep{jin2020efficientlysolvingmdpsstochastic, NEURIPS2021_096ffc29, NEURIPS2025_d8557df7, chen2026achievingvarepsilon2dependenceaveragereward, zhang2023sharpermodelfreereinforcementlearning, pmlr-v291-agrawal25a, lee2025nearoptimalsamplecomplexitymdps, lee2026learningpolicysingletrajectory}.

\textbf{Distributionally Robust Average-Reward RL:} \citet{grand-clement_beyond_2025} establish conditions under which a stationary optimal policy exists.
On the algorithmic side, \citet{wang_model-free_2023, wang_robust_2023, wang2024robust} develop robust relative value iteration and TD/Q-learning algorithms and establish their convergence in the distributionally robust average-reward setting.~\citet{roch2026distributionallyrobustmarkovgames} study average-reward distributionally robust Markov games, establishing Nash equilibrium existence under irreducibility and weak communication conditions and developing algorithms with convergence guarantees.
Moving beyond convergence analysis, subsequent work provides sample complexity guarantees under SA-rectangular uncertainty, with assumptions ranging from irreducibility and aperiodicity to the unichain condition~\citep{xu2025finitesampleanalysispolicyevaluation, xu2026efficientqlearningactorcriticmethods, roch2026modelfreerobustaveragerewardreinforcement, chen2025sample, yang2026robustaveragerewardmarkovdecision}.
For $s$-rectangular uncertainty, \citet{wang2025bellmanoptimalityaveragerewardrobust} extend the structural theory by providing weak communication conditions ensuring the existence of solutions to the constant-gain robust Bellman equation.
Despite these advances, a gap remains between structural results and finite-sample analysis under weak communication assumptions: to the best of our knowledge, non-asymptotic sample complexity guarantees in this setting remain open for both SA and S-rectangular uncertainty.

We provide a summary of state-of-the-art sample complexity results in the literature in Table~\ref{tab:summary_rate}.
\begin{table}[htbp]
  \caption{Summary of DR-AMDP sample complexity bounds under a generative model. Here, $\kappa\in(0,1)$  denotes the robust Bellman contraction factor under the specific seminorm defined in \citet{xu2025finitesampleanalysispolicyevaluation},  $\cH$ is the maximum of bias-span over all optimal policies and corresponding worst-case kernel which is at least $\operatorname{Span}(u_{\delta}^*)$. $t_{\mathrm{mix}}$ denotes the mixing-time bound used in the corresponding cited result.}
  \label{tab:summary_rate}
  \centering
  \begin{adjustbox}{max width=\linewidth}
  \begin{tabular}{lllll}
    \toprule
    Setting & Uncertainty Type & Sample Complexity & Assumption & Origin \\
    \midrule

    \multirow{6}{*}{SA-rectangular}
    & \text{TV, Wasserstein}
    & $\widetilde O\bracket{|\cS||\cA|t_{\mathrm{mix}}^2(1-\kappa)^{-2}\varepsilon^{-2}}$
    & Irreducible and Aperiodic
    & \citet{xu2025finitesampleanalysispolicyevaluation} \\
    
    & \text{TV, Wasserstein}
    & $\widetilde O\bracket{|\cS|^2|\cA|^2\varepsilon^{-2}}$
    & Irreducible and Aperiodic
    & \citet{xu2026efficientqlearningactorcriticmethods} \\

    & $\mathrm{KL}$, $\chi^2$
    & $\widetilde O\bracket{|\cS||\cA|\cH^2\varepsilon^{-(2+o(1))}}$
    & Irreducible
    & \citet{roch2026modelfreerobustaveragerewardreinforcement} \\

    & KL, $f_k$
    & $\widetilde O(|\cS||\cA|t_\mathrm{mix}^2p_{\wedge}^{-1}\varepsilon^{-2})$
    & Uniformly ergodic
    & \citet{chen2025sample} \\

    & TV
    & $\widetilde O(|\cS||\cA| (\operatorname{Span}(u_{\delta}^*)+\delta \operatorname{Span}(u_{\delta}^*)^2)\varepsilon^{-2})$
    & Unichain
    & \citet{yang2026robustaveragerewardmarkovdecision} \\

    & KL, $f_k$, TV, Wasserstein
    & $\widetilde O(|\cS||\cA|p_{\wedge}^{-2}\operatorname{Span}(u_{\delta}^*)^2\varepsilon^{-2})$
    & Weakly communicating
    & \textbf{This work} \\

    \midrule

    \multirow{1}{*}{S-rectangular}
    & KL, $f_k$, TV, Wasserstein
    & $\widetilde O\bracket{|\cS||\cA|p_{\wedge}^{-2}\operatorname{Span}(u_{\delta}^*)^2\varepsilon^{-2}}$
    & Weakly communicating
    & \textbf{This work} \\

    \bottomrule
  \end{tabular}
  \end{adjustbox}
\end{table}

\section{Preliminaries and Problem Formulation}\label{sec:preliminaries}

\subsection{Markov Decision Process}
\label{subsec:controlled-dynamics}

We briefly review the classical finite tabular MDP models and
introduce the notation used throughout the paper. Let $\Delta(\mathcal{S})$, $\Delta(\mathcal{A})$ denote the probability simplex over the finite state space $\mathcal{S}$ and action space $\mathcal{A}$ respectively. $P:=\{p_{s,a}\in\Delta(\mathcal S):(s,a)\in\mathcal S\times\mathcal A\}$ is the controlled transition kernel, where $p_{s,a}(s')$ denotes the probability of transitioning to state $s'$ after action $a$ is selected in state $s$. $r:\cS\times \cA\times \cS\to [0,1]$ is the reward function. A finite discounted MDP (DMDP) is specified by the tuple $\cM_{\gamma}=(\cS,\cA,P,r,\gamma)$ where $\gamma\in (0,1)$ is the discount factor, and an average-reward MDP
(AMDP) is specified as $\cM=(\mathcal S,\mathcal A,P,r)$.

Let $\Pi_{\mathrm{HD}}$ denote the class of history-dependent randomized policies $\pi=(\pi_t)_{t\ge0}$, where $\pi_t(\cdot\mid h_t)\in\Delta(\mathcal A)$ and $h_t=(s_0,a_0,\ldots,s_t)$ is the observed history. Stationary randomized policies depend only on the current state and form $\Pi_{\mathrm{SR}}:=\{\pi:\mathcal S\to\Delta(\mathcal A)\}$. Stationary deterministic policies select a single action in each state and form $\Pi_{\mathrm{SD}}:=\{\pi:\mathcal S\to\mathcal A\}$. For ordinary finite DMDPs and AMDPs, it suffices to consider stationary deterministic policies~\citep[Chapters~6 and~8]{puterman_markov_2009}. Identifying deterministic policies with point-mass action distributions gives $\Pi_{\mathrm{SD}}\subseteq\Pi_{\mathrm{SR}}\subseteq\Pi_{\mathrm{HD}}$. For $\pi\in\Pi_{\mathrm{SR}}$, the induced nominal transition kernel is $P_\pi(x\mid s):=\sum_{a\in\mathcal A}\pi(a\mid s)p_{s,a}(x)$.

\textbf{Discounted criterion.} For a given discounted MDP (DMDP) instance $\cM_{\gamma}=(\cS, \cA, P, r, \gamma)$ and policy $\pi\in \Pi_{\mathrm{SD}}$, the \textit{discounted value function} $V_\gamma^{\pi}(s):={\operatorname E}_P^{\pi}\left[\sum_{t=0}^{\infty}\gamma^t r(S_t,A_t,S_{t+1})\mid S_0=s\right]$. An optimal policy $\pi^*\in \Pi_{\mathrm{SD}}$ achieves the optimal value $V_{\gamma}^*(s):=\max_{\pi \in \Pi_{\mathrm{SD}}}V_{\gamma}^{\pi}(s)$.

\textbf{Average-reward criterion.} An average-reward MDP (AMDP) instance $\cM=(\cS, \cA, P, r)$ and policy $\pi\in \Pi_{\mathrm{SD}}$, the \textit{long-term average-reward function} (gain) $g^{\pi}:\cS\to \mathbb{R}$ is defined as $g^{\pi}(s):=\limsup_{T\to\infty}T^{-1}{\operatorname E}_P^{\pi}[\sum_{t=0}^{T-1}r(S_t,A_t,S_{t+1})\mid S_0=s]$ and optimal gain $g^*(s):=\sup_{\pi\in \Pi_{\mathrm{SD}}}g^{\pi}(s)$ is a constant function when $P$ is weakly communicating (Definition~\ref{def:weakly-communicating-mdp}).

For average-reward MDPs, the following weak communication condition ensures
a state-independent optimal gain without requiring every stationary policy
to induce an ergodic chain~\citep{puterman_markov_2009}.
\begin{definition}[Weakly Communicating]
\label{def:weakly-communicating-mdp}
An MDP is weakly communicating if there exists a partition $\mathcal S=C\cup T$ such that every pair of states in $C$ communicates under some stationary policy $\pi\in\Pi_{\mathrm{SR}}$, and every state in $T$ is transient under every stationary policy $\pi\in\Pi_{\mathrm{SR}}$.
\end{definition}

% The dependence on $s$ allows the long-run gain to vary with the initial state. We return to the conditions for a common optimal gain after introducing transition uncertainty.

\subsection{Robust MDPs and Rectangularity}
\label{subsec:robust-mdps-and-rectangularity}

We study robust MDPs under both SA- and $S$-rectangular
uncertainty, where an adversary perturbs the nominal transition
probabilities within a prescribed uncertainty set. We consider
divergence-based sets defined by Kullback--Leibler (KL) and
$f_k$ divergences, and distance-based sets defined by total
variation (TV) and Wasserstein distances. We use $U\in \{D_{\mathrm{KL}}, D_{f_k}, d_{\mathrm{TV}}, W_{\ell}\}$ to specify the uncertainty measure and $\delta$ for its radius. For
$q,p\in\Delta(\mathcal S)$ with $q\ll p$, define
\[
D_{\mathrm{KL}}(q\|p)
:= \sum_{s\in\operatorname{supp}(p)}
q(s)\log\frac{q(s)}{p(s)},
\qquad
D_{f_k}(q\|p)
:= \sum_{s\in\operatorname{supp}(p)}
p(s)f_k\!\left(\frac{q(s)}{p(s)}\right),
\]
where $0\log 0:=0$ and
$f_k(t):=(t^k-kt+k-1)/(k(k-1))$ for $t\ge0$ and
$k\in(1,\infty)$. For arbitrary $q,p\in\Delta(\mathcal S)$,
define 
\[
d_{\mathrm{TV}}(q,p):=\frac12\|q-p\|_1,
\qquad
W_\ell(q,p)
:= \left(
\inf_{\xi\in\Gamma(q,p)}
\sum_{x,y\in\mathcal S}\rho(x,y)^\ell\xi(x,y)
\right)^{1/\ell},
\]
where $\Gamma(q,p)$ denotes the set of couplings of $q$ and $p$,
$\rho$ is a metric on $\mathcal S$, and $\ell\in[1,\infty)$.
Unlike the divergence-based sets considered here, the
distance-based sets do not require absolute continuity with
respect to the nominal transition distribution.
When $U$ is a divergence, $U(q,p)=D(q\|p)$.

The statistical complexity of policy learning in robust MDPs are primarily studied under SA- and S-rectangular uncertainty sets. SA-rectangularity allows the adversary to choose each state-action pair’s transition distribution in-
dependently. Under S-rectangularity, the constraints separate only across states, while the transition
distributions for different actions at the same state may be jointly constrained. From this point forward, we introduce SA- and S-rectangular respectively.
\begin{definition}[SA-rectangularity]
    The uncertainty set $\cP$ is called an SA-rectangular set if it satisfies $\cP =\prod_{(s,a)\in\mathcal S\times\mathcal A}\mathcal P_{s,a}(U,\delta)$ where $\cP_{s,a}\subseteq \Delta(\cS)$ for all $(s,a)\in \cS\times \cA$.
\end{definition}
Specifically, for divergence-based uncertainty set $U\in\{D_{\mathrm{KL}},D_{f_k}\}$, we set
\[
\mathcal P_{s,a}(U,\delta)
:=\left\{q_{s,a}\in\Delta(\mathcal S):q_{s,a}\ll p_{s,a},\ U(q_{s,a}\|p_{s,a})\le\delta\right\}.
\]
For distance-based uncertainty $U\in\{d_{\mathrm{TV}},W_\ell\}$, $\ll$ is not required and
\[
\mathcal P_{s,a}(U,\delta)
:=\left\{q_{s,a}\in\Delta(\mathcal S):U(q_{s,a},p_{s,a})\le\delta\right\}.
\]

In addition,~\citet{wiesemann2013robust} propose the extensive version S-rectangular set which is detained in Definition~\ref{defn:s_rectangularity}.

\begin{definition}[S-rectangularity]\label{defn:s_rectangularity}
    The uncertainty set $\cP$ is called an $S$-rectangular set if it satisfies $\cP = \prod_{s\in \cS}\cP_{s}(U, \delta)$ where $\cP_{s}\subseteq \Delta(\cS)^{|\cA|}$ and $\Delta(\cS)^{|\cA|}:=\{q_s=(q_{s,a})_{a\in \cA}|q_{s,a}\in \Delta(\cS),\text{ for all } a\in \cA\}$.
\end{definition}
where $\mathcal P_s(U,\delta)$ jointly constrains the collection $(q_{s,a})_{a\in\mathcal A}$ at state $s$. In this paper, we consider the sum-form of S-rectangularity, which is commonly studied in the literature; see, for example, \citet{yang2022toward,li_wang_si_2026_s_rectangular}. In particular, for divergence-based uncertainty, $U\in\{D_{\mathrm{KL}},D_{f_k}\}$, we set
\[
\cP_{s}(U, \delta) := \set{q_{s}=(q_{s,a})_{a\in \cA}:q_{s,a}\ll p_{s,a}, \sum_{a\in \cA}U(q_{s,a}\| p_{s,a})\leq |\cA|\delta}
\]
For distance-based uncertainty $\mathcal P_s(U,\delta)$ where $U\in \{d_{\mathrm{TV}}, W_{\ell}\}$, the budget constraint is
\[
\sum_{a\in\mathcal A}d_{\mathrm{TV}}(q_{s,a},p_{s,a})\le|\mathcal A|\delta,\qquad \sum_{a\in\mathcal A}W_\ell(q_{s,a},p_{s,a})^\ell\le|\mathcal A|\delta^\ell.
\]

For SA-rectangularity, take $\mathcal P_s=\prod_{a\in\mathcal A}\mathcal P_{s,a}$ in the following definition.
\begin{definition}
\label{defn:bellman-operators}
    The discounted and average-reward optimal Bellman operators $\cT_{\gamma}^{*}, \cT^{*}:\mathbb{R}^{|\cS|}\to \mathbb{R}^{|\cS|}$ are defined by
    \begin{align}
        \mathcal T_\gamma^*(v)(s):=&\max_{\phi\in\Delta(\mathcal A)}\inf_{q_s\in\mathcal P_s}\sum_{a\in\mathcal A}\sum_{s'\in\mathcal S}\phi(a)q_{s,a}(s')\bigl(r(s,a,s')+\gamma v(s')\bigr), \label{eq:discounted-bellman-operator}\\
        \mathcal T^*(v)(s):=&\max_{\phi\in\Delta(\mathcal A)}\inf_{q_s\in\mathcal P_s}
\sum_{a\in\mathcal A}\sum_{s'\in\mathcal S}\phi(a)q_{s,a}(s')\bigl(r(s,a,s')+v(s')\bigr).\label{eq:average-bellman-operator}
    \end{align}
Under SA-rectangularity, the adversary minimizes each action independently, so these operators reduce to $\mathcal T_\gamma^*(v)(s)=\max_{a\in\mathcal A}\inf_{q_{s,a}\in\mathcal P_{s,a}}\sum_{s'\in\mathcal S}q_{s,a}(s')\bigl(r(s,a,s')+\gamma v(s')\bigr)$ and $\mathcal T^*(v)(s)=\max_{a\in\mathcal A}\inf_{q_{s,a}\in\mathcal P_{s,a}}\sum_{s'\in\mathcal S}q_{s,a}(s')\bigl(r(s,a,s')+v(s')\bigr)$~\citep{iyengar2005robust,nilim2005robust}. Under S-rectangularity, randomization may be necessary~\citep{wiesemann2013robust,wang2025bellmanoptimalityaveragerewardrobust}.
\end{definition}

It is well known that, under either SA- or S-rectangularity, the discounted Bellman equation associated with $\mathcal T_\gamma^*$ in \eqref{eq:discounted-bellman-operator} admits a unique fixed point $v_\gamma^*$, i.e., $v_\gamma^*=\mathcal T_\gamma^*(v_\gamma^*)$ which equals the robust optimal discounted value~\citep{iyengar2005robust,wiesemann2013robust}.

\subsection{Long-Run Structure: Gain and Bias Span}
\label{subsec:long-run-structure}

We now turn to long-run performance under transition uncertainty. In the
average-reward setting, the controller uses a history-dependent policy,
while the adversary may choose admissible transition kernels based on the
observed history. For a fixed $\pi\in\Pi_{\mathrm{HD}}$, let
$\mathbf Q=(Q_t)_{t\ge0}$ denote an admissible sequence of transition
kernels, where $Q_t\in\mathcal P$ may depend on the history up to time $t$.
The robust average reward of $\pi$ is defined by
\[
g_\delta^\pi(s)
:=\inf_{\mathbf Q}\limsup_{T\to\infty}
{\operatorname E}_{\mathbf Q}^{\pi}
\left[\frac{1}{T}\sum_{t=0}^{T-1}r(S_t,A_t,S_{t+1})\,\middle|\,S_0=s\right],
\qquad s\in\mathcal S.
\]
The robust optimal gain is then
$g_\delta^*(s):=\sup_{\pi\in\Pi_{\mathrm{HD}}}g_\delta^\pi(s)$.
When $g_\delta^*(s)$ is independent of the initial state, we denote the
common value by $g_\delta^*$.

\citet{wang2025bellmanoptimalityaveragerewardrobust} establish the existence
of a constant-gain robust Bellman solution under adversary-side weak
communication of the uncertainty set, together with suitable compactness
and minimax conditions. For divergence-based uncertainty sets, this
requirement on the admissible kernels $Q\in\mathcal P$ does not directly
translate into conditions on the nominal kernel $P$ and radius $\delta$.
We therefore seek such conditions for the same constant-gain Bellman
equation, written with our average-reward optimal Bellman operator
$\mathcal T^*$ as
\begin{equation}
\label{eq:constant-gain-bellman}
    \mathcal{T}^*(u_\delta^*) = u_\delta^* + g\mathbf{1}.
\end{equation}
We assume $\operatorname{Span}(u_\delta^*)\ge1$ in the finite-sample analysis, where $\operatorname{Span}(v):=\max_{s\in\mathcal S}v(s)-\min_{s\in\mathcal S}v(s)$ for all $v\in \mathbb{R}^{|\cS|}$. Note that the solution to \eqref{eq:constant-gain-bellman} need not be unique. However, our main complexity result (Theorem \ref{thm:robust-average-reward-error} below) holds for any $u_\delta^*$ satisfying \eqref{eq:constant-gain-bellman}.

If the robust Bellman equation \eqref{eq:constant-gain-bellman} admits a solution $(u_\delta^*,g)$,
then $g_\delta^*(s)=g$ for every $s\in\mathcal S$, hence $g=g_\delta^*$.
By \citet[Theorem~1]{wang2025bellmanoptimalityaveragerewardrobust}, this property
holds for $r(s,a)$; the same verification argument extends directly to our
setting with bounded transition-dependent rewards $r(s,a,s')$.
% This identity connects the Bellman gain estimated in
% Theorem~\ref{thm:robust-average-reward-error}
% to the original long-run average-reward objective.

% Without a solution to \eqref{eq:constant-gain-bellman}, there is
% no common scalar Bellman target for this comparison, and the subsequent
% finite-sample analysis cannot be established through our approach.

\section{Structural Properties and Learning Guarantees}
\label{sec:structural-properties-learning-guarantees}

\newcommand{\BasicPropertiesReference}{Appendix~\ref{sec:basic-properties}}
\newcommand{\ConstantGainUniquenessReference}{Lemma~\ref{lem:constant-gain-uniqueness}}
\newcommand{\MainResultsReference}{Appendix~\ref{sec:main-results}}

This section presents our main structural and learning results. We first establish conditions for the robust average-reward Bellman equation, and then derive the learning algorithm and its finite-sample guarantees.

A key step in our finite-sample analysis is to bound the span of the optimal discounted value uniformly as $\gamma$ approaches one. \citet{chen2025sample} obtain such control through uniform ergodicity and mixing-time bounds. Our result instead uses a weaker sufficient condition: the existence of a constant-gain solution to the optimal average-reward Bellman equation. Specifically, a solution $(u_\delta^*,g_\delta^*)$ satisfying $\mathcal T^*(u_\delta^*)=u_\delta^*+g_\delta^*\mathbf 1$ yields $\operatorname{Span}(v_\gamma^*)\leq 2\operatorname{Span}(u_\delta^*)$ for every $\gamma\in(0,1)$ (Lemma~\ref{lem:discounted-value-span-bound}). This comparison, building on \citet{wang2025bellmanoptimalityaveragerewardrobust}, allows us to control the discounted value span directly through the robust optimal bias span and express the resulting estimation bounds in terms of this quantity. We therefore first establish conditions ensuring a constant-gain Bellman solution for the uncertainty sets considered here, and then use this span comparison in the finite-sample analysis of Section~\ref{subsec:robust_average_rewrad_learning}.

\subsection{Bellman optimality existence}
\label{subsec:assumptions-and-properties}

In addressing different construction sources for uncertainty sets, this section presents two conclusions: for divergence-based uncertainty sets $D_\mathrm{KL}$ and $D_{f_k}$, when nominal kernel $P$ satisfies weakly communicating and the uncertainty radius condition, then all kernels $Q\in \mathcal{P}$ within the SA-rect and S-rect sets possess the weakly communicating property; this guarantees the existence of an optimal Bellman equation and ensures the existence of a constant optimal robust gain. Meanwhile, since distance-based uncertainty sets $D_{\mathrm{TV}}$ and $W_{\ell}$ do not require absolute continuity, a positive uncertainty radius suffices to ensure the existence of a constant optimal gain function as a solution to the Bellman equation.

% \begin{assumption}
% \label{ass:divergence-uncertainty}
% The nominal MDP $M=(\mathcal{S},\mathcal{A},P,r)$ is weakly communicating.
% The uncertainty set is based on either KL divergence or $f_k$ divergence and
% has either SA-rectangular or $S$-rectangular structure. For the $f_k$
% divergence, let $k\in(1,\infty)$. In all cases, assume that the radius satisfies
% \[
% \delta<
% \begin{cases}
% \displaystyle
% \log\!\left(\frac{1}{1-p_{\wedge}}\right),
% & \mathcal P=\prod_{(s,a)\in\mathcal S\times\mathcal A}\mathcal P_{s,a}(D_{\mathrm{KL}},\delta),\\[1.2ex]
% \displaystyle
% \frac{p_{\wedge}}{k},
% & \mathcal P=\prod_{(s,a)\in\mathcal S\times\mathcal A}\mathcal P_{s,a}(D_{f_k},\delta),\\[1.2ex]
% \displaystyle
% \frac{1}{|\mathcal{A}|}
% \log\!\left(\frac{1}{1-p_{\wedge}}\right),
% & \mathcal P=\prod_{s\in\mathcal S}\mathcal P_s(D_{\mathrm{KL}},\delta),\\[1.2ex]
% \displaystyle
% \frac{p_{\wedge}}{|\mathcal{A}|k},
% & \mathcal P=\prod_{s\in\mathcal S}\mathcal P_s(D_{f_k},\delta).
% \end{cases}
% \]
% \end{assumption}

\begin{definition}
\label{def:minimum-support}
Define the minimum support as:
\begin{equation}
\label{eq:minimum-support}
    p_\wedge := \min_{(s,a,s')\in\cS\times\cA\times\cS}\set{p_{s,a}(s'):p_{s,a}(s')>0}.
\end{equation}
\end{definition}

\begin{assumption}
\label{ass:divergence-uncertainty}
The nominal MDP $M=(\mathcal S,\mathcal A,P,r)$ is weakly
communicating. Under either SA- or $S$-rectangularity, for $U\in \{D_{\mathrm{KL}}, D_{f_k}\}$, suppose the uncertainty radius satisfies
\[
\delta <
\begin{cases}
\frac{1}{\zeta}\log\bracket{\frac{1}{1-p_{\wedge}}}, & U=D_{\mathrm{KL}},\\[0.5ex]
\frac{p_\wedge}{k\zeta}, & U=D_{f_k}.
\end{cases} \;\text{where}\;\;\zeta :=
\begin{cases}
1, & \text{SA-rectangularity},\\
|\mathcal A|, & \text{$S$-rectangularity}.
\end{cases}
\]
\end{assumption}

Inspired by \citeauthor{wang2025bellmanoptimalityaveragerewardrobust}'s \citeyearpar{wang2025bellmanoptimalityaveragerewardrobust} adversarial weak-communication framework, we derive explicit conditions on the nominal kernel $P$ and radius $\delta$ that ensure every $Q\in\mathcal P$ is weakly communicating. We impose radius restrictions for KL- and $f_k$-divergence balls that preserve nominal transition supports, transferring nominal weak communication to every $Q\in\mathcal P$. Together with compactness and convexity, this establishes Bellman existence under both rectangularity structures through explicit nominal-model and radius conditions.

\begin{proposition}
\label{prop:divergence-constant-gain}
Under Assumption~\ref{ass:divergence-uncertainty}, the optimal average-reward
Bellman equation~\eqref{eq:constant-gain-bellman}\ admits a solution $(u_\delta^*,g_\delta^*)$ with a unique
constant gain:
\[
\mathcal T^*(u_\delta^*)=u_\delta^*+g_\delta^*\mathbf 1.
\]
\end{proposition}

\begin{assumption}
\label{ass:metric-uncertainty}
For distance-based uncertainty where $U\in \{d_{\mathrm{TV}}, W_\ell\}$, under both SA- and S-rectangularity suppose the uncertainty radius satisfies $\delta > 0$.
\end{assumption}

Under Assumption~\ref{ass:metric-uncertainty}, following \citet[Corollary~7.1]{wang2025bellmanoptimalityaveragerewardrobust}, any radius $\delta>0$ admits a feasible full-support perturbation. This yields a uniform bound on $\operatorname{Span}(v_\gamma^*)$ and, through a limit of normalized discounted values, a constant-gain Bellman solution for transition-dependent rewards under both rectangularity structures.

\begin{proposition}
\label{prop:metric-constant-gain}
Under Assumption~\ref{ass:metric-uncertainty}, the optimal average-reward
Bellman equation~\eqref{eq:constant-gain-bellman}\ admits a solution $(u_\delta^*,g_\delta^*)$ with a unique
constant gain:
\[
\mathcal T^*(u_\delta^*)=u_\delta^*+g_\delta^*\mathbf 1.
\]
\end{proposition}

\begin{corollary}
\label{cor:stationary-average-optimal-policy}
Under Assumption~\ref{ass:divergence-uncertainty} or
Assumption~\ref{ass:metric-uncertainty}, there exists a stationary policy
$\pi^*$ attaining the robust optimal gain $g_\delta^*$.
In particular, $\pi^*\in\Pi_{\mathrm{SD}}$ under
SA-rectangularity, while $\pi^*\in\Pi_{\mathrm{SR}}$ under
$S$-rectangularity.
\end{corollary}

\begin{remark}
Propositions~\ref{prop:divergence-constant-gain}--\ref{prop:metric-constant-gain} and Corollary~\ref{cor:stationary-average-optimal-policy} strengthen the structural theory in two directions. First, compared with \citet{wang2025bellmanoptimalityaveragerewardrobust}, we give explicit uncertainty-set-specific conditions ensuring a constant-gain Bellman solution. Second, under the same conditions, we establish stationary average-reward optimal policies for both SA- and S-rectangular uncertainty, providing explicit sufficient conditions for when S-rectangular average-reward optimality can be achieved by a stationary policy~\citep{grand-clement_beyond_2025}.
\end{remark}

The proofs are given in \BasicPropertiesReference. These results yield uniform control of discounted value spans through the robust optimal bias span, providing the basis for the finite-sample guarantees in the next section without assuming uniform ergodicity.

\subsection{Robust Average-Reward Learning}
\label{subsec:robust_average_rewrad_learning}
With the constant-gain Bellman solution established, we now study how to
estimate the robust optimal average reward and learn a near-optimal policy
from samples of the nominal transition kernel. Following the reduction
approach of \citet{chen2025sample},
Algorithm~\ref{alg:dr-amdp-reduction} constructs an empirical robust
discounted MDP with $\gamma_n=1-n^{-1/2}$ and solves its Bellman equation.
It returns the empirical optimal policy $\widehat\pi^*$ and uses
$(1-\gamma_n)\widehat v_{\gamma_n}^*$ to estimate $g_\delta^*\mathbf1$.
We then establish finite-sample bounds for both the estimation error and
the average-reward suboptimality of the returned policy under the true
uncertainty set $\mathcal P$.

\textbf{Generative model.} We assume the accessability to a simulator, i.e., a \textit{generative model} that allows the agent to sample independently from the nominal transition kernel $p_{s,a}$ for all $(s,a)\in \cS\times \cA$. Given a $n$, we sample i.i.d. $\{S_{s,a}^{(1)}, \ldots, S_{s,a}^{(n)}\}$ from $p_{s,a}$ and construct the empirical kernel transition $\widehat{P}=\{\widehat p_{s,a},(s,a)\in \cS\times \cA\}$ where
\begin{equation}
\label{equ:empirical_kernel}
    \widehat{p}_{s,a}(s'):=\frac{1}{n}\sum_{i=1}^{n}\mathbbm{1}\set{S_{s,a}^{(i)}=s'}
\end{equation}

\begin{algorithm}[!htbp]
\caption{Distributionally Robust Average-Reward Learning (DR-ARL)}
\label{alg:dr-amdp-reduction}
\begin{algorithmic}
\STATE \textbf{Input:} sample size $n\ge1$,
$U\in \{D_\mathrm{KL}, D_{f_k}, d_{\mathrm{TV}}, W_{\ell}\}$, uncertainty radius $\delta$.
\STATE {Set $\gamma_n:=1-n^{-1/2}$, for every $(s,a)\in\mathcal S\times\mathcal A$, construct an $n$-sample empirical transition kernel $\widehat p_{s,a}$ according to~\eqref{equ:empirical_kernel}.}
\STATE \textit{SA-rectangular:} For each $(s,a)$, set $\widehat{\mathcal P}=\prod_{(s,a)\in\mathcal S\times\mathcal A}\widehat{\mathcal P}_{s,a}(U,\delta)$ where $\widehat{\mathcal P}_{s,a}(U,\delta):=\{q_{s,a}\in\Delta(\mathcal S):U(q_{s,a},\widehat p_{s,a})\le\delta\}$,
% with $q_{s,a}\ll\widehat p_{s,a}$ when $U\in\{D_{\mathrm{KL}},D_{f_k}\}$.
\STATE Solve the empirical optimal discounted Bellman equation
\begin{equation}
\label{equ:empirical_robust_discounted_bellman_equation_sa_rect}
    \widehat v_{\gamma_n}^*(s)=\max_{a\in\mathcal A}\inf_{q_{s,a}\in\widehat{\mathcal P}_{s,a}}\sum_{s'\in\mathcal S}q_{s,a}(s')\bigl(r(s,a,s')+\gamma_n\widehat v_{\gamma_n}^*(s')\bigr),
\end{equation}
and for each $s\in\mathcal S$, let $\widehat\pi^*(s)$ be any action that maximize the right-hand side of~\eqref{equ:empirical_robust_discounted_bellman_equation_sa_rect}.
\STATE \textit{$S$-rectangular:} Construct $\widehat{\mathcal P}=\prod_{s\in\mathcal S}\widehat{\mathcal P}_s$ where $\widehat{\mathcal P}_s$ is jointly
around $(\widehat p_{s,a})_{a\in\mathcal A}$ with radius $\delta$.
\STATE Solve the empirical optimal discounted Bellman equation
\begin{equation}
\label{equ:empirial_robust_discounted_bellman_equation_s_rect}
    \widehat v_{\gamma_n}^*(s)=\max_{\phi\in\Delta(\mathcal A)}\inf_{q_s\in\widehat{\mathcal P}_s}\sum_{a\in\mathcal A}\sum_{s'\in\mathcal S}\phi(a)q_{s,a}(s')\bigl(r(s,a,s')+\gamma_n\widehat v_{\gamma_n}^*(s')\bigr),
\end{equation}
and for each $s\in\mathcal S$, let $\widehat\pi^*(\cdot|s)$ be any action distribution that maximize the right-hand side of~\eqref{equ:empirial_robust_discounted_bellman_equation_s_rect}.
\STATE \textbf{return} $\widehat\pi^*$ and $\widehat v_{\gamma_n}^*/\sqrt n$.
\end{algorithmic}
\end{algorithm}

We now state the gain-estimation and policy guarantees for
Algorithm~\ref{alg:dr-amdp-reduction}.
\begin{theorem}
\label{thm:robust-average-reward-error}
Suppose  either Assumption~\ref{ass:divergence-uncertainty} or
Assumption~\ref{ass:metric-uncertainty} holds. Let $(u_\delta^*,g_\delta^*)$ be
any solution of the corresponding optimal average-reward Bellman equation~\eqref{eq:constant-gain-bellman}. Fix $\beta\in(0,1)$. If
\[
n
\ge
\frac{16}{p_{\wedge}}
\log\left(\frac{2|\mathcal{S}|^2|\mathcal{A}|}{\beta}\right),
\]
then the output of
Algorithm~\ref{alg:dr-amdp-reduction}
satisfies both of the following inequalities simultaneously with probability at least
$1-\beta$:
\begin{align}
    \left\|\frac{\widehat{v}_{\gamma_n}^{*}}{\sqrt{n}} - g_{\delta}^*\mathbf{1}\right\|_{\infty} \leq&\frac{16\operatorname{Span}(u_{\delta}^*)}{\sqrt{np_\wedge}}\sqrt{\log\bracket{\frac{2|\cS|^2|\cA|}{\beta}}}\label{equ:value_error}\\
    \linftynorm{g_{\delta}^*\mathbf{1} - g_{\delta}^{\widehat{\pi}^*}}\leq&\frac{151\operatorname{Span}(u_{\delta}^*)}{\sqrt{n}p_\wedge}\log\bracket{\frac{2|\cS|^2|\cA|}{\beta}}\label{equ:policy_error}
\end{align}
where $g_\delta^{\widehat\pi^*}(s)$ is the robust average reward of the
returned policy from initial state $s$, evaluated over the true uncertainty
set $\mathcal P$ according to
Section~\ref{subsec:long-run-structure}.
\end{theorem}

We develop a unified Bellman-operator perturbation analysis for discounted
reduction, establishing finite-sample guarantees for both gain estimation
and policy learning. Across all eight uncertainty configurations, our
algorithm attains $\widetilde O(n^{-1/2})$ error bounds without prior knowledge
of $p_\wedge$ or the robust optimal bias span.

\begin{remark}
Our upper bounds hold for any $u_\delta^*$ satisfying the robust Bellman equation~\eqref{eq:constant-gain-bellman}. While \citet[Remark~3.3]{roch2026modelfreerobustaveragerewardreinforcement} claim that their results can be extended to weakly communicating settings, their complexity measure depends on the maximum of $\operatorname{Span}(u_\delta^*)$ over all solutions $u_\delta^*$ to the robust Bellman equation.
\end{remark}
% \textbf{Proof sketch of Theorem~\ref{thm:robust-average-reward-error}.}
% For gain estimation, we bound the sampling and discounting errors by
% $\|\widehat v_{\gamma_n}^*/\sqrt n-g_\delta^*\mathbf1\|_\infty
% \le n^{-1/2}\bigl(\|\widehat v_{\gamma_n}^*-v_{\gamma_n}^*\|_\infty
% +\operatorname{Span}(v_{\gamma_n}^*)\bigr)$.
% With $1-\gamma_n=n^{-1/2}$, Bellman contraction and concentration give
% $\|\widehat v_{\gamma_n}^*-v_{\gamma_n}^*\|_\infty
% =\widetilde O(\operatorname{Span}(u_\delta^*)/\sqrt{p_\wedge})$.
% Together with $\operatorname{Span}(v_{\gamma_n}^*)\le
% 2\operatorname{Span}(u_\delta^*)$, this yields the rate in~\eqref{equ:value_error}.

% For policy learning, empirical Bellman optimality and summation along
% trajectories give
% $\|g_\delta^*\mathbf1-g_\delta^{\widehat\pi^*}\|_\infty
% \le\|g_\delta^*\mathbf1-\widehat v_{\gamma_n}^*/\sqrt n\|_\infty
% +\widetilde O\bigl((1+\operatorname{Span}(\widehat v_{\gamma_n}^*))/\sqrt{np_\wedge}\bigr)$.
% Since $\operatorname{Span}(\widehat v_{\gamma_n}^*)
% \le\operatorname{Span}(v_{\gamma_n}^*)
% +2\|\widehat v_{\gamma_n}^*-v_{\gamma_n}^*\|_\infty
% =\widetilde O(\operatorname{Span}(u_\delta^*)/\sqrt{p_\wedge})$,
% the additional term is
% $\widetilde O(\operatorname{Span}(u_\delta^*)/(p_\wedge\sqrt n))$,
% giving the rate in~\eqref{equ:policy_error}.

The next theorem makes the bias-span dependence in
Theorem~\ref{thm:robust-average-reward-error} explicit for all four uncertainty
sets under both rectangularity structures.

\begin{theorem}
\label{thm:divergence-bias-span-bounds}
Suppose Assumption~\ref{ass:divergence-uncertainty} or
Assumption~\ref{ass:metric-uncertainty} holds. Let \(c=\max_{x,y\in\mathcal S}\rho(x,y)\). If
\[
\delta\le
\begin{cases}
\displaystyle\frac{p_{\wedge}}{8\zeta}, & U=D_{\mathrm{KL}},\\[1ex]
\displaystyle\frac{p_{\wedge}}{\zeta\max\{8,4k\}}, & U=D_{f_k},\\[1ex]
1, & U=d_{\mathrm{TV}},\\
c, & U=W_\ell,
\end{cases}
\quad\text{where}\quad
\zeta:=
\begin{cases}
1, & \text{SA-rectangularity},\\
|\mathcal A|, & \text{$S$-rectangularity},
\end{cases}
\]
then the corresponding Bellman equation~\eqref{eq:constant-gain-bellman}
admits a solution $(u_\delta^*,g_\delta^*)$ satisfying
\[
\operatorname{Span}(u_\delta^*)\le
\begin{cases}
\displaystyle\frac{2}{2-p_{\wedge}}
\left[\left(\frac{2}{p_{\wedge}}\right)^{|\mathcal S|-1}-1\right],
& U\in\{D_{\mathrm{KL}},D_{f_k}\},\\[2ex]
1/\delta, & U=d_{\mathrm{TV}},\\
(c/\delta)^\ell, & U=W_\ell.
\end{cases}
\]
\end{theorem}
Therefore, combining Theorems~\ref{thm:robust-average-reward-error}
and~\ref{thm:divergence-bias-span-bounds}, the sample complexity for
learning an $\varepsilon$-optimal policy is
\[
\widetilde O\!\left(
|\mathcal S||\mathcal A|p_\wedge^{-2}\varepsilon^{-2}
\begin{cases}
\left[(2/p_\wedge)^{|\mathcal S|-1}-1\right]^2,
& U\in\{D_{\mathrm{KL}},D_{f_k}\},\\
\delta^{-2}, & U=d_{\mathrm{TV}},\\
(c/\delta)^{2\ell}, & U=W_\ell.
\end{cases}
\right).
\]

The TV and Wasserstein upper bounds in Theorem~\ref{thm:divergence-bias-span-bounds} are tight, as shown by the two-state examples in the remark in \MainResultsReference. Furthermore, for divergence-based uncertainty sets, Theorem~\ref{thm:divergence-bias-span-lower-bound} below shows that the exponential term $p_\wedge^{-|\mathcal{S}|}$ in the span bound is sharp.
\begin{theorem}
\label{thm:divergence-bias-span-lower-bound}
For every $0<p_{\wedge}\le1/2$, there exists a weakly communicating
 nominal MDP such that, under either KL or $f_k$ uncertainty, for
every sufficiently small $\delta>0$,
\[
\operatorname{Span}(u_\delta^*)
\ge
\frac{p_{\wedge}^{-(|\mathcal S|-1)}-1}{1-p_{\wedge}}.
\]
\end{theorem}
The proof of Theorems~\ref{thm:robust-average-reward-error},  \ref{thm:divergence-bias-span-bounds} and \ref{thm:divergence-bias-span-lower-bound} are given in
\MainResultsReference.

\graphicspath{{manuscript/figures/}{figures/}{../figures/}}
\newcommand{\RobustAverageRewardAlgorithm}{Algorithm~\ref{alg:dr-amdp-reduction}}
\newcommand{\RobustAverageRewardTheorem}{\hyperref[thm:robust-average-reward-error]{Theorem~\ref*{thm:robust-average-reward-error}}}
\section{Numerical Experiments}\label{sec:numberical_experiments}

Theorem~\ref{thm:robust-average-reward-error} shows that the gain-estimation error decreases
at a rate of $n^{-1/2}$. We verify our main theorem with
numerical validation. Considering both
divergence-based and metric-based uncertainty sets under the two
rectangularity structures. On the constructed MDP instance~\ref{def:hard-mdp-family}, we ran Algorithm~\ref{alg:dr-amdp-reduction} under the four uncertainty sets: KL-divergence, $\chi^2$-divergence, TV and Wasserstein uncertainty sets under both SA- and S-rectangularity.

\begin{definition}
\label{def:hard-mdp-family}
Consider the family of MDPs with state space $\mathcal S=\{0,1,2\}$ and
action space $\mathcal A=\{a_1,a_2\}$. For $p\in(0,1)$, the nominal
transition rows, listed in state order $(0,1,2)$, are
\[
\begin{aligned}
p_{0,a_1}&=(p,1-p,0), & p_{0,a_2}&=(p,0,1-p),\\
p_{1,a_1}=p_{2,a_2}&=(0,1,0), & p_{1,a_2}=p_{2,a_1}&=(0,0,1).
\end{aligned}
\]
For every $x\in\mathcal S$, the rewards are $r(0,a,x)=0$ for
$a\in\mathcal A$, and $r(s,a_1,x)=1$, $r(s,a_2,x)=5/6$ for $s\in\{1,2\}$.
\end{definition}

\begin{figure}[H]
\centering
\includegraphics[width=0.40\linewidth]{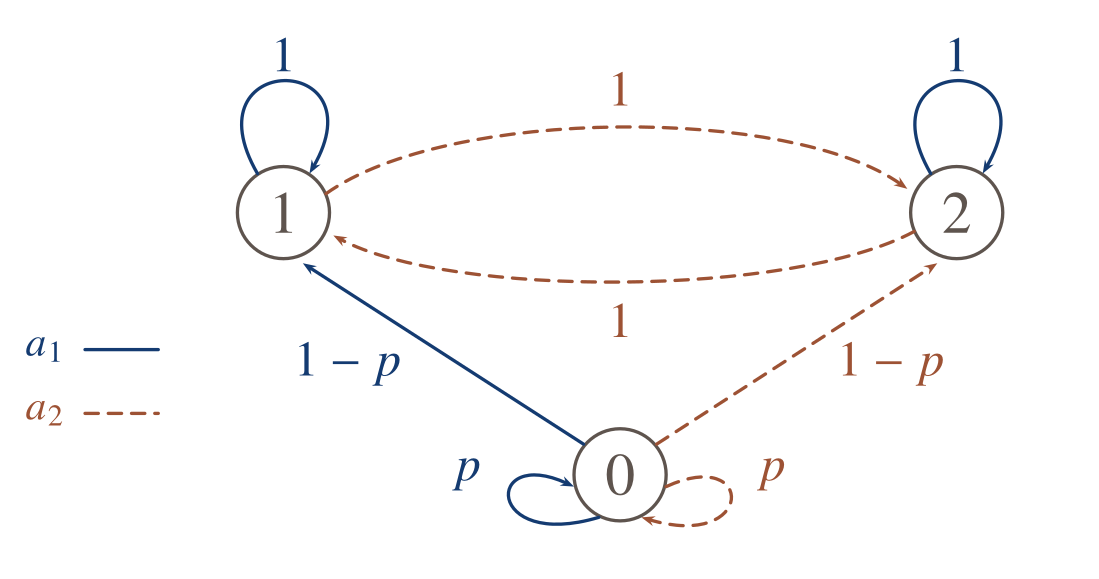}
\caption{Nominal MDP in Definition~\ref{def:hard-mdp-family}.}
\label{fig:hard-mdp-transition}
\end{figure}

This family, adapted from \citet[Figure~2(b)]{wan2024convergence}, is weakly
communicating but not unichain: state $0$ is transient under every policy,
and choosing $a_1$ at states $1$ and $2$ creates two recurrent classes.
We use $p\in\{0.2,0.5,0.8\}$, with $\delta=0.01$ for KL and $\chi^2$
($k=2$), and $\delta=0.1$ for TV and Wasserstein. For Wasserstein, we set
$\ell=2$ and $\rho(s,s')=\mathbbm{1}\{s\ne s'\}$. We evaluate $20$ sample
sizes from $10$ to $10^5$, with $30$ independent trials per sample size.
Figure~\ref{fig:hard-mdp-finite-sample} reports the results.

\begin{figure}[H]
\centering
% Two aligned rows with one shared title for each axis.
\begin{minipage}[c]{0.035\linewidth}
\centering
\rotatebox{90}{\footnotesize mean $\ell_\infty$ error with 95\% CI}
\end{minipage}\hfill
\begin{minipage}[c]{0.96\linewidth}
\centering
{\small SA-rectangular\par}
\smallskip
\begin{subfigure}[t]{0.24\linewidth}
\centering
\includegraphics[width=\linewidth]{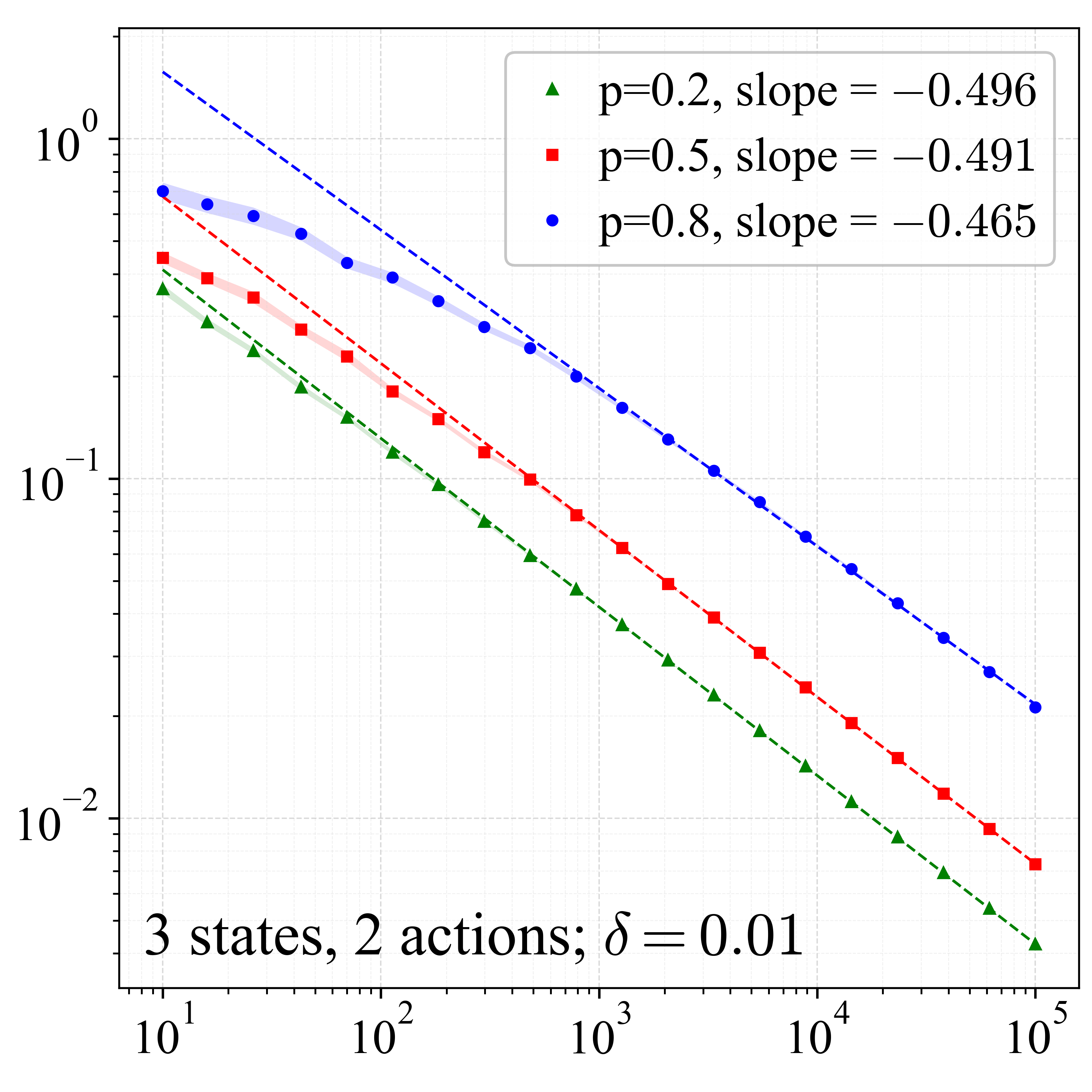}
\caption{KL}
\end{subfigure}\hfill
\begin{subfigure}[t]{0.24\linewidth}
\centering
\includegraphics[width=\linewidth]{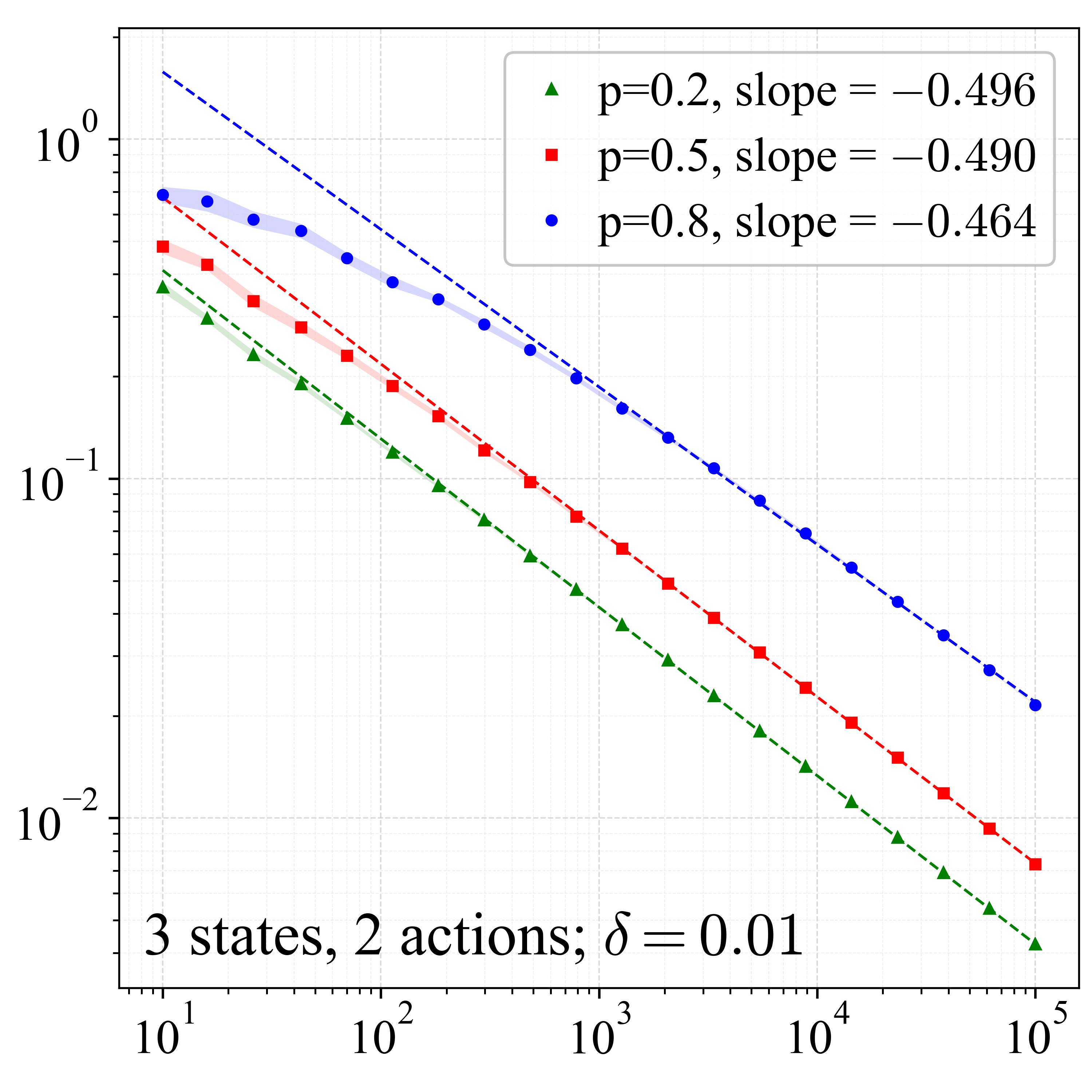}
\caption{$\chi^2$}
\end{subfigure}\hfill
\begin{subfigure}[t]{0.24\linewidth}
\centering
\includegraphics[width=\linewidth]{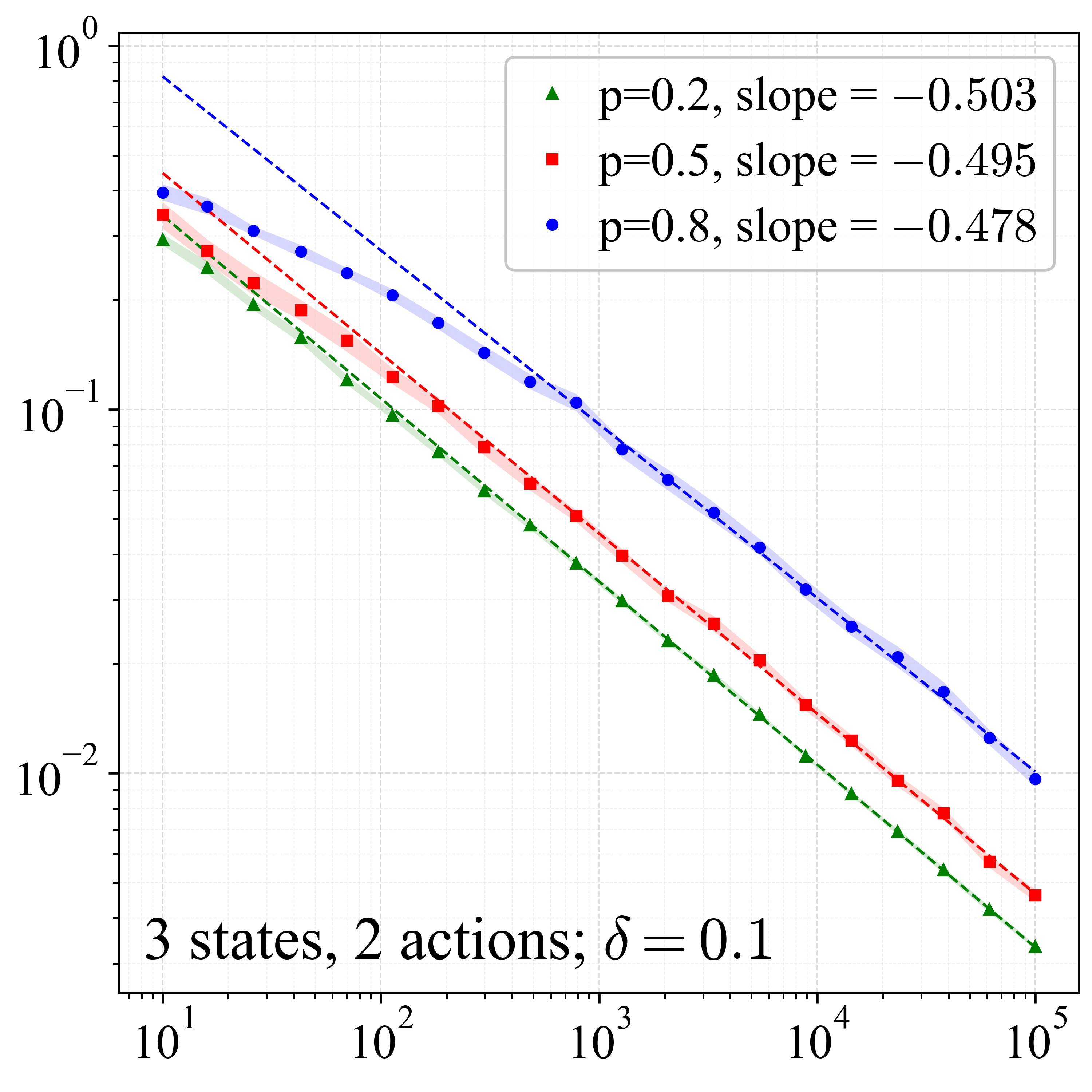}
\caption{TV}
\end{subfigure}\hfill
\begin{subfigure}[t]{0.24\linewidth}
\centering
\includegraphics[width=\linewidth]{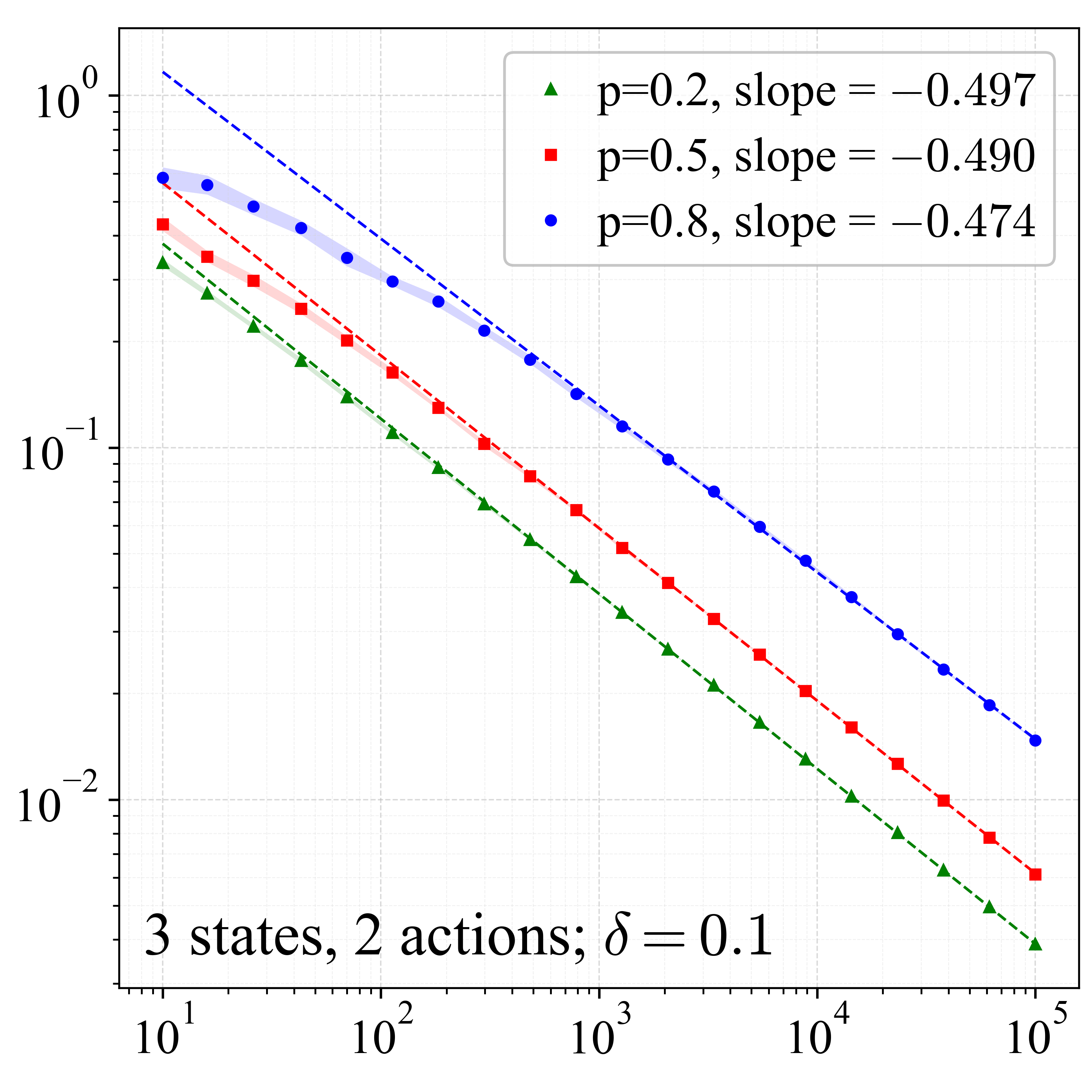}
\caption{Wasserstein}
\end{subfigure}

\medskip
{\small S-rectangular\par}
\smallskip
\begin{subfigure}[t]{0.24\linewidth}
\centering
\includegraphics[width=\linewidth]{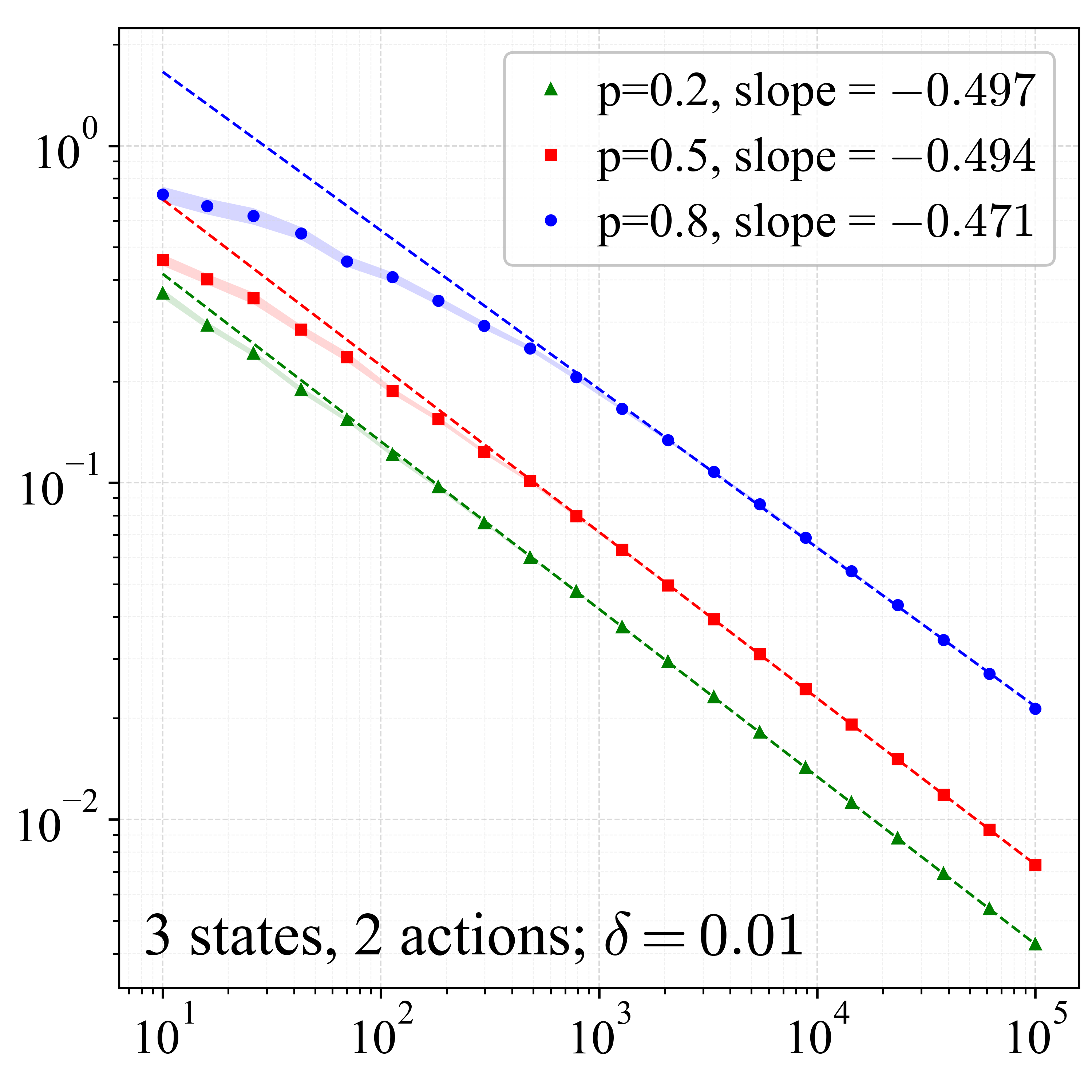}
\caption{KL}
\end{subfigure}\hfill
\begin{subfigure}[t]{0.24\linewidth}
\centering
\includegraphics[width=\linewidth]{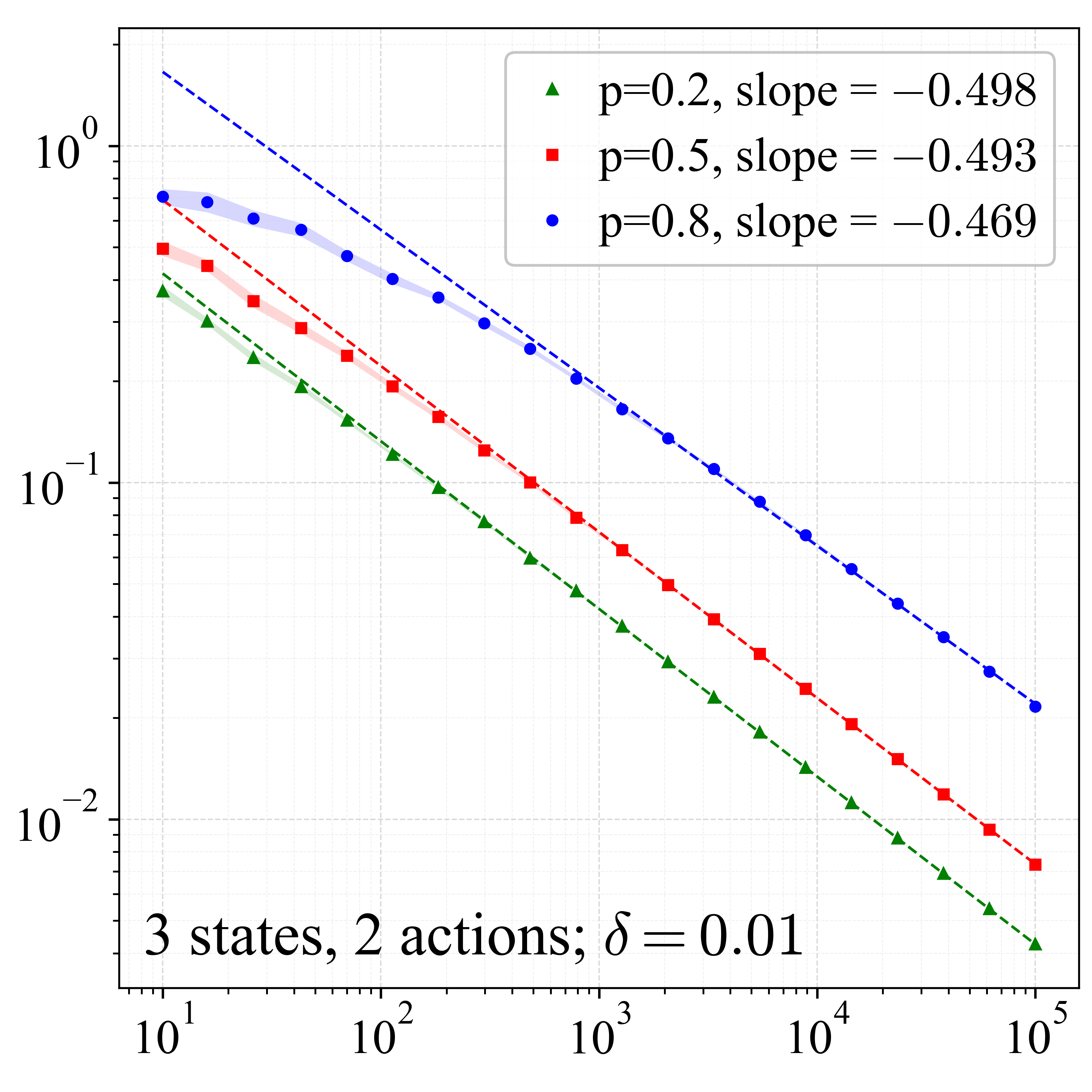}
\caption{$\chi^2$}
\end{subfigure}\hfill
\begin{subfigure}[t]{0.24\linewidth}
\centering
\includegraphics[width=\linewidth]{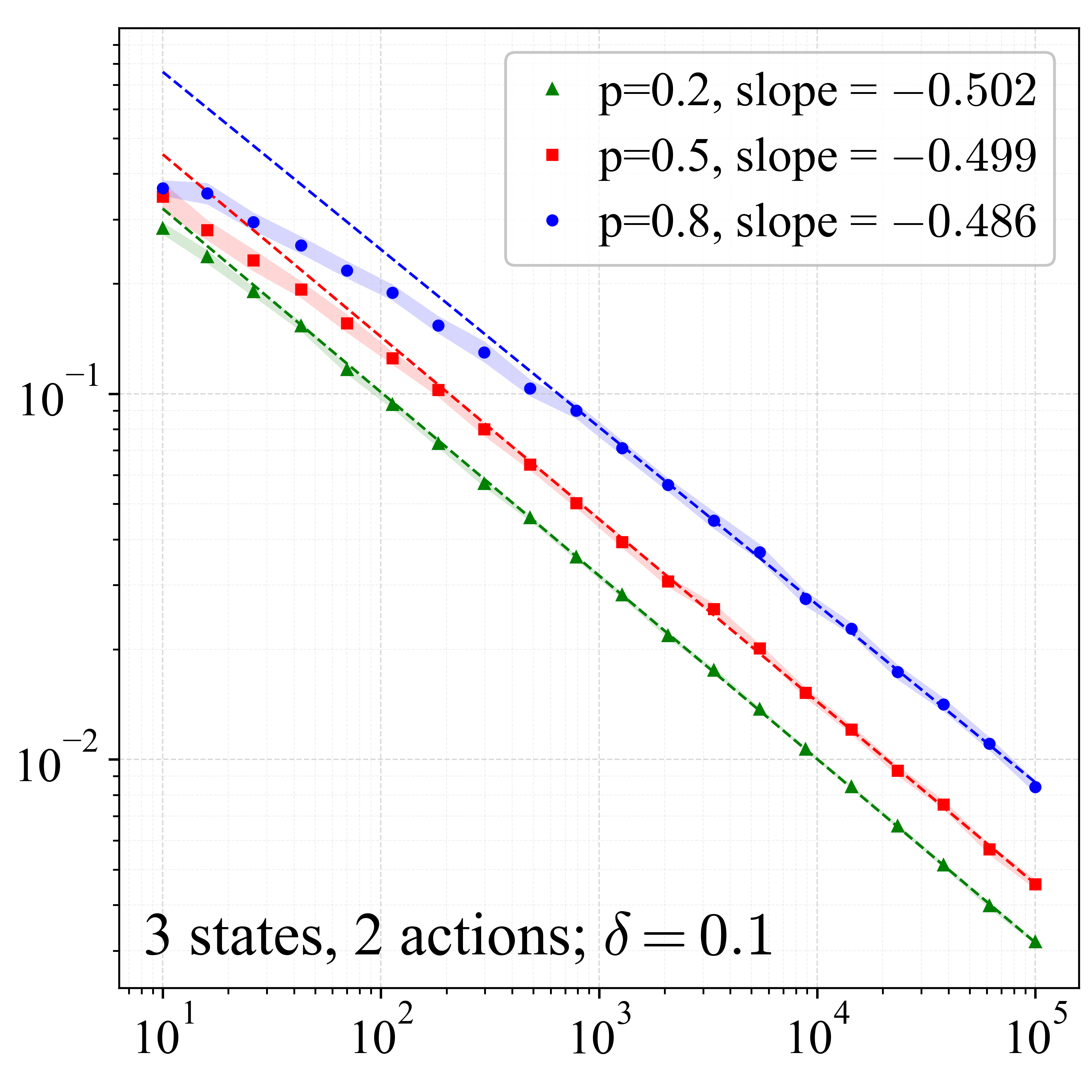}
\caption{TV}
\end{subfigure}\hfill
\begin{subfigure}[t]{0.24\linewidth}
\centering
\includegraphics[width=\linewidth]{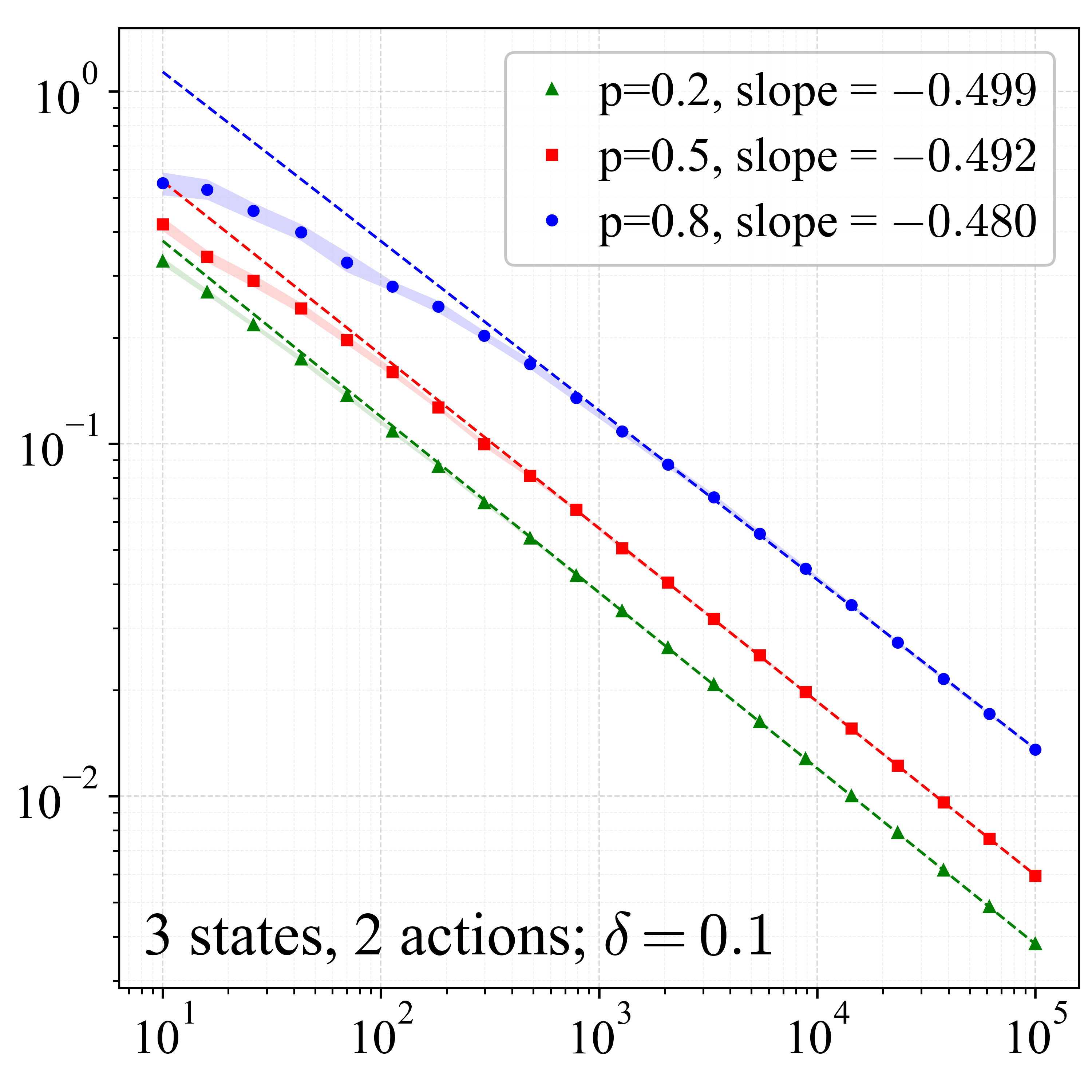}
\caption{Wasserstein}
\end{subfigure}
\end{minipage}
\par\smallskip
{\footnotesize number of samples per state-action pair\par}
\caption[Finite-sample performance on the hard MDP family]{Mean $\ell_\infty$ gain-estimation error of \RobustAverageRewardAlgorithm\ versus the number of samples per state-action pair for $p\in\{0.2,0.5,0.8\}$. Shaded bands indicate percentile-bootstrap 95\% confidence intervals over $30$ trials. Dashed lines are log--log fits using $n\ge10^3$, extended over the full plotted range.}
\label{fig:hard-mdp-finite-sample}
\end{figure}

Figure~\ref{fig:hard-mdp-finite-sample} plots the mean $\ell_\infty$ estimation error against the number of samples per state-action pair $n$, with both axes on logarithmic scales. Across all eight configurations and all three values of $p$, the fitted slopes range from $-0.503$ to $-0.464$, indicating an empirical error decay close to $n^{-1/2}$. This agrees with Theorem~\ref{thm:robust-average-reward-error}, which implies an $O(\epsilon^{-2})$ sample complexity, up to problem-dependent and logarithmic factors, for achieving estimation error at most $\epsilon$. Additional large-scale experiments on a nominal MDP with $20$ states and $30$ actions are provided in
Appendix~\ref{sec:additional-experiments}
for all eight combinations of rectangularity structures and uncertainty sets, where the fitted slopes range from $-0.527$ to $-0.490$ and exhibit the same convergence behavior.

\section*{Acknowledgement}
N. Si's research is supported in part by the Early Career Scheme [Grant 26210125] from the Hong Kong Research Grants Council.

\bibliographystyle{author_year}
\bibliography{reference}

\clearpage

\appendix
% \appendixpage

\section{Basic Properties}
\label{sec:basic-properties}

This appendix provides results underlying
Propositions~\ref{prop:divergence-constant-gain} and~\ref{prop:metric-constant-gain}
and proves Corollary~\ref{cor:stationary-average-optimal-policy}.
We first show that the inf--sup equation has a solution, justify exchanging
the infimum and maximum, and prove uniqueness of the constant gain when the
rewards depend on the realized next state. We then apply these results to KL-
and $f_k$-divergence uncertainty sets and treat TV and Wasserstein uncertainty
sets by taking a limit of normalized discounted values.

For $p\in\Delta(\mathcal S)$ and $f:\mathcal S\to\mathbb R$, write
$p[f]:=\sum_{x\in\mathcal S}p(x)f(x)$.

\begin{definition}
\label{def:adversary-weakly-communicating}
The adversary is weakly communicating if, for every $Q\in\mathcal P$, the
MDP with transition kernel $Q$ is weakly communicating in the sense of
Definition~\ref{def:weakly-communicating-mdp}.
\end{definition}

The preceding definition is identical to the adversary-side definition in
\citet[Definition~3]{wang2025bellmanoptimalityaveragerewardrobust}. The inf--sup equation in the following
lemma is adapted from \citet[Theorem~6]{wang2025bellmanoptimalityaveragerewardrobust} to rewards that depend on
the realized next state.

\begin{lemma}
\label{lem:inf-sup-constant-gain-existence}
Suppose that $\mathcal P=\prod_{s\in\mathcal S}\mathcal P_s$ is
$S$-rectangular, the adversary is weakly communicating, and every
$\mathcal P_s$ is compact. Then there exists
$(u',\alpha')\in\mathbb R^{\mathcal S}\times[0,1]$ such that
\begin{equation*}
\tag{IS}\label{eq:inf-sup-constant-gain}
u'(s)
=
\inf_{q_s\in\mathcal P_s}\max_{\phi\in\Delta(\mathcal A)}
\sum_{a\in\mathcal A}\sum_{s'\in\mathcal S}
\phi(a)q_{s,a}(s')
\bigl[r(s,a,s')-\alpha'+u'(s')\bigr],
\qquad s\in\mathcal S.
\end{equation*}
\end{lemma}

\begin{proof}
For $\pi\in\Pi_{\mathrm{SR}}$ and $Q\in\mathcal P$, write
\[
V_\gamma^{\pi,Q}(s)
:=
{\operatorname E}_Q^\pi\!\left[
\sum_{t=0}^{\infty}\gamma^t r(S_t,A_t,S_{t+1})
\mathrel{\Big|}S_0=s
\right].
\]
Fix $Q=\{q_{s,a}\}_{s,a}\in\mathcal P$ and define
\[
\bar r_Q(s,a)
:=
\sum_{s'\in\mathcal S}q_{s,a}(s')r(s,a,s').
\]
For every $\pi\in\Pi_{\mathrm{SR}}$, conditional expectation gives
\[
{\operatorname E}_Q^\pi\!\left[
\sum_{t=0}^{\infty}\gamma^t r(S_t,A_t,S_{t+1})
\mathrel{\Big|}S_0=s
\right]
=
{\operatorname E}_Q^\pi\!\left[
\sum_{t=0}^{\infty}\gamma^t\bar r_Q(S_t,A_t)
\mathrel{\Big|}S_0=s
\right].
\]
Thus, once $Q$ is fixed, this is an ordinary discounted MDP, and the
stationary $\varepsilon$-optimality argument in the proof of
\citet[Theorem~6]{wang2025bellmanoptimalityaveragerewardrobust} applies.

The finite-cover construction and Lemmas~5--8 in Appendix~D of
\citet{wang2025bellmanoptimalityaveragerewardrobust} concern only induced transition kernels, communicating
classes, and hitting times, and hence remain valid when the reward is
$r(s,a,s')$. They yield a finite cover $\{G_{Q_k}\}_{k\in B'}$ of $\mathcal P$,
classes $C_{Q_k}$, and $\delta'>0$ such that, for $Q\in G_{Q_k}$ and
$y\in C_{Q_k}$, some $\widetilde\pi\in\Pi_{\mathrm{SR}}$ satisfies
\[
\max_{w\in\mathcal S}
{\operatorname E}_Q^{\widetilde\pi}[\tau_y\mid S_0=w]
\le \frac{|\mathcal S|}{\delta'}.
\]
Moreover, with
\[
T_k:=\inf\{t\ge0:S_t\in C_{Q_k}\},
\]
\citet[Lemma~6]{wang2025bellmanoptimalityaveragerewardrobust} and the finite cover give
\[
\max_{k\in B'}
\sup_{\substack{Q\in G_{Q_k},\,\pi\in\Pi_{\mathrm{SR}}\\
w\in\mathcal S}}
{\operatorname E}_Q^\pi[T_k\mid S_0=w]
<\infty.
\]

For $\gamma\in(0,1)$, let $v'_\gamma$ solve
\begin{equation*}
\tag{D}\label{eq:inf-sup-discounted}
v'_\gamma(s)
=
\inf_{q_s\in\mathcal P_s}\max_{\phi\in\Delta(\mathcal A)}
\sum_{a\in\mathcal A}\sum_{s'\in\mathcal S}
\phi(a)q_{s,a}(s')
\bigl[r(s,a,s')+\gamma v'_\gamma(s')\bigr].
\end{equation*}
Fix $\varepsilon>0$. For each $s\in\mathcal S$, choose
$q_{\varepsilon,s}=\{q_{\varepsilon,s,a}\}_{a\in\mathcal A}
\in\mathcal P_s$ such that
\[
\max_{\phi\in\Delta(\mathcal A)}
\sum_{a,s'}\phi(a)q_{\varepsilon,s,a}(s')
\bigl[r(s,a,s')+\gamma v'_\gamma(s')\bigr]
\le v'_\gamma(s)+(1-\gamma)\varepsilon.
\]
By $S$-rectangularity, the statewise choices form
$Q_\varepsilon\in\mathcal P$. Discounted fixed-point comparison gives,
simultaneously for all $s\in\mathcal S$,
\[
v'_\gamma(s)
\le
\max_{\pi\in\Pi_{\mathrm{SR}}}V_\gamma^{\pi,Q_\varepsilon}(s)
\le
v'_\gamma(s)+\varepsilon.
\]
The first inequality follows because $Q_\varepsilon$ is feasible in the
infimum in \eqref{eq:inf-sup-discounted}, while
$(1-\gamma)\varepsilon$ accumulates to at most
$\varepsilon$. For this finite discounted MDP, choose
$\pi_\varepsilon\in\Pi_{\mathrm{SD}}$ that is optimal for every initial state, so that
\[
\max_{\pi\in\Pi_{\mathrm{SR}}}V_\gamma^{\pi,Q_\varepsilon}(s)
=
V_\gamma^{\pi_\varepsilon,Q_\varepsilon}(s),
\qquad s\in\mathcal S.
\]
Choose $k\in B'$ with $Q_\varepsilon\in G_{Q_k}$ and let
\begin{equation*}
\tag{Y}\label{eq:fixed-policy-class-maximizer}
\begin{gathered}
s_\vee\in\arg\max_{s\in\mathcal S}v'_\gamma(s),
\qquad
s_\wedge\in\arg\min_{s\in\mathcal S}v'_\gamma(s),\\
y_\vee\in\arg\max_{y\in C_{Q_k}}
V_\gamma^{\pi_\varepsilon,Q_\varepsilon}(y).
\end{gathered}
\end{equation*}
The preceding comparison gives
\begin{equation*}
\tag{1}\label{eq:span-decomposition}
\begin{aligned}
\operatorname{Span}(v'_\gamma)
&\le
V_\gamma^{\pi_\varepsilon,Q_\varepsilon}(s_\vee)
-V_\gamma^{\pi_\varepsilon,Q_\varepsilon}(s_\wedge)+2\varepsilon \\
&=
\underbrace{
V_\gamma^{\pi_\varepsilon,Q_\varepsilon}(s_\vee)
-V_\gamma^{\pi_\varepsilon,Q_\varepsilon}(y_\vee)
}_{\xi_2}
+
\underbrace{
V_\gamma^{\pi_\varepsilon,Q_\varepsilon}(y_\vee)
-V_\gamma^{\pi_\varepsilon,Q_\varepsilon}(s_\wedge)
}_{\xi_1}
+2\varepsilon.
\end{aligned}
\end{equation*}

\emph{Upper-bounding $\xi_1$.}
Fix $y\in C_{Q_k}$. By \citet[Lemma~8]{wang2025bellmanoptimalityaveragerewardrobust}, there exists
$\widetilde\pi\in\Pi_{\mathrm{SR}}$ such that
\[
\max_{w\in\mathcal S}
{\operatorname E}_{Q_\varepsilon}^{\widetilde\pi}
[\tau_y\mid S_0=w]
\le \frac{|\mathcal S|}{\delta'}.
\]
Since $V_\gamma^{\pi_\varepsilon,Q_\varepsilon}$ is the optimal discounted
value for $Q_\varepsilon$, its Bellman inequality under $\widetilde\pi$
gives, by iteration up to $\tau_y$, for every $x\in\mathcal S$,
\begin{equation*}
\tag{2}\label{eq:hitting-time-decomposition}
\begin{aligned}
V_\gamma^{\pi_\varepsilon,Q_\varepsilon}(x)
&\ge
{\operatorname E}_{Q_\varepsilon}^{\widetilde\pi}\!\left[
\sum_{t=0}^{\tau_y-1}\gamma^t r(S_t,A_t,S_{t+1})
\mathrel{\Big|}S_0=x
\right]\\
&\quad+
{\operatorname E}_{Q_\varepsilon}^{\widetilde\pi}\!\left[
\gamma^{\tau_y}
V_\gamma^{\pi_\varepsilon,Q_\varepsilon}(y)
\mathrel{\Big|}S_0=x
\right]\\
&\ge
V_\gamma^{\pi_\varepsilon,Q_\varepsilon}(y)
{\operatorname E}_{Q_\varepsilon}^{\widetilde\pi}
[\gamma^{\tau_y}\mid S_0=x].
\end{aligned}
\end{equation*}
The stopped inequality follows by truncation and boundedness of the
discounted value. The first sum includes the reward on the transition
entering $y$. Hence
\[
\begin{aligned}
&V_\gamma^{\pi_\varepsilon,Q_\varepsilon}(y)
-V_\gamma^{\pi_\varepsilon,Q_\varepsilon}(x)\\
&\quad\le
V_\gamma^{\pi_\varepsilon,Q_\varepsilon}(y)
\left(1-
{\operatorname E}_{Q_\varepsilon}^{\widetilde\pi}
[\gamma^{\tau_y}\mid S_0=x]
\right)\\
&\quad=
(1-\gamma)V_\gamma^{\pi_\varepsilon,Q_\varepsilon}(y)
\,
{\operatorname E}_{Q_\varepsilon}^{\widetilde\pi}\!\left[
\frac{1-\gamma^{\tau_y}}{1-\gamma}
\mathrel{\Big|}S_0=x
\right]\\
&\quad\le
{\operatorname E}_{Q_\varepsilon}^{\widetilde\pi}[\tau_y\mid S_0=x].
\end{aligned}
\]
The last inequality follows from
$0\le V_\gamma^{\pi_\varepsilon,Q_\varepsilon}(y)
\le(1-\gamma)^{-1}$ and
$(1-\gamma^t)/(1-\gamma)=\sum_{j=0}^{t-1}\gamma^j\le t$.
Taking $x=s_\wedge$ and $y=y_\vee$ therefore gives
\begin{equation*}
\tag{3}\label{eq:xi-one-bound}
\xi_1
\le
{\operatorname E}_{Q_\varepsilon}^{\widetilde\pi}
[\tau_{y_\vee}\mid S_0=s_\wedge]
\le \frac{|\mathcal S|}{\delta'}.
\end{equation*}

\emph{Upper-bounding $\xi_2$.}
Let $T_k:=\inf\{t\ge0:S_t\in C_{Q_k}\}$. Then
\begin{equation*}
\tag{4}\label{eq:entrance-time-decomposition}
\begin{aligned}
\xi_2
&=
V_\gamma^{\pi_\varepsilon,Q_\varepsilon}(s_\vee)
-V_\gamma^{\pi_\varepsilon,Q_\varepsilon}(y_\vee)\\
&=
{\operatorname E}_{Q_\varepsilon}^{\pi_\varepsilon}\!\left[
\sum_{t=0}^{T_k-1}\gamma^t r(S_t,A_t,S_{t+1})
+\sum_{t=T_k}^{\infty}\gamma^t r(S_t,A_t,S_{t+1})
\mathrel{\Big|}S_0=s_\vee
\right]\\
&\quad-V_\gamma^{\pi_\varepsilon,Q_\varepsilon}(y_\vee)\\
&\stackrel{\mathrm{(i)}}{=}
{\operatorname E}_{Q_\varepsilon}^{\pi_\varepsilon}\!\left[
\sum_{t=0}^{T_k-1}\gamma^t r(S_t,A_t,S_{t+1})
+\gamma^{T_k}
V_\gamma^{\pi_\varepsilon,Q_\varepsilon}(S_{T_k})
\mathrel{\Big|}S_0=s_\vee
\right]\\
&\quad-V_\gamma^{\pi_\varepsilon,Q_\varepsilon}(y_\vee)\\
&\stackrel{\mathrm{(ii)}}{\le}
{\operatorname E}_{Q_\varepsilon}^{\pi_\varepsilon}\!\left[
\sum_{t=0}^{T_k-1}\gamma^t r(S_t,A_t,S_{t+1})
\mathrel{\Big|}S_0=s_\vee
\right]\\
&\le
{\operatorname E}_{Q_\varepsilon}^{\pi_\varepsilon}
[T_k\mid S_0=s_\vee].
\end{aligned}
\end{equation*}
Here \emph{(i)} follows from the strong Markov property, while
\emph{(ii)} follows from $S_{T_k}\in C_{Q_k}$, the choice of $y_\vee$ in
\eqref{eq:fixed-policy-class-maximizer}, and nonnegativity of the discounted
value. Consequently,
\begin{equation*}
\tag{5}\label{eq:xi-two-bound}
\xi_2
\le
\max_{k\in B'}
\sup_{\substack{Q\in G_{Q_k},\,\pi\in\Pi_{\mathrm{SR}}\\
w\in\mathcal S}}
{\operatorname E}_Q^\pi[T_k\mid S_0=w]
<\infty.
\end{equation*}
Combining \eqref{eq:span-decomposition}--\eqref{eq:xi-two-bound}, letting
$\varepsilon\downarrow0$, and applying the
uniform entrance-time bound yields
\begin{equation*}
\tag{6}\label{eq:uniform-inf-sup-span-bound}
\sup_{\gamma\in(0,1)}\operatorname{Span}(v'_\gamma)
\le
\frac{|\mathcal S|}{\delta'}
+\max_{k\in B'}
\sup_{\substack{Q\in G_{Q_k},\,\pi\in\Pi_{\mathrm{SR}}\\
w\in\mathcal S}}
{\operatorname E}_Q^\pi[T_k\mid S_0=w]
<\infty.
\end{equation*}
Fix $s_0\in\mathcal S$ and set
\[
h_\gamma:=v'_\gamma-v'_\gamma(s_0)\mathbf 1,
\qquad
\alpha_\gamma:=(1-\gamma)v'_\gamma(s_0).
\]
Since $r(s,a,s')\in[0,1]$, \eqref{eq:inf-sup-discounted} gives
$0\le v'_\gamma(s)\le(1-\gamma)^{-1}$ and hence
$0\le\alpha_\gamma\le1$. By
\eqref{eq:uniform-inf-sup-span-bound}, along some $\gamma_n\uparrow1$,
\[
h_{\gamma_n}\to u',
\qquad
\alpha_{\gamma_n}\to\alpha'\in[0,1].
\]
The normalized form of \eqref{eq:inf-sup-discounted} is
\[
h_\gamma(s)
=
\inf_{q_s\in\mathcal P_s}\max_{\phi\in\Delta(\mathcal A)}
\sum_{a,s'}\phi(a)q_{s,a}(s')
\bigl[r(s,a,s')-\alpha_\gamma+\gamma h_\gamma(s')\bigr].
\]
Moreover,
\[
\begin{aligned}
&\max_{s\in\mathcal S}
\sup_{\substack{\phi\in\Delta(\mathcal A)\\q_s\in\mathcal P_s}}
\left|
\begin{aligned}
&\sum_{a,s'}\phi(a)q_{s,a}(s')
\bigl[r(s,a,s')-\alpha_{\gamma_n}
+\gamma_nh_{\gamma_n}(s')\bigr]\\
&\quad-
\sum_{a,s'}\phi(a)q_{s,a}(s')
\bigl[r(s,a,s')-\alpha'+u'(s')\bigr]
\end{aligned}
\right|\\
&\le
|\alpha_{\gamma_n}-\alpha'|
+\max_{s'\in\mathcal S}
|\gamma_nh_{\gamma_n}(s')-u'(s')|
\longrightarrow0.
\end{aligned}
\]
The normalization-and-subsequence argument in
\citet[Theorem~4]{wang2025bellmanoptimalityaveragerewardrobust} now allows passage to the limit in
\eqref{eq:inf-sup-discounted}, giving
\eqref{eq:inf-sup-constant-gain}.
\end{proof}

The following lemma adapts \citet[Theorem~7(1)]{wang2025bellmanoptimalityaveragerewardrobust} to rewards that
depend on the realized next state.

\begin{lemma}
\label{lem:inf-sup-to-optimal-bellman}
Suppose that $(u',\alpha')$ solves the inf--sup equation in
Lemma~\ref{lem:inf-sup-constant-gain-existence}. Assume that the controller
action set is compact and convex and that every $\mathcal P_s$ is convex.
Then $(u',\alpha')$ satisfies the optimal average-reward Bellman equation~\eqref{eq:constant-gain-bellman}{}
defined in the preliminaries.
\end{lemma}

\begin{proof}
For fixed $s\in\mathcal S$, let
\[
F_s(\phi,q_s)
:=
\sum_{a\in\mathcal A}\sum_{s'\in\mathcal S}
\phi(a)q_{s,a}(s')
\bigl[r(s,a,s')-\alpha'+u'(s')\bigr].
\]
The function $F_s$ is linear and continuous in each argument. Since
$\mathcal A$ is finite, the controller action set $\Delta(\mathcal A)$ is
compact and convex; together with the convexity of $\mathcal P_s$, this
satisfies the conditions of Sion's minimax theorem
\citep[Corollary~3.3]{sion_1958_minimax}. Hence,
\[
\inf_{q_s\in\mathcal P_s}\max_{\phi\in\Delta(\mathcal A)}F_s(\phi,q_s)
=
\max_{\phi\in\Delta(\mathcal A)}\inf_{q_s\in\mathcal P_s}F_s(\phi,q_s).
\]
Combining this equality with the inf--sup equation in
Lemma~\ref{lem:inf-sup-constant-gain-existence} gives the claimed Bellman
equation~\eqref{eq:constant-gain-bellman}.
\end{proof}

Below, we show that if the optimal average-reward Bellman equation~\eqref{eq:constant-gain-bellman}\ admits a solution, then the constant gain must be unique.

\begin{lemma}
\label{lem:constant-gain-uniqueness}
For either the SA-rectangular or $S$-rectangular optimal average-reward
Bellman operator $\mathcal T^*$\ in \eqref{eq:average-bellman-operator}, suppose that $(u,\alpha)$ and $(v,\beta)$
are two solutions\ of \eqref{eq:constant-gain-bellman}:
\[
\mathcal T^*(u)=u+\alpha\mathbf 1,
\qquad
\mathcal T^*(v)=v+\beta\mathbf 1.
\]
Then $\alpha=\beta$. Thus only the constant gain is necessarily unique;
the relative value function need not be unique.
\end{lemma}

\begin{proof}
For both the SA-rectangular and $S$-rectangular operators in~\eqref{eq:average-bellman-operator},
\[
u\le v
\quad\Longrightarrow\quad
\mathcal T^*(u)\le\mathcal T^*(v),
\qquad
\mathcal T^*(v+c\mathbf 1)=\mathcal T^*(v)+c\mathbf 1,
\quad c\in\mathbb R.
\]
The term $r(s,a,s')$ is unchanged when the relative value function is compared
or shifted, so it does not affect these properties.

Let
\[
c:=\max_{s\in\mathcal S}\{u(s)-v(s)\},
\qquad
\bar s\in\arg\max_{s\in\mathcal S}\{u(s)-v(s)\}.
\]
Then $u\le v+c\mathbf 1$ and $u(\bar s)=v(\bar s)+c$. Monotonicity and
additive homogeneity give
\[
u+\alpha\mathbf 1
=
\mathcal T^*(u)
\le
\mathcal T^*(v+c\mathbf 1)
=
v+(\beta+c)\mathbf 1.
\]
Evaluating at $\bar s$ yields $\alpha\le\beta$. Interchanging the two
solutions gives $\beta\le\alpha$, and therefore $\alpha=\beta$.
\end{proof}

\begin{lemma}
\label{prop:adversarial-weak-communication}
Under Assumption~\ref{ass:divergence-uncertainty}, every adversarial
transition kernel $Q\in\mathcal P$ is weakly communicating.
\end{lemma}

\begin{proof}
Fix $Q\in\mathcal P$. For every $(s,a)$, the definition of the
divergence-based uncertainty sets gives $q_{s,a}\ll p_{s,a}$. Thus,
\[
p_{s,a}(s')=0\quad\Longrightarrow\quad q_{s,a}(s')=0.
\]
Conversely, let $p_{s,a}(s')>0$ and suppose that $q_{s,a}(s')=0$. In the KL
case, applying the log-sum inequality to
$\mathcal S\setminus\{s'\}$ gives
\[
\begin{aligned}
D_{\mathrm{KL}}(q_{s,a}\|p_{s,a})
&=
q_{s,a}(s')\log\!\left(\frac{q_{s,a}(s')}{p_{s,a}(s')}\right)
+
\sum_{x\ne s'}q_{s,a}(x)
\log\!\left(\frac{q_{s,a}(x)}{p_{s,a}(x)}\right)\\
&\ge
q_{s,a}(s')\log\!\left(\frac{q_{s,a}(s')}{p_{s,a}(s')}\right)
+
\bigl(1-q_{s,a}(s')\bigr)
\log\!\left(
\frac{1-q_{s,a}(s')}{1-p_{s,a}(s')}
\right)\\
&=
\log\!\left(\frac{1}{1-p_{s,a}(s')}\right)
\ge
\log\!\left(\frac{1}{1-p_{\wedge}}\right).
\end{aligned}
\]
Consequently, in the SA-rectangular case,
\[
\delta
\ge D_{\mathrm{KL}}(q_{s,a}\|p_{s,a})
\ge \log\!\left(\frac{1}{1-p_{\wedge}}\right),
\]
whereas in the $S$-rectangular case,
\[
|\mathcal A|\delta
\ge
\sum_{b\in\mathcal A}D_{\mathrm{KL}}(q_{s,b}\|p_{s,b})
\ge
D_{\mathrm{KL}}(q_{s,a}\|p_{s,a})
\ge
\log\!\left(\frac{1}{1-p_{\wedge}}\right).
\]
Both conclusions contradict the corresponding KL radius condition in
Assumption~\ref{ass:divergence-uncertainty}.

In the $f_k$-divergence case, convexity of $t\mapsto t^k$ gives
\[
t^k\ge 1+k(t-1),
\qquad
f_k(t)=\frac{t^k-kt+k-1}{k(k-1)}\ge0,
\qquad t\ge0,
\]
and $f_k(0)=1/k$. Therefore,
\[
\begin{aligned}
D_{f_k}(q_{s,a}\|p_{s,a})
&=
\sum_{x\in\mathcal S}p_{s,a}(x)
f_k\!\left(\frac{q_{s,a}(x)}{p_{s,a}(x)}\right)\\
&\ge
p_{s,a}(s')f_k\!\left(\frac{q_{s,a}(s')}{p_{s,a}(s')}\right)
=
\frac{p_{s,a}(s')}{k}
\ge
\frac{p_{\wedge}}{k}.
\end{aligned}
\]
Thus, in the SA-rectangular case,
\[
\delta
\ge D_{f_k}(q_{s,a}\|p_{s,a})
\ge \frac{p_{\wedge}}{k},
\]
while in the $S$-rectangular case,
\[
|\mathcal A|\delta
\ge
\sum_{b\in\mathcal A}D_{f_k}(q_{s,b}\|p_{s,b})
\ge
D_{f_k}(q_{s,a}\|p_{s,a})
\ge
\frac{p_{\wedge}}{k}.
\]
These conclusions contradict the corresponding $f_k$-divergence radius
condition in
Assumption~\ref{ass:divergence-uncertainty}.
Hence
\[
\operatorname{supp}(q_{s,a})
=
\operatorname{supp}(p_{s,a}),
\qquad (s,a)\in\mathcal S\times\mathcal A.
\]
Therefore, for every $\pi\in\Pi_{\mathrm{SR}}$,
\[
\sum_{a\in\mathcal A}\pi(a\mid s)p_{s,a}(s')>0
\quad\Longleftrightarrow\quad
\sum_{a\in\mathcal A}\pi(a\mid s)q_{s,a}(s')>0,
\qquad s,s'\in\mathcal S.
\]
Thus the transition kernels induced by $P$ and $Q$ under $\pi$ have the same
directed graph. The partition witnessing weak communication of the nominal
MDP also witnesses weak communication for the MDP with transition kernel
$Q$. Since $Q\in\mathcal P$ was arbitrary, the adversary is weakly
communicating according to
Definition~\ref{def:adversary-weakly-communicating}.
\end{proof}

\begin{proof}[Proof of
Proposition~\ref{prop:divergence-constant-gain}]
By Lemma~\ref{prop:adversarial-weak-communication}
and Definition~\ref{def:adversary-weakly-communicating}, the adversary is
weakly communicating. Under
Assumption~\ref{ass:divergence-uncertainty},
the constraint $q_{s,a}\ll p_{s,a}$ defines a closed face of
$\Delta(\mathcal S)$. On this face,
$D_{\mathrm{KL}}(q_{s,a}\|p_{s,a})$ is continuous and convex in $q_{s,a}$. Likewise,
$D_{f_k}(q_{s,a}\|p_{s,a})$ is continuous and convex because $f_k$ is continuous
and convex. Hence every KL- or $f_k$-divergence set
$\mathcal P_{s,a}(U,\delta)$, with $U\in\{D_{\mathrm{KL}},D_{f_k}\}$, is a closed convex subset of the compact simplex
$\Delta(\mathcal S)$, and is therefore compact and convex.

In the SA-rectangular case, set
$\mathcal P_s:=\prod_{a\in\mathcal A}\mathcal P_{s,a}(U,\delta)$; these
statewise products are compact and convex.
In the $S$-rectangular case, the two defining functions
\[
\sum_{a\in\mathcal A}D_{\mathrm{KL}}(q_{s,a}\|p_{s,a})
\quad\text{and}\quad
\sum_{a\in\mathcal A}D_{f_k}(q_{s,a}\|p_{s,a})
\]
are continuous and convex on a finite product of the same closed support
faces. Their sublevel sets $\mathcal P_s(U,\delta)$ are therefore compact and
convex. Thus, under either structure,
$\mathcal P=\prod_{s\in\mathcal S}\mathcal P_s$ has compact and convex
statewise sets. Finally, $\Delta(\mathcal A)$ is compact and convex.
Lemma~\ref{lem:inf-sup-constant-gain-existence} therefore yields a solution
$(u',\alpha')$ of the inf--sup constant-gain equation, and
Lemma~\ref{lem:inf-sup-to-optimal-bellman} shows that the same pair solves
the optimal average-reward Bellman equation~\eqref{eq:constant-gain-bellman}. For the SA-rectangular model,
\[
\begin{aligned}
&\max_{\phi\in\Delta(\mathcal A)}
\inf_{q_s\in\prod_{a\in\mathcal A}\mathcal P_{s,a}}
\sum_{a\in\mathcal A}q_{s,a}
[\phi(a)(r(s,a,\cdot)+u')]\\
&\qquad=
\max_{a\in\mathcal A}\inf_{q_{s,a}\in\mathcal P_{s,a}}
q_{s,a}[r(s,a,\cdot)+u'].
\end{aligned}
\]
Thus this is precisely the SA-rectangular operator in Definition~\ref{defn:bellman-operators}.
Taking $u_\delta^*=u'$ and $g_\delta^*=\alpha'$ establishes existence. By
Lemma~\ref{lem:constant-gain-uniqueness}, every other solution has the same
constant gain $g_\delta^*$, which proves uniqueness.
\end{proof}

\begin{proof}[Proof of
Proposition~\ref{prop:metric-constant-gain}]
The argument is the same for the SA-rectangular and $S$-rectangular
formulations.

Choose any full-support $\nu\in\Delta(\mathcal S)$ and, for
$\varepsilon\in(0,1)$, set $\bar Q=\{\bar q_{s,a}\}_{s,a}$, where
\[
\bar q_{s,a}
:=(1-\varepsilon)p_{s,a}+\varepsilon\nu.
\]
For total variation,
\[
d_{\mathrm{TV}}(\bar q_{s,a},p_{s,a})
=\varepsilon d_{\mathrm{TV}}(\nu,p_{s,a})
\longrightarrow0,
\]
whereas, on the finite metric space $(\mathcal S,\rho)$,
\[
W_\ell(\bar q_{s,a},p_{s,a})^\ell
\le
\varepsilon\max_{x,y\in\mathcal S}\rho(x,y)^\ell
\longrightarrow0.
\]
Both distances, as well as their finite action-wise sums, vanish as
$\varepsilon\downarrow0$. Thus, under either rectangularity structure, a
common sufficiently small $\varepsilon>0$ gives $\bar Q\in\mathcal P$.
Moreover,
\[
\bar q_{\wedge}
:=
\min_{s,a,s'}\bar q_{s,a}(s')
\ge
\varepsilon\min_{s'}\nu(s')
>0.
\]

Fix $\gamma\in(0,1)$ and choose
$s_-\in\arg\min_x v_\gamma^*(x)$. For every $(s,a)$,
\[
\begin{aligned}
\bar q_{s,a}[v_\gamma^*]
&\le
\bar q_{s,a}(s_-)\min_x v_\gamma^*(x)
+\bigl(1-\bar q_{s,a}(s_-)\bigr)\max_x v_\gamma^*(x)\\
&\le
\bar q_{\wedge}\min_x v_\gamma^*(x)
+(1-\bar q_{\wedge})\max_x v_\gamma^*(x).
\end{aligned}
\]
Evaluating the inner infimum in $\mathcal T_\gamma^*$\ from \eqref{eq:discounted-bellman-operator}\ at the feasible
kernel $\bar Q$ gives, for both rectangularity structures,
\[
\begin{aligned}
\max_s v_\gamma^*(s)
&\le
\max_{s,a,s'}r(s,a,s')
+\gamma\left(
\bar q_{\wedge}\min_x v_\gamma^*(x)
+(1-\bar q_{\wedge})\max_x v_\gamma^*(x)
\right),\\
\min_s v_\gamma^*(s)
&\ge
\min_{s,a,s'}r(s,a,s')
+\gamma\min_x v_\gamma^*(x).
\end{aligned}
\]
Subtracting the two inequalities and using
$v_\gamma^*=\mathcal T_\gamma^*(v_\gamma^*)$ yields
\[
\begin{aligned}
\operatorname{Span}(v_\gamma^*)
&\le
\max_{s,a,s'}r(s,a,s')-\min_{s,a,s'}r(s,a,s')\\
&\quad+\gamma(1-\bar q_{\wedge})\operatorname{Span}(v_\gamma^*)\\
&\le
1+\gamma(1-\bar q_{\wedge})\operatorname{Span}(v_\gamma^*).
\end{aligned}
\]
and therefore
\[
\operatorname{Span}(v_\gamma^*)
\le
\frac{1}{1-\gamma(1-\bar q_{\wedge})}
\le
\frac{1}{\bar q_{\wedge}},
\qquad \gamma\in(0,1).
\]
The minimizing kernel need not have full support: feasibility of $\bar Q$
alone is sufficient because the inner infimum is no larger than its value
at $\bar Q$.

Fix $s_0\in\mathcal S$. The uniform span bound gives
\[
\left\|
v_\gamma^*-v_\gamma^*(s_0)\mathbf 1
\right\|_\infty
\le
\frac{1}{\bar q_{\wedge}}.
\]
Boundedness of the reward and the discounted Bellman equation for the operator\ in \eqref{eq:discounted-bellman-operator}\ also give
\[
0\le v_\gamma^*(s)\le\frac{1}{1-\gamma},
\qquad s\in\mathcal S.
\]
Hence finite dimensionality yields a sequence $\gamma_n\uparrow1$ such
that
\[
v_{\gamma_n}^*-v_{\gamma_n}^*(s_0)\mathbf 1
\longrightarrow u_\delta^*,
\qquad
(1-\gamma_n)v_{\gamma_n}^*(s_0)
\longrightarrow g_\delta^*.
\]

For the $S$-rectangular model, the normalized discounted Bellman equation using the operator~\eqref{eq:discounted-bellman-operator}\ is
\[
\begin{aligned}
v_\gamma^*(s)-v_\gamma^*(s_0)
&=
\max_{\phi\in\Delta(\mathcal A)}
\inf_{q_s\in\mathcal P_s}
\sum_{a\in\mathcal A}\sum_{s'\in\mathcal S}
\phi(a)q_{s,a}(s')\\
&\quad\times
\Bigl[
r(s,a,s')-(1-\gamma)v_\gamma^*(s_0)
+\gamma\bigl(v_\gamma^*(s')-v_\gamma^*(s_0)\bigr)
\Bigr].
\end{aligned}
\]
For the SA-rectangular model, using the reduction in Definition~\ref{defn:bellman-operators}, the right-hand side is replaced by
\[
\max_{a\in\mathcal A}
\inf_{q_{s,a}\in\mathcal P_{s,a}}
\sum_{s'\in\mathcal S}q_{s,a}(s')
\Bigl[
r(s,a,s')-(1-\gamma)v_\gamma^*(s_0)
+\gamma\bigl(v_\gamma^*(s')-v_\gamma^*(s_0)\bigr)
\Bigr].
\]
The corresponding one-stage expressions converge uniformly because
\[
\left|(1-\gamma_n)v_{\gamma_n}^*(s_0)-g_\delta^*\right|+\left\|\gamma_n\bigl(v_{\gamma_n}^*-v_{\gamma_n}^*(s_0)\mathbf 1\bigr)-u_\delta^*\right\|_\infty\longrightarrow0.
\]
Since a max--infimum is $1$-Lipschitz with respect to the uniform norm of
its objective, passage to the limit in the corresponding equation gives~\eqref{eq:constant-gain-bellman}:
\[
\mathcal T^*(u_\delta^*)=u_\delta^*+g_\delta^*\mathbf 1
\]
under either rectangularity structure. Lemma~\ref{lem:constant-gain-uniqueness}
then implies that every other solution has the same constant gain
$g_\delta^*$.
\end{proof}

\begin{proof}[Proof of Corollary~\ref{cor:stationary-average-optimal-policy}]
Propositions~\ref{prop:divergence-constant-gain}
and~\ref{prop:metric-constant-gain} provide a solution
$(u_\delta^*,g_\delta^*)$ of~\eqref{eq:constant-gain-bellman}.
By \citet[Theorem~3 and Remark~3]{wang2025bellmanoptimalityaveragerewardrobust},
a stationary policy attaining the maximization in the constant-gain
Bellman equation is robust optimal. Their result is stated for rewards
$r(s,a)$, and the same verification argument applies to our bounded
transition-dependent rewards $r(s,a,s')$.

For the SA-rectangular case,
\[
u_\delta^*(s)+g_\delta^*
=\max_{a\in\mathcal A}\inf_{q_{s,a}\in\mathcal P_{s,a}}
q_{s,a}[r(s,a,\cdot)+u_\delta^*].
\]
Since $\mathcal A$ is finite, for each $s\in\mathcal S$ there exists
\[
a^*(s)\in\arg\max_{a\in\mathcal A}
\inf_{q_{s,a}\in\mathcal P_{s,a}}
q_{s,a}[r(s,a,\cdot)+u_\delta^*].
\]
Define $\pi^*(a\mid s)=\mathbf 1\{a=a^*(s)\}$.
Then $\pi^*$ is stationary deterministic, so
$\pi^*\in\Pi_{\mathrm{SD}}$, and it attains $g_\delta^*$.

For the $S$-rectangular case,
\[
u_\delta^*(s)+g_\delta^*
=\max_{\phi\in\Delta(\mathcal A)}\inf_{q_s\in\mathcal P_s}
\sum_{a\in\mathcal A}q_{s,a}
\!\left[\phi(a)(r(s,a,\cdot)+u_\delta^*)\right].
\]
The objective is continuous in $\phi$, and $\Delta(\mathcal A)$ is compact, so a maximizer exists.
Choose
\[
\phi^*(s)\in\arg\max_{\phi\in\Delta(\mathcal A)}
\inf_{q_s\in\mathcal P_s}
\sum_{a\in\mathcal A}q_{s,a}
\!\left[\phi(a)(r(s,a,\cdot)+u_\delta^*)\right],
\]
and define $\pi^*(\cdot\mid s)=\phi^*(s)$.
Then $\pi^*$ is stationary randomized, so
$\pi^*\in\Pi_{\mathrm{SR}}$, and it attains $g_\delta^*$.

To verify optimality for the kernel-sequence criterion, summing the Bellman inequality shows that either constructed policy has expected $T$-step reward at least $Tg_\delta^*-\operatorname{Span}(u_\delta^*)$ from every initial state under every $\mathbf Q\in(\mathcal P)^{\mathbb N}$. Conversely, compactness and actionwise minimization under SA-rectangularity, or compactness, convexity, and the minimax argument in Lemma~\ref{lem:inf-sup-to-optimal-bellman} under $S$-rectangularity, yield a fixed $Q=(q_{s,a})_{s,a}\in\mathcal P$ such that $q_{s,a}[r(s,a,\cdot)+u_\delta^*]\le u_\delta^*(s)+g_\delta^*$ for all $(s,a)$. Against the constant sequence $(Q,Q,\ldots)$, summing this inequality bounds the expected $T$-step reward of every $\pi\in\Pi_{\mathrm{HD}}$ by $Tg_\delta^*+\operatorname{Span}(u_\delta^*)$. Dividing by $T$ and taking $\limsup$ therefore proves that $\pi^*$ attains the robust optimal gain for every initial state.
\end{proof}

\section{Notations and Dual Forms}
\label{app:notations-dual-forms}
Our algorithms replace the nominal transition distributions with estimates
obtained from samples. We describe the resulting empirical DR-MDP and its
Bellman operators, then introduce the event on which these estimates remain
close to their nominal counterparts. We then state the fixed-policy dual
representations and give their proofs.

For each $(s,a)\in\mathcal{S}\times\mathcal{A}$, let
$S_{s,a}^{(1)},\ldots,S_{s,a}^{(n)}$ be independent samples drawn from
$p_{s,a}$. For a sample outcome $\omega$, define the empirical transition
distribution by
\[
\widehat p_{s,a}^{\,\omega}(s')
:=
\frac{1}{n}
\sum_{i=1}^{n}
\mathbf{1}\{S_{s,a}^{(i)}(\omega)=s'\},
\qquad s'\in\mathcal{S}.
\]

For the observed samples, we suppress $\omega$ and write $\widehat p_{s,a}$.
Since the samples are drawn from $p_{s,a}$, with probability one,
$\operatorname{supp}(\widehat p_{s,a})\subseteq\operatorname{supp}(p_{s,a})$
holds simultaneously for all $(s,a)\in\mathcal S\times\mathcal A$.
Replacing each nominal center $p_{s,a}$ in the uncertainty sets introduced
in the preliminaries by $\widehat p_{s,a}$ gives the empirical DR-MDP.
The uncertainty radius $\delta$ and the rectangularity structure are
unchanged, so the empirical Bellman operators describe the same local
controller--adversary games with uncertainty centered at the estimated
transition distributions.

\textbf{SA-rectangular model.} The empirical uncertainty set factors
over state-action pairs:
\[
\widehat{\mathcal P}
=
\prod_{(s,a)\in\mathcal S\times\mathcal A}
\widehat{\mathcal P}_{s,a}.
\]
The adversary again minimizes each action's payoff independently.
Replacing the uncertainty sets in the
optimal discounted Bellman operator in Definition~\ref{defn:bellman-operators}\ gives
\[
\widehat{\mathcal T}_\gamma^*(v)(s)
:=
\max_{a\in\mathcal A}
\inf_{q_{s,a}\in\widehat{\mathcal P}_{s,a}}
q_{s,a}[r(s,a,\cdot)+\gamma v].
\]
The empirical optimal discounted value $\widehat v_\gamma^*$ is its fixed
point:
\begin{equation}
\widehat v_\gamma^*
=
\widehat{\mathcal T}_\gamma^*(\widehat v_\gamma^*).
\label{eq:sa-dr-dmdp-empirical-optimal-bellman}
\end{equation}
Removing the discount factor gives the corresponding empirical
average-reward Bellman operator, as in Definition~\ref{defn:bellman-operators}:
\[
\widehat{\mathcal T}^*(v)(s)
:=
\max_{a\in\mathcal A}
\inf_{q_{s,a}\in\widehat{\mathcal P}_{s,a}}
q_{s,a}[r(s,a,\cdot)+v].
\]

\textbf{$S$-rectangular model.} Here the empirical uncertainty set factors
over states:
\[
\widehat{\mathcal P}
=
\prod_{s\in\mathcal S}\widehat{\mathcal P}_s.
\]
The action-indexed transition distributions remain jointly constrained
at each state. For a controller's action distribution
$\phi\in\Delta(\mathcal A)$, the adversary's response is described by
the empirical discounted Bellman operator
\[
\widehat{\mathcal T}_\gamma^\phi(v)(s)
:=
\inf_{q_s\in\widehat{\mathcal P}_s}
\sum_{a\in\mathcal A}
q_{s,a}[\phi(a)(r(s,a,\cdot)+\gamma v)].
\]
Maximizing over $\phi$ gives the empirical counterpart of the optimal
discounted Bellman operator~\eqref{eq:discounted-bellman-operator}:
\[
\widehat{\mathcal T}_\gamma^*(v)(s)
:=
\max_{\phi\in\Delta(\mathcal A)}
\widehat{\mathcal T}_\gamma^\phi(v)(s).
\]
Its fixed point again determines the empirical optimal discounted value:
\begin{equation}
\widehat v_\gamma^*
=
\widehat{\mathcal T}_\gamma^*(\widehat v_\gamma^*).
\label{eq:s-dr-dmdp-empirical-optimal-bellman}
\end{equation}
The same local game without discounting gives the empirical
average-reward Bellman operator, as in \eqref{eq:average-bellman-operator}:
\[
\widehat{\mathcal T}^*(v)(s)
:=
\max_{\phi\in\Delta(\mathcal A)}
\inf_{q_s\in\widehat{\mathcal P}_s}
\sum_{a\in\mathcal A}
q_{s,a}[\phi(a)(r(s,a,\cdot)+v)].
\]

\textbf{Fixed-policy discounted Bellman operators.}
For $\pi\in\Pi_{\mathrm{SR}}$ and $\gamma\in(0,1)$, define
$\mathcal T_\gamma^\pi:\mathbb R^{|\mathcal S|}\to\mathbb R^{|\mathcal S|}$
under SA-rectangularity by
\begin{equation}
\label{eq:sa-discount-policy-bellman-operator}
\mathcal T_\gamma^\pi(v)(s):=\sum_{a\in\mathcal A}\pi(a| s)\inf_{q_{s,a}\in\mathcal P_{s,a}}\sum_{s'\in\mathcal S}q_{s,a}(s')\bigl(r(s,a,s')+\gamma v(s')\bigr),
\end{equation}
and under S-rectangularity by
\begin{equation}
\label{eq:s-discount-policy-bellman-operator}
\mathcal T_\gamma^\pi(v)(s):=\inf_{q_s\in\mathcal P_s}\sum_{a\in\mathcal A}\sum_{s'\in\mathcal S}\pi(a\mid s)q_{s,a}(s')\bigl(r(s,a,s')+\gamma v(s')\bigr).
\end{equation}
The empirical fixed-policy discounted operator
$\widehat{\mathcal T}_\gamma^\pi$ is obtained from
$\mathcal T_\gamma^\pi$ by replacing the true uncertainty sets with
their empirical counterparts while keeping $\pi$ unchanged.
Under either structure,
$\mathcal T_\gamma^*(v)=\max_{\pi\in\Pi_{\mathrm{SR}}}\mathcal T_\gamma^\pi(v)$
componentwise, and the same identity holds for the empirical operators.

Under either rectangularity structure, comparing the empirical and true
Bellman operators requires controlling the error in their centers.
We therefore introduce the good event $\Omega_{n,\eta}$, on which all
empirical transition probabilities are uniformly close to their nominal
counterparts in relative error.

\begin{definition}
\label{def:good-event}
For $\eta>0$, define the good event on which every empirical transition
distribution remains close to its nominal transition distribution in relative
error by
\[
\Omega_{n,\eta}
:=
\bigcap_{(s,a)\in\mathcal{S}\times\mathcal{A}}
\left\{
\omega:
\left\|
\frac{\widehat p_{s,a}^{\,\omega}-p_{s,a}}{p_{s,a}}
\right\|_{L^{\infty}(p_{s,a})}
\le \eta
\right\},
\]
where
\[
\left\|
\frac{\widehat p_{s,a}^{\,\omega}-p_{s,a}}{p_{s,a}}
\right\|_{L^{\infty}(p_{s,a})}
:=
\max_{s'\in\operatorname{supp}(p_{s,a})}
\left|
\frac{\widehat p_{s,a}^{\,\omega}(s')-p_{s,a}(s')}
{p_{s,a}(s')}
\right|.
\]
\end{definition}

\begin{proposition}
\label{prop:good-event-concentration}
For any $\beta\in(0,1)$, choose
\[
\eta
=
\frac{1}{3np_{\wedge}}
\log\frac{2|\mathcal{S}|^2|\mathcal{A}|}{\beta}
+
\sqrt{
\frac{2}{np_{\wedge}}
\log\frac{2|\mathcal{S}|^2|\mathcal{A}|}{\beta}
}.
\]
Then the good event satisfies
\[
\mathbb{P}\!\left(\Omega_{n,\eta}^{c}\right)\le \beta.
\]
\end{proposition}
This proposition is the same as \citet[Lemma~A.2]{chen2025sample},
written in the present notation.

\subsection{Uncertainty Sets and Robust Bellman Oracles}
\label{subsec:uncertainty-sets-and-robust-bellman-oracles}
We give the fixed-policy dual representations for the uncertainty sets
defined in Section~\ref{subsec:robust-mdps-and-rectangularity}.
For fixed $\gamma\in(0,1)$ and $v:\mathcal S\to\mathbb{R}$, write
\[
    z_{s,a}(x)
    :=
    r(s,a,x)+\gamma v(x),
    \qquad x\in\mathcal S.
\]
The dependence of $z_{s,a}$ on $\gamma$ and $v$ is suppressed when it
is clear from context.

\textbf{Kullback-Leibler (KL) Divergence.}

\begin{lemma}[{Lemma~C.1 in \citet{chen2025sample}}]
\label{lem:sa-rec-kl-dual}
Suppose $\cP$ is SA-rectangular and $U=D_{\mathrm{KL}}$ with $\delta > 0$. For any $\pi\in\Pi_{\mathrm{SR}}$, $s\in\mathcal S$, and $v:\mathcal S\to\mathbb R$,
the fixed-policy Bellman operator satisfies
\[
\mathcal T_\gamma^\pi(v)(s)
=\sum_{a\in\mathcal A}\pi(a\mid s)\sup_{\lambda_a\ge 0}\left\{-\lambda_a\delta-\lambda_a\log\bigl(p_{s,a}[e^{-z_{s,a}/\lambda_a}]\bigr)\right\}.
\]
\end{lemma}

% For each fixed $(s,a)$, apply the cited $S$-rectangular result to a singleton action set with nominal distribution $p_{s,a}$, payoff $z_{s,a}$, and radius $\rho=\delta$; the divergence constraint then becomes the local SA-rectangular constraint. Substituting these action-wise dual representations into the fixed-policy Bellman operator gives the displayed weighted sum.

\begin{lemma}[{Lemma~1 in \citet{li_wang_si_2026_s_rectangular}}]
\label{lem:s-rec-kl-dual}
Suppose $\cP$ is S-rectangular and $U=D_{\mathrm{KL}}$ with $\delta >0$. The fixed-policy Bellman operator
has the following representation for any $\pi\in\Pi_{\mathrm{SR}}$, $s\in\mathcal S$, and $v:\mathcal S\to\mathbb R$:
\[
\mathcal T_\gamma^\pi(v)(s)
=\sup_{\lambda>0}\left\{-\lambda|\mathcal A|\delta-\lambda\sum_{a\in\mathcal A}\log p_{s,a}\left[\exp\left(-\frac{\pi(a\mid s)z_{s,a}}{\lambda}\right)\right]\right\}.
\]
\end{lemma}

\textbf{$f_k$-Divergence.}

\begin{lemma}[{Lemma~F.1 in \citet{chen2025sample}}]
\label{lem:sa-rec-fk-dual}
Suppose $\cP$ is SA-rectangular and $U=D_{f_k}$ with $\delta >0$. Let $k^*:=k/(k-1)$ denote the conjugate exponent, then, for any $\pi\in\Pi_{\mathrm{SR}}$, $s\in\mathcal S$, and $v:\mathcal S\to\mathbb R$,
the fixed-policy Bellman operator satisfies
\[
\mathcal T_\gamma^\pi(v)(s)
=\sum_{a\in\mathcal A}\pi(a\mid s)\sup_{\theta_a\in\mathbb R}\left\{\theta_a-\bigl(1+k(k-1)\delta\bigr)^{1/k}\bigl(p_{s,a}[(\theta_a-z_{s,a})_+^{k^*}]\bigr)^{1/k^*}\right\}.
\]
\end{lemma}

\begin{lemma}[{Lemma~2 in \citet{li_wang_si_2026_s_rectangular}}]
\label{lem:s-rec-fk-dual}
Suppose $\cP$ is S-rectangular and $U=D_{f_k}$ with $\delta >0$. Let $k^*:=k/(k-1)$ denote the conjugate exponent, then, for any $\pi\in\Pi_{\mathrm{SR}}$, $s\in\mathcal S$, and $v:\mathcal S\to\mathbb R$,
the fixed-policy Bellman operator satisfies
\[
\adjustbox{max width=\linewidth}{$\displaystyle
\mathcal T_\gamma^\pi(v)(s)=\sup_{\theta\in\mathbb R^{|\mathcal A|}}\Biggl\{\sum_{a\in\mathcal A}\theta_a-|\mathcal A|^{1/k}\bigl(1+k(k-1)\delta\bigr)^{1/k}\left(\sum_{a\in\mathcal A}p_{s,a}\left[(\theta_a-\pi(a\mid s)z_{s,a})_+^{k^*}\right]\right)^{1/k^*}\Biggr\}.
$}
\]
\end{lemma}

\textbf{Total-Variation (TV) distance.}

\begin{lemma}[{Lemma~6 in \citet{iyengar2005robust}}]
\label{lem:sa-rec-tv-dual}
For $\mathcal P=\prod_{(s,a)\in\mathcal S\times\mathcal A}\mathcal P_{s,a}(d_{\mathrm{TV}},\delta)$
with $\delta\in(0,\infty)$, the fixed-policy Bellman operator
can be written as follows for any $\pi\in\Pi_{\mathrm{SR}}$, $s\in\mathcal S$, and $v:\mathcal S\to\mathbb R$:
\[
\mathcal T_\gamma^\pi(v)(s)
=\sum_{a\in\mathcal A}\pi(a\mid s)\sup_{\mu_a\ge0}\Bigl\{p_{s,a}[z_{s,a}-\mu_a]-\delta\operatorname{Span}(z_{s,a}-\mu_a)\Bigr\},
\]
where each $\mu_a$ is a nonnegative vector on $\mathcal S$.
\end{lemma}

Apply the cited $L_1$-ball duality to each payoff $z_{s,a}$ with $L_1$ radius $2\delta$, using $d_{\mathrm{TV}}(q,p)=\frac12\|q-p\|_1$, and substitute the resulting dual values into the fixed-policy Bellman operator.

\begin{lemma}
\label{lem:s-rec-tv-dual}
Let $\mathcal P=\prod_{s\in\mathcal S}\mathcal P_s(d_{\mathrm{TV}},\delta)$
with $\delta\in(0,\infty)$. For any $\pi\in\Pi_{\mathrm{SR}}$, $s\in\mathcal S$, and $v:\mathcal S\to\mathbb R$,
the fixed-policy Bellman operator satisfies
\[
\adjustbox{max width=\linewidth}{$\displaystyle
\mathcal T_\gamma^\pi(v)(s)=\sup_{\mu\in(\mathbb R_+^{|\mathcal S|})^{\mathcal A}}\Biggl\{\sum_{a\in\mathcal A}p_{s,a}[\pi(a\mid s)z_{s,a}-\mu_a]-|\mathcal A|\delta\max_{a\in\mathcal A}\operatorname{Span}(\pi(a\mid s)z_{s,a}-\mu_a)\Biggr\},
$}
\]
where each $\mu_a$ is a nonnegative vector on $\mathcal S$.
\end{lemma}

\textbf{Wasserstein distance.}

\begin{lemma}[{Lemma~5.1 in \citet{wang_model-free_2023}}]
\label{lem:sa-rec-wasserstein-dual}
Suppose $\mathcal P=\prod_{(s,a)\in\mathcal S\times\mathcal A}\mathcal P_{s,a}(W_\ell,\delta)$
with $\delta\in(0,\infty)$. Then, for any $\pi\in\Pi_{\mathrm{SR}}$, $s\in\mathcal S$, and $v:\mathcal S\to\mathbb R$,
the fixed-policy Bellman operator satisfies
\[
\mathcal T_\gamma^\pi(v)(s)
=\sum_{a\in\mathcal A}\pi(a\mid s)\sup_{\lambda_a\ge0}\left\{-\lambda_a\delta^{\ell}+p_{s,a}\left[\inf_{y\in\mathcal S}\left(z_{s,a}(y)+\lambda_a\rho(\cdot,y)^{\ell}\right)\right]\right\}.
\]
\end{lemma}

The cited support-function identity applies to any payoff on $\mathcal S$; substitute $z_{s,a}=r(s,a,\cdot)+\gamma v$ for its value vector and combine the action-wise dual values with weights $\pi(a\mid s)$ in the fixed-policy Bellman operator.

\begin{lemma}
\label{lem:s-rec-wasserstein-dual}
For $\mathcal P=\prod_{s\in\mathcal S}\mathcal P_s(W_\ell,\delta)$
with $\delta\in(0,\infty)$, the fixed-policy Bellman operator
has the following dual form for any $\pi\in\Pi_{\mathrm{SR}}$, $s\in\mathcal S$, and $v:\mathcal S\to\mathbb R$:
\[
\adjustbox{max width=\linewidth}{$\displaystyle
\mathcal T_\gamma^\pi(v)(s)=\sup_{\lambda\ge0}\Biggl\{-\lambda|\mathcal A|\delta^{\ell}+\sum_{a\in\mathcal A}p_{s,a}\left[\inf_{y\in\mathcal S}\left(\pi(a\mid s)z_{s,a}(y)+\lambda\rho(\cdot,y)^{\ell}\right)\right]\Biggr\}.
$}
\]
\end{lemma}

\begin{proof}[Proof of
Lemmas~\ref{lem:sa-rec-kl-dual} and~\ref{lem:sa-rec-fk-dual}]
Fix $\pi\in\Pi_{\mathrm{SR}}$, $s\in\mathcal S$, $\gamma\in(0,1)$, and
$v:\mathcal S\to\mathbb R$. Although
\citet{chen2025sample} consider rewards
$r(s,a)$, their Lemmas~C.1 and~F.1 give dual representations for the
worst-case expectation of an arbitrary real-valued vector. We apply these
identities to $z_{s,a}(x)=r(s,a,x)+\gamma v(x)$, which remains fixed while
the transition distribution varies. For $U\in\{D_{\mathrm{KL}},D_{f_k}\}$,
\[
\begin{aligned}
\mathcal T_\gamma^\pi(v)(s)
&=\sum_{a\in\mathcal A}\pi(a\mid s)
\inf_{q_{s,a}\in\mathcal P_{s,a}(U,\delta)}
\sum_{x\in\mathcal S}q_{s,a}(x)\bigl[r(s,a,x)+\gamma v(x)\bigr]\\
&=\sum_{a\in\mathcal A}\pi(a\mid s)
\inf_{q_{s,a}\in\mathcal P_{s,a}(U,\delta)}
\sum_{x\in\mathcal S}q_{s,a}(x)z_{s,a}(x).
\end{aligned}
\]
Applying their Lemma~C.1 to each inner infimum gives the KL representation
in Lemma~\ref{lem:sa-rec-kl-dual}.
For $f_k$ uncertainty, their Lemma~F.1 yields
Lemma~\ref{lem:sa-rec-fk-dual}\
with the same coefficient $(1+k(k-1)\delta)^{1/k}$.
The dual variables are chosen separately for each action.
Transition-dependent rewards thus enter these identities through the
payoff vector $z_{s,a}$. In
Lemma~\ref{lem:sa-rec-kl-dual},
the expression at $\lambda_a=0$ is interpreted by its limit as
$\lambda_a\downarrow0$.
\end{proof}

We next prove the two $S$-rectangular metric dual representations in
Lemmas~\ref{lem:s-rec-tv-dual} and~\ref{lem:s-rec-wasserstein-dual}.

\begin{proof}[Proof of
Lemma~\ref{lem:s-rec-tv-dual}]
\phantomsection\label{proof:s-rec-tv-dual}
Fix $\pi\in\Pi_{\mathrm{SR}}$ and $s\in\mathcal S$, set $\phi=\pi(\cdot\mid s)$, and write
\[
z_{s,a}(x):=r(s,a,x)+\gamma v(x),
\qquad a\in\mathcal A,\ x\in\mathcal S.
\]
The primal problem is
\begin{equation}
\begin{aligned}
\mathcal T_\gamma^\pi(v)(s)
&:=
\inf_{q_s}\sum_{a\in\mathcal A}q_{s,a}[\phi(a)z_{s,a}]\\
&\text{s.t.}\quad q_{s,a}\in\Delta(\mathcal S),\quad a\in\mathcal A,\\
&\phantom{\text{s.t.}\quad}\sum_{a\in\mathcal A}
\lVert q_{s,a}-p_{s,a}\rVert_1
\le 2|\mathcal A|\delta.
\end{aligned}
\label{eq:lemma-6-tv-primal}
\end{equation}
Put
\[
y_{s,a}:=q_{s,a}-p_{s,a}.
\]
Since $q_{s,a}=p_{s,a}+y_{s,a}$,
\[
q_{s,a}\in\Delta(\mathcal S)
\Longleftrightarrow
\left\{
\begin{array}{l}
y_{s,a}\ge-p_{s,a},\\
\displaystyle\sum_{x\in\mathcal S}y_{s,a}(x)=0,
\end{array}
\right.
\qquad
\lVert q_{s,a}-p_{s,a}\rVert_1=\lVert y_{s,a}\rVert_1.
\]
Thus the primal problem in \eqref{eq:lemma-6-tv-primal} becomes
\[
\begin{aligned}
\mathcal T_\gamma^\pi(v)(s)
&=\sum_{a\in\mathcal A}p_{s,a}[\phi(a)z_{s,a}]\\
&\quad+\min_y\sum_{a\in\mathcal A}\sum_{x\in\mathcal S}
y_{s,a}(x)\phi(a)z_{s,a}(x)\\
&\text{s.t.}\quad
\sum_{a\in\mathcal A}\lVert y_{s,a}\rVert_1
\le 2|\mathcal A|\delta,\\
&\phantom{\text{s.t.}\quad}\sum_{x\in\mathcal S}y_{s,a}(x)=0,
\quad a\in\mathcal A,\\
&\phantom{\text{s.t.}\quad}y_{s,a}\ge-p_{s,a},\quad a\in\mathcal A.
\end{aligned}
\]
Introduce $\mu_a\in\mathbb R_+^{|\mathcal S|}$ and
$c_a\in\mathbb R$. The dual-norm identity gives
\[
\begin{aligned}
&\min_{\sum_a\lVert y_{s,a}\rVert_1\le2|\mathcal A|\delta}
\sum_{a,x}y_{s,a}(x)
\bigl(\phi(a)z_{s,a}(x)-\mu_a(x)-c_a\bigr)\\
&\qquad=-2|\mathcal A|\delta
\max_{a,x}
\bigl|\phi(a)z_{s,a}(x)-\mu_a(x)-c_a\bigr|.
\end{aligned}
\]
By the $L_1$-ball duality of
\citet[Lemma~6]{iyengar2005robust} and
finite-dimensional linear-programming strong duality, the primal problem can
then be written as
\[
\begin{aligned}
\mathcal T_\gamma^\pi(v)(s)
&=\sum_{a\in\mathcal A}p_{s,a}[\phi(a)z_{s,a}]\\
&\quad+\sup_{\substack{\mu\in(\mathbb R_+^{|\mathcal S|})^{\mathcal A}\\
c\in\mathbb R^{\mathcal A}}}
\Biggl\{
-\sum_{a\in\mathcal A}p_{s,a}[\mu_a]
-2|\mathcal A|\delta
\max_{a\in\mathcal A}
\lVert\phi(a)z_{s,a}-\mu_a-c_a\mathbf 1\rVert_\infty
\Biggr\}.
\end{aligned}
\]
\[
\begin{aligned}
\inf_{c\in\mathbb R^{\mathcal A}}
\max_{a\in\mathcal A}
\lVert\phi(a)z_{s,a}-\mu_a-c_a\mathbf 1\rVert_\infty
&=\frac12\max_{a\in\mathcal A}
\operatorname{Span}(\phi(a)z_{s,a}-\mu_a).
\end{aligned}
\]
Here we use
\[
\inf_{c\in\mathbb R}\lVert v-c\mathbf 1\rVert_\infty
=\frac12\operatorname{Span}(v).
\]
\[
\mathcal T_\gamma^\pi(v)(s)
=\sup_{\mu\in(\mathbb R_+^{|\mathcal S|})^{\mathcal A}}
\Biggl\{
\sum_{a\in\mathcal A}p_{s,a}[\pi(a\mid s)z_{s,a}-\mu_a]
-|\mathcal A|\delta
\max_{a\in\mathcal A}
\operatorname{Span}(\pi(a\mid s)z_{s,a}-\mu_a)
\Biggr\},
\]
which is the claimed dual representation.

\end{proof}

\begin{proof}[Proof of
Lemma~\ref{lem:s-rec-wasserstein-dual}]
\phantomsection\label{proof:s-rec-wasserstein-dual}
Fix $\pi\in\Pi_{\mathrm{SR}}$ and $s\in\mathcal S$, and set $\phi=\pi(\cdot\mid s)$. The primal problem is
\[
\begin{aligned}
\mathcal T_\gamma^\pi(v)(s)
&=\inf_{q_s}\sum_{a\in\mathcal A}q_{s,a}[\phi(a)z_{s,a}]\\
&\text{s.t.}\quad q_{s,a}\in\Delta(\mathcal S),\quad a\in\mathcal A,\\
&\phantom{\text{s.t.}\quad}\sum_{a\in\mathcal A}
W_\ell(q_{s,a},p_{s,a})^\ell
\le |\mathcal A|\delta^\ell.
\end{aligned}
\]
Let $X:=\mathcal A\times\mathcal S$ and define
\[
\bar p_s(a,x):=\frac{1}{|\mathcal A|}p_{s,a}(x),\qquad
\bar q_s(a,y):=\frac{1}{|\mathcal A|}q_{s,a}(y),
\]
and
\[
c((a,x),(b,y)):=
\begin{cases}
\rho(x,y)^\ell,&a=b,\\
+\infty,&a\ne b.
\end{cases}
\]
\[
\inf_{\xi\in\Gamma(\bar p_s,\bar q_s)}
\sum_{a,b\in\mathcal A}\sum_{x,y\in\mathcal S}
c((a,x),(b,y))\xi((a,x),(b,y))
=\frac1{|\mathcal A|}\sum_{a\in\mathcal A}
W_\ell(q_{s,a},p_{s,a})^\ell,
\]
and
\[
\sum_{a,y}\bar q_s(a,y)\phi(a)z_{s,a}(y)
=\frac1{|\mathcal A|}\sum_{a\in\mathcal A}
q_{s,a}[\phi(a)z_{s,a}].
\]
Conversely, every finite-cost coupling must preserve the action coordinate.
Hence
\[
\sum_y\bar q_s(a,y)=\sum_x\bar p_s(a,x)=\frac1{|\mathcal A|},
\qquad a\in\mathcal A,
\]
so $q_{s,a}(y):=|\mathcal A|\bar q_s(a,y)$ belongs to
$\Delta(\mathcal S)$. Thus, the primal problem is $|\mathcal A|$ times
the corresponding product-space problem with radius $\delta^\ell$.
We adopt the convention $0\cdot(+\infty)=+\infty$. For each
$\lambda\ge0$, define
\[
\varphi_\lambda\bigl((a,x),(b,y)\bigr)
:=-\phi(b)z_{s,b}(y)-\lambda c\bigl((a,x),(b,y)\bigr).
\]
Since $X$ is finite, the pointwise supremum of $\varphi_\lambda$ is
measurable. Choosing a maximizing second coordinate for each first coordinate
yields a coupling that attains the expectation of this pointwise supremum,
while the reverse inequality holds pointwise. Thus $\varphi_\lambda$ satisfies
the interchangeability principle. By
\citet[Theorem~2.2]{zhang_yang_gao_2024_wasserstein_duality}, applied to the payoff
$(a,y)\mapsto-\phi(a)z_{s,a}(y)$,
\[
\begin{aligned}
\mathcal T_\gamma^\pi(v)(s)
&=\sup_{\lambda\ge0}\Biggl\{
-\lambda|\mathcal A|\delta^\ell
+|\mathcal A|\sum_{a,x}\bar p_s(a,x)
\inf_{b,y}\left(
\phi(b)z_{s,b}(y)+\lambda c((a,x),(b,y))
\right)\Biggr\}.
\end{aligned}
\]
Since $c((a,x),(b,y))=+\infty$ for $b\ne a$,
\[
\inf_{b,y}\left(\phi(b)z_{s,b}(y)+\lambda c((a,x),(b,y))\right)
=\inf_{y\in\mathcal S}\left(\phi(a)z_{s,a}(y)+\lambda\rho(x,y)^\ell\right).
\]
Substituting $\bar p_s(a,x)=|\mathcal A|^{-1}p_{s,a}(x)$ gives
\[
\begin{aligned}
\mathcal T_\gamma^\pi(v)(s)
&=\sup_{\lambda\ge0}\Biggl\{
-\lambda|\mathcal A|\delta^\ell\\
&\qquad+\sum_{a\in\mathcal A}p_{s,a}\left[
\inf_{y\in\mathcal S}\left(\pi(a\mid s)z_{s,a}(y)
+\lambda\rho(\cdot,y)^\ell\right)\right]
\Biggr\},
\end{aligned}
\]
which is the claimed dual representation.
\end{proof}
\begin{remark}
The proof of Lemma~\ref{lem:s-rec-wasserstein-dual} reduces the
$S$-rectangular Wasserstein problem to a standard Wasserstein
distributionally robust optimization problem on the enlarged space
$\mathcal A\times\mathcal S$. The key step is to assign infinite
transportation cost between points corresponding to different actions.
Consequently, any finite-cost coupling is forced to preserve the action
coordinate, while transportation within each action incurs exactly the
original state-space cost $\rho^\ell$. Under the uniform action marginal
$1/|\mathcal A|$, the resulting Kantorovich cost becomes
\[
K_c(\bar p_s,\bar q_s)
=
\frac{1}{|\mathcal A|}
\sum_{a\in\mathcal A}W_\ell(q_{s,a},p_{s,a})^\ell.
\]
Thus, the coupled $S$-rectangular constraint is represented by a single
Wasserstein ball on $\mathcal A\times\mathcal S$. Standard Wasserstein
duality can then be applied directly, and the infinite cross-action cost
reduces the inner infimum back to an action-wise infimum over
$y\in\mathcal S$. Thus, the proof converts the coupling across actions in the
uncertainty budget into an ordinary Wasserstein dual problem without
separating the budget action by action.
\end{remark}

\section{Perturbation Bounds}
\label{sec:perturbation-bounds}

\newcommand{\UncertaintySetsAndRobustBellmanOraclesReference}{Appendix~\ref{subsec:uncertainty-sets-and-robust-bellman-oracles}}
\newcommand{\SAKLDualReference}{Lemma~\ref{lem:sa-rec-kl-dual}}
\newcommand{\SKLDualReference}{Lemma~\ref{lem:s-rec-kl-dual}}
\newcommand{\SAFkDualReference}{Lemma~\ref{lem:sa-rec-fk-dual}}
\newcommand{\SFkDualReference}{Lemma~\ref{lem:s-rec-fk-dual}}
\newcommand{\SATVDualReference}{Lemma~\ref{lem:sa-rec-tv-dual}}
\newcommand{\STVDualReference}{Lemma~\ref{lem:s-rec-tv-dual}}
\newcommand{\SAWassersteinDualReference}{Lemma~\ref{lem:sa-rec-wasserstein-dual}}
\newcommand{\SWassersteinDualReference}{Lemma~\ref{lem:s-rec-wasserstein-dual}}

Based on the dual forms in
\UncertaintySetsAndRobustBellmanOraclesReference, we analyze
perturbation bounds for the fixed-policy Bellman operator under eight uncertainty-set
and rectangularity combinations. Specifically, for any
\(v\in\mathbb R^{|\mathcal S|}\), \(\pi\in\Pi_{\mathrm{SR}}\), and \(s\in\mathcal S\), we aim to bound
$\left|\widehat{\mathcal T}_\gamma^\pi(v)(s)-\mathcal T_\gamma^\pi(v)(s)\right|$
in terms of \(1+\operatorname{Span}(v)\), thereby quantifying the effect of
replacing the true transition distributions with their empirical counterparts.
The resulting bounds are established separately in
Corollaries~\ref{cor:sa-rec-kl-perturbation}--\ref{cor:s-rec-wasserstein-perturbation}.
Their pointwise estimates use only the relative-error inequalities on
$\Omega_{n,\eta}$ and hold for arbitrary $v$, $\pi$, and $\gamma\in(0,1)$;
thus the bounds are simultaneous on this event, without a union bound over policies.

\begin{lemma}
\label{lem:tv-span}
For any two probability distributions \(p,q\in\Delta(\mathcal S)\) and any
bounded vector \(h:\mathcal S\to\mathbb R\),
\[
\left|\sum_{x\in\mathcal S}(p(x)-q(x))h(x)\right|
\le d_{\mathrm{TV}}(p,q)\operatorname{Span}(h).
\]
\end{lemma}

\begin{proof}
Let \(M=\max_{x\in\mathcal S}h(x)\) and \(m=\min_{x\in\mathcal S}h(x)\), and
set \(\bar h(x)=h(x)-(M+m)/2\). Then
\(\sum_x (p(x)-q(x))\bar h(x)=\sum_x (p(x)-q(x))h(x)\), and
\[
\max_{x\in\mathcal S}|\bar h(x)|
\;=\;
\frac{M-m}{2}
=\frac{\operatorname{Span}(h)}{2}.
\]
Hence
\[
\left|\sum_x(p-q)(x)h(x)\right|
\le\sum_x |p(x)-q(x)|\,|\bar h(x)|
\le \|p-q\|_{1}\frac{\operatorname{Span}(h)}{2}
=d_{\mathrm{TV}}(p,q)\operatorname{Span}(h).
\]
The last equality uses \(d_{\mathrm{TV}}(p,q)=\tfrac12\|p-q\|_1\).
\end{proof}

We will use this inequality repeatedly below.

\begin{corollary}
\label{cor:sa-rec-kl-perturbation}
Under the SA-rectangular KL uncertainty set, on the good event
$\Omega_{n,\eta}$ with $\eta\le 1/2$, simultaneously for every
$v\in\mathbb R^{|\mathcal S|}$, $\pi\in\Pi_{\mathrm{SR}}$, $s\in\mathcal S$, and $\gamma\in(0,1)$,
\[
\left|
\widehat{\mathcal T}_\gamma^\pi(v)(s)-\mathcal T_\gamma^\pi(v)(s)
\right|
\le3\eta\bigl(1+\operatorname{Span}(v)\bigr).
\]
\end{corollary}

\begin{proof}[Proof of Corollary~\ref{cor:sa-rec-kl-perturbation}]
Fix $\pi\in\Pi_{\mathrm{SR}}$ and $s\in\mathcal S$. The fixed-policy operator definition and its empirical counterpart give
\[
\mathcal T_{\gamma}^{\pi}(v)(s)
=\sum_{a\in\mathcal A}\pi(a\mid s)
\inf_{q_{s,a}\in\mathcal P_{s,a}(D_{\mathrm{KL}},\delta)}q_{s,a}[r(s,a,\cdot)+\gamma v],
\]
\[
\widehat{\mathcal T}_{\gamma}^{\pi}(v)(s)
=\sum_{a\in\mathcal A}\pi(a\mid s)
\inf_{q_{s,a}\in\widehat{\mathcal P}_{s,a}(D_{\mathrm{KL}},\delta)}q_{s,a}[r(s,a,\cdot)+\gamma v].
\]
Using the dual representation in
\SAKLDualReference, on
$\Omega_{n,\eta}$, for each $a\in\mathcal A$,
Chen et al.'s Lemma~C.1 provides the KL primal--dual equivalence, while
their Lemma~C.6 bounds the change of the resulting dual objective when
$p_{s,a}$ is replaced by $\widehat p_{s,a}$; see
\citet[Lemmas~C.1 and~C.6]{chen2025sample}.
\[
\begin{aligned}
&\left|
\inf_{q_{s,a}\in\widehat{\mathcal P}_{s,a}(D_{\mathrm{KL}},\delta)}q_{s,a}[r(s,a,\cdot)+\gamma v]
-\inf_{q_{s,a}\in\mathcal P_{s,a}(D_{\mathrm{KL}},\delta)}q_{s,a}[r(s,a,\cdot)+\gamma v]
\right|\\
&\quad\le 3\eta\,\operatorname{Span}(r(s,a,\cdot)+\gamma v).
\end{aligned}
\]
Weighting these action-wise estimates by $\pi(a\mid s)$ gives
\[
\begin{aligned}
&\left|
\widehat{\mathcal T}_{\gamma}^{\pi}(v)(s)-\mathcal T_{\gamma}^{\pi}(v)(s)
\right|\\
&\le
\sum_{a\in\mathcal A}\pi(a\mid s)
\left|
\inf_{q_{s,a}\in\widehat{\mathcal P}_{s,a}(D_{\mathrm{KL}},\delta)}q_{s,a}[r(s,a,\cdot)+\gamma v]
-\inf_{q_{s,a}\in\mathcal P_{s,a}(D_{\mathrm{KL}},\delta)}q_{s,a}[r(s,a,\cdot)+\gamma v]
\right|
\\
&\le 3\eta\sum_{a\in\mathcal A}\pi(a\mid s)
\operatorname{Span}(r(s,a,\cdot)+\gamma v)\\
&\le 3\eta\bigl(1+\operatorname{Span}(v)\bigr).
\end{aligned}
\]
The last inequality uses $\sum_a\pi(a\mid s)=1$.
\end{proof}

\begin{corollary}
\label{cor:s-rec-kl-perturbation}
Under the $S$-rectangular KL uncertainty set, on the good event
$\Omega_{n,\eta}$ with $\eta\le 1/2$, simultaneously for every
$v\in\mathbb R^{|\mathcal S|}$, $\pi\in\Pi_{\mathrm{SR}}$, $s\in\mathcal S$, and $\gamma\in(0,1)$,
\[
\left|
\widehat{\mathcal T}_\gamma^\pi(v)(s)-\mathcal T_\gamma^\pi(v)(s)
\right|
\le3\eta\bigl(1+\operatorname{Span}(v)\bigr).
\]
\end{corollary}

\begin{proof}[Proof of Corollary~\ref{cor:s-rec-kl-perturbation}]
Fix $\pi\in\Pi_{\mathrm{SR}}$ and $s\in\mathcal S$, and set $\phi=\pi(\cdot\mid s)$. Using the dual representation in
\SKLDualReference, the common term
$-\lambda|\mathcal A|\delta$ cancels, and hence
\[
\begin{aligned}
&\left|
\widehat{\mathcal T}_\gamma^\pi(v)(s)
-\mathcal T_\gamma^\pi(v)(s)
\right|\\
&\quad\le
\sup_{\lambda>0}
\left|
\sum_{a\in\mathcal A}\lambda
\log \widehat p_{s,a}\left[
\exp\left(-\frac{\phi(a)(r(s,a,\cdot)+\gamma v)}{\lambda}\right)
\right]
\right.\\
&\hspace{42mm}\left.
-\sum_{a\in\mathcal A}\lambda
\log p_{s,a}\left[
\exp\left(-\frac{\phi(a)(r(s,a,\cdot)+\gamma v)}{\lambda}\right)
\right]
\right|\\
&\quad\le
\sum_{a\in\mathcal A}
\sup_{\lambda>0}\left|
\lambda\log \widehat p_{s,a}\left[
\exp\left(-\frac{\phi(a)(r(s,a,\cdot)+\gamma v)}{\lambda}\right)
\right]
\right.\\
&\hspace{42mm}\left.
-\lambda\log p_{s,a}\left[
\exp\left(-\frac{\phi(a)(r(s,a,\cdot)+\gamma v)}{\lambda}\right)
\right]
\right|.
\end{aligned}
\]
On $\Omega_{n,\eta}$, Lemma~C.6 of
\citet{chen2025sample} gives, for every $a\in\mathcal A$,
\[
\begin{aligned}
&\sup_{\lambda>0}\left|
\lambda\log \widehat p_{s,a}\left[
\exp\left(-\frac{\phi(a)(r(s,a,\cdot)+\gamma v)}{\lambda}\right)
\right]
\right.\\
&\hspace{32mm}\left.
-\lambda\log p_{s,a}\left[
\exp\left(-\frac{\phi(a)(r(s,a,\cdot)+\gamma v)}{\lambda}\right)
\right]
\right|
\\
&\qquad\le
3\eta\,\operatorname{Span}\bigl(
\phi(a)(r(s,a,\cdot)+\gamma v)
\bigr)\\
&\qquad=
3\eta\,\phi(a)\operatorname{Span}\bigl(r(s,a,\cdot)+\gamma v\bigr).
\end{aligned}
\]
Therefore,
\[
\begin{aligned}
\left|
\widehat{\mathcal T}_\gamma^\pi(v)(s)
-\mathcal T_\gamma^\pi(v)(s)
\right|
&\le
3\eta
\sum_{a\in\mathcal A}\phi(a)
\operatorname{Span}\bigl(r(s,a,\cdot)+\gamma v\bigr)\\
&\le
3\eta\max_{a\in\mathcal A}
\operatorname{Span}\bigl(r(s,a,\cdot)+\gamma v\bigr)\\
&\le
3\eta\bigl(1+\operatorname{Span}(v)\bigr).
\end{aligned}
\]
This proves the fixed-policy bound with the same coupled uncertainty budget.
\end{proof}

We next develop a common perturbation argument for the $f_k$-divergence
uncertainty sets. The following lemmas describe the coupled dual problem
for an arbitrary finite action set. We then apply them to each action
separately in the SA-rectangular case and to the full action set in the
$S$-rectangular case.

To this end, we define
$p_{t,s,a}=t\widehat p_{s,a}+(1-t)p_{s,a}$ and study the corresponding
dual value $G(t)$, where $G(0)$ and $G(1)$ are the true and empirical dual
values, respectively.
Lemma~\ref{lem:s-rec-fk-compact-reduction}
shows that the supremum defining $G(t)$ can be restricted to the same compact
set $X$ for all $t\in[0,1]$.
Lemma~\ref{lem:s-rec-fk-envelope-conditions}
establishes the existence of $X^*(t)$, the regularity of $f(\theta,t)$, and a
uniform bound on $f_t(\theta,t)$, while
Lemma~\ref{lem:s-rec-fk-envelope-theorem}
expresses $G(1)-G(0)$ as the integral of
$f_t(\theta^*(t),t)$.
Lemma~\ref{lem:s-rec-fk-first-order-condition}
gives the identity satisfied by each optimizer $\theta^*(t)$, which is then
used in
Lemma~\ref{lem:s-rec-fk-perturbation-derivative-bound}
to show that
\[
\left|f_t\bigl(\theta^*(t),t\bigr)\right|
\le
2\eta\max_{a\in\mathcal A}\operatorname{Span}(z_{s,a}).
\]
Combining these results yields the perturbation bounds in
Corollaries~\ref{cor:sa-rec-fk-perturbation} and~\ref{cor:s-rec-fk-perturbation}.

\begin{lemma}
\label{lem:s-rec-fk-compact-reduction}
Fix $s\in\mathcal S$, $\phi\in\Delta(\mathcal A)$, $k>1$,
$k^*=k/(k-1)$, and $\delta>0$. Let
\[
\begin{aligned}
c
&:=
|\mathcal A|^{1/k}\bigl(1+k(k-1)\delta\bigr)^{1/k},
&
p_{t,s,a}
&:=t\widehat p_{s,a}+(1-t)p_{s,a},
&
\Delta_{s,a}
&:=\widehat p_{s,a}-p_{s,a}.
\end{aligned}
\]
For $\theta\in\mathbb R^{|\mathcal A|}$ and $t\in[0,1]$, define
\[
\begin{aligned}
f(\theta,t)
:={}&
\sum_{a\in\mathcal A}\theta_a
-c\left(
\sum_{a\in\mathcal A}
p_{t,s,a}\left[
\left(\theta_a-\phi(a)z_{s,a}\right)_+^{k^*}
\right]
\right)^{1/k^*},\\
G(t)
:={}&
\sup_{\theta\in\mathbb R^{|\mathcal A|}}f(\theta,t).
\end{aligned}
\]
For every $a\in\mathcal A$, let
\[
\begin{aligned}
\mathcal X_{s,a}
&:=\operatorname{supp}(\widehat p_{s,a})\cup\operatorname{supp}(p_{s,a}),\\
L_a
&:=\min_{x\in\mathcal X_{s,a}}\phi(a)z_{s,a}(x),
\qquad
M_a
:=\max_{x\in\mathcal X_{s,a}}\phi(a)z_{s,a}(x),
\end{aligned}
\]
and define
\[
X:=\prod_{a\in\mathcal A}
\left[
L_a,\,
M_a+
\frac{\sum_{a\in\mathcal A}(M_a-L_a)}
{c-|\mathcal A|^{1/k}}
\right].
\]
Then, for every $t\in[0,1]$,
\[
G(t)=\sup_{\theta\in X}f(\theta,t).
\]
\end{lemma}

\begin{proof}
Fix $t\in[0,1]$. If $\theta_a\le L_a$, then, for every
$x\in\mathcal X_{s,a}$,
\[
\bigl(\theta_a-\phi(a)z_{s,a}(x)\bigr)_+
=
\bigl(L_a-\phi(a)z_{s,a}(x)\bigr)_+
=0.
\]
Thus, replacing $\theta_a$ by $L_a$ leaves the second term of
$f(\theta,t)$ unchanged and weakly increases its first term. Hence it
suffices to consider $\theta_a\ge L_a$ for all $a\in\mathcal A$.

Suppose that, for some $a_0\in\mathcal A$,
\[
\theta_{a_0}>
M_{a_0}+
\frac{\sum_{a\in\mathcal A}(M_a-L_a)}
{c-|\mathcal A|^{1/k}}.
\]
For every $a\in\mathcal A$ and $x\in\mathcal X_{s,a}$,
\[
\bigl(\theta_a-\phi(a)z_{s,a}(x)\bigr)_+
\ge
(\theta_a-M_a)_+.
\]
Consequently,
\[
\begin{aligned}
f(\theta,t)
&\le
\sum_{a\in\mathcal A}M_a
+\sum_{a\in\mathcal A}(\theta_a-M_a)_+\\
&\quad
-c\left(
\sum_{a\in\mathcal A}(\theta_a-M_a)_+^{k^*}
\right)^{1/k^*}.
\end{aligned}
\]
By H\"older's inequality,
\[
\sum_{a\in\mathcal A}(\theta_a-M_a)_+
\le
|\mathcal A|^{1/k}
\left(
\sum_{a\in\mathcal A}(\theta_a-M_a)_+^{k^*}
\right)^{1/k^*}.
\]
Therefore,
\[
\begin{aligned}
f(\theta,t)
&\le
\sum_{a\in\mathcal A}M_a
-\bigl(c-|\mathcal A|^{1/k}\bigr)
\left(
\sum_{a\in\mathcal A}(\theta_a-M_a)_+^{k^*}
\right)^{1/k^*}\\
&\le
\sum_{a\in\mathcal A}M_a
-\bigl(c-|\mathcal A|^{1/k}\bigr)(\theta_{a_0}-M_{a_0})\\
&<
\sum_{a\in\mathcal A}L_a
=f\bigl((L_a)_{a\in\mathcal A},t\bigr).
\end{aligned}
\]
Thus no point outside $X$ can attain the supremum, which proves the claim.
\end{proof}

\begin{lemma}
\label{lem:s-rec-fk-envelope-conditions}
Under the notation of Lemma~\ref{lem:s-rec-fk-compact-reduction}, assume
that, for every $a\in\mathcal A$ and $t\in(0,1)$,
\[
p_{t,s,a}\ge\frac12p_{s,a},
\qquad
\left\|
\frac{\Delta_{s,a}}{p_{t,s,a}}
\right\|_{L^\infty(p_{t,s,a})}
\le 2\eta.
\]
For $\theta\in X$, define
\[
B_\theta(t)
:=
\sum_{a\in\mathcal A}
p_{t,s,a}\left[
\left(\theta_a-\phi(a)z_{s,a}\right)_+^{k^*}
\right],
\qquad
X^*(t):=\arg\max_{\theta\in X}f(\theta,t).
\]
Then the following hold.
\begin{enumerate}
\item The set $X$ is compact and
\[
X^*(t)\ne\varnothing,
\qquad
\forall t\in[0,1].
\]

\item For every $\theta\in X$, $f(\theta,\cdot)$ is absolutely continuous
on $[0,1]$, differentiable on $(0,1)$, and
\[
\begin{aligned}
&\left|f_t(\theta,t)\right|\\
&=
\begin{cases}
0,
& B_\theta(0)=B_\theta(1)=0,\\[1mm]
\dfrac{c}{k^*}
\left(tB_\theta(1)+(1-t)B_\theta(0)\right)^{\frac1{k^*}-1}
\left|B_\theta(1)-B_\theta(0)\right|,
& \text{otherwise}.
\end{cases}
\end{aligned}
\]

\item There exists a finite constant $B$ such that
\[
\left|f_t(\theta,t)\right|\le B,
\qquad
\forall t\in(0,1),\quad\forall\theta\in X.
\]
\end{enumerate}
\end{lemma}

\begin{proof}
The set $X$ is a finite product of closed bounded intervals. For every
$t\in[0,1]$, $f(\cdot,t)$ is continuous on $X$; hence $X$ is compact and
Weierstrass' theorem gives $X^*(t)\ne\varnothing$.

Fix $\theta\in X$. Since
$p_{t,s,a}=t\widehat p_{s,a}+(1-t)p_{s,a}$,
\[
B_\theta(t)=tB_\theta(1)+(1-t)B_\theta(0).
\]
Because $u\mapsto u^{1/k^*}$ is concave on $\mathbb R_+$,
$f(\theta,\cdot)$ is finite and convex on $[0,1]$, and is therefore
absolutely continuous. If $B_\theta(t_0)=0$ for some $t_0\in(0,1)$, then
\[
B_\theta(0)=B_\theta(1)=0,
\qquad
B_\theta(t)=0,
\quad \forall t\in[0,1].
\]
Thus $f_t(\theta,t)=0$. Otherwise, $B_\theta(t)>0$ on $(0,1)$, and
\[
\left|f_t(\theta,t)\right|
=
\frac{c}{k^*}
\left(tB_\theta(1)+(1-t)B_\theta(0)\right)^{\frac1{k^*}-1}
\left|B_\theta(1)-B_\theta(0)\right|.
\]

Moreover,
\[
\begin{aligned}
\left|B_\theta(1)-B_\theta(0)\right|
&=
\left|
\sum_{a\in\mathcal A}
\Delta_{s,a}\left[
\left(\theta_a-\phi(a)z_{s,a}\right)_+^{k^*}
\right]
\right|\\
&\le
2\eta
\sum_{a\in\mathcal A}
p_{t,s,a}\left[
\left(\theta_a-\phi(a)z_{s,a}\right)_+^{k^*}
\right]\\
&=2\eta B_\theta(t).
\end{aligned}
\]
For every $\theta\in X$ and $t\in[0,1]$,
\[
B_\theta(t)
\le
\sum_{a\in\mathcal A}
\left(
M_a-L_a+
\frac{\sum_{b\in\mathcal A}(M_b-L_b)}
{c-|\mathcal A|^{1/k}}
\right)^{k^*}.
\]
Hence
\[
\left|f_t(\theta,t)\right|
\le
\frac{2\eta c}{k^*}
\left(
\sum_{a\in\mathcal A}
\left(
M_a-L_a+
\frac{\sum_{b\in\mathcal A}(M_b-L_b)}
{c-|\mathcal A|^{1/k}}
\right)^{k^*}
\right)^{1/k^*},
\]
which is finite and uniform over $t\in(0,1)$ and $\theta\in X$.
\end{proof}

\begin{lemma}
\label{lem:s-rec-fk-envelope-theorem}
Under the conditions of Lemma~\ref{lem:s-rec-fk-envelope-conditions}, for
any measurable selection $\theta^*(t)\in X^*(t)$,
\[
G(1)=G(0)+\int_0^1 f_t\bigl(\theta^*(r),r\bigr)\,dr.
\]
\end{lemma}

\begin{proof}
By Lemma~\ref{lem:s-rec-fk-envelope-conditions},
\[
X^*(t)\ne\varnothing,
\qquad
\forall t\in[0,1],
\]
and, for every $\theta\in X$, $f(\theta,\cdot)$ is absolutely continuous
on $[0,1]$, differentiable on $(0,1)$, and satisfies
\[
\left|f_t(\theta,t)\right|\le B,
\qquad
\forall t\in(0,1),\quad\forall\theta\in X.
\]
Since $X$ is compact and $f$ is jointly continuous, $X^*(t)$ admits a
measurable selection.
Thus the hypotheses of \citet[Theorem~2]{milgrom2002envelope} hold
with the integrable envelope $t\mapsto B$, and hence
\[
G(1)=G(0)+\int_0^1 f_t\bigl(\theta^*(r),r\bigr)\,dr.
\]
\end{proof}

\begin{lemma}
\label{lem:s-rec-fk-first-order-condition}
Let $\theta^*(t)\in X^*(t)$ and suppose that
$B_{\theta^*(t)}(t)>0$. Then, for every $a\in\mathcal A$,
\[
1
=
cB_{\theta^*(t)}(t)^{\frac1{k^*}-1}
p_{t,s,a}\left[
\left(\theta_a^*(t)-\phi(a)z_{s,a}\right)_+^{k^*-1}
\right].
\]
\end{lemma}

\begin{proof}
By Lemma~\ref{lem:s-rec-fk-compact-reduction}, $\theta^*(t)$ maximizes
$f(\cdot,t)$ over $\mathbb R^{|\mathcal A|}$. Since $k^*>1$ and
$B_{\theta^*(t)}(t)>0$,
\[
\begin{aligned}
0
&=f_{\theta_a}(\theta^*(t),t)\\
&=1
-cB_{\theta^*(t)}(t)^{\frac1{k^*}-1}
p_{t,s,a}\left[
\left(\theta_a^*(t)-\phi(a)z_{s,a}\right)_+^{k^*-1}
\right].
\end{aligned}
\]
Rearranging proves the result.
\end{proof}

\begin{lemma}
\label{lem:s-rec-fk-perturbation-derivative-bound}
Under the conditions of Lemma~\ref{lem:s-rec-fk-envelope-conditions}, for
every $\theta^*(t)\in X^*(t)$ and every $t\in(0,1)$,
\[
\left|f_t\bigl(\theta^*(t),t\bigr)\right|
\le
2\eta\max_{a\in\mathcal A}\operatorname{Span}(z_{s,a})
\le
2\eta\bigl(1+\operatorname{Span}(v)\bigr).
\]
\end{lemma}

\begin{proof}
If $\theta^*(t)=(L_a)_{a\in\mathcal A}$, then
$B_{\theta^*(t)}(r)=0$ for every $r\in[0,1]$, and hence
\[
f_t\bigl(\theta^*(t),t\bigr)=0.
\]
Otherwise, $B_{\theta^*(t)}(t)>0$. Indeed, if it were zero, then, since
$t\in(0,1)$ and $\operatorname{supp}(p_{t,s,a})=\mathcal X_{s,a}$,
\[
\theta_a^*(t)\le\phi(a)z_{s,a}(x),
\qquad
\forall a\in\mathcal A,\quad\forall x\in\mathcal X_{s,a}.
\]
Thus $\theta_a^*(t)\le L_a$ for every $a$, which, together with
$\theta^*(t)\in X$, implies $\theta^*(t)=(L_a)_{a\in\mathcal A}$.

Write $\theta^*=\theta^*(t)$. Since $\Delta_{s,a}[\mathbf 1]=0$,
\[
\begin{aligned}
\left|B_{\theta^*}(1)-B_{\theta^*}(0)\right|
&=
\left|
\sum_{a\in\mathcal A}
\Delta_{s,a}\left[
\left(\theta_a^*-\phi(a)z_{s,a}\right)_+^{k^*}
-\left(\theta_a^*-M_a\right)_+^{k^*}
\right]
\right|\\
&\le
2\eta\sum_{a\in\mathcal A}p_{t,s,a}\left[
\left|
\left(\theta_a^*-\phi(a)z_{s,a}\right)_+^{k^*}
-\left(\theta_a^*-M_a\right)_+^{k^*}
\right|
\right].
\end{aligned}
\]
For every $x\in\mathcal X_{s,a}$,
\[
\begin{aligned}
&\left|
\left(\theta_a^*-\phi(a)z_{s,a}(x)\right)_+^{k^*}
-\left(\theta_a^*-M_a\right)_+^{k^*}
\right|\\
&\qquad\le
k^*\left(\theta_a^*-\phi(a)z_{s,a}(x)\right)_+^{k^*-1}
\operatorname{Span}\left(
\left(\theta_a^*-\phi(a)z_{s,a}\right)_+
\right)\\
&\qquad\le
k^*\left(\theta_a^*-\phi(a)z_{s,a}(x)\right)_+^{k^*-1}
\phi(a)\operatorname{Span}(z_{s,a}),
\end{aligned}
\]
where the first inequality is the mean-value theorem. Therefore,
\[
\begin{aligned}
\left|f_t\bigl(\theta^*(t),t\bigr)\right|
&\le
2\eta c B_{\theta^*}(t)^{\frac1{k^*}-1}
\sum_{a\in\mathcal A}
p_{t,s,a}\left[
\left(\theta_a^*-\phi(a)z_{s,a}\right)_+^{k^*-1}
\right]
\phi(a)\operatorname{Span}(z_{s,a})\\
&=
2\eta\sum_{a\in\mathcal A}
\phi(a)\operatorname{Span}(z_{s,a})\\
&\le
2\eta\max_{a\in\mathcal A}\operatorname{Span}(z_{s,a})\\
&\le
2\eta\bigl(1+\operatorname{Span}(v)\bigr),
\end{aligned}
\]
where the equality uses Lemma~\ref{lem:s-rec-fk-first-order-condition}.
\end{proof}

We first apply these lemmas to the separate local problems in the
SA-rectangular model.

\begin{corollary}
\label{cor:sa-rec-fk-perturbation}
Under the SA-rectangular $f_k$-divergence uncertainty set, on the good event
$\Omega_{n,\eta}$ with $\eta\le 1/2$, simultaneously for every
$v\in\mathbb R^{|\mathcal S|}$, $\pi\in\Pi_{\mathrm{SR}}$, $s\in\mathcal S$, and $\gamma\in(0,1)$,
\[
\left|
\widehat{\mathcal T}_\gamma^\pi(v)(s)-\mathcal T_\gamma^\pi(v)(s)
\right|
\le2\eta\bigl(1+\operatorname{Span}(v)\bigr).
\]
\end{corollary}

\begin{proof}[Proof of Corollary~\ref{cor:sa-rec-fk-perturbation}]
Fix $\pi\in\Pi_{\mathrm{SR}}$ and $s\in\mathcal S$, and first suppose $\delta>0$. For each fixed
$a\in\mathcal A$, apply
Lemmas~\ref{lem:s-rec-fk-compact-reduction}--\ref{lem:s-rec-fk-perturbation-derivative-bound}
to an auxiliary problem whose action set is the singleton $\{a\}$.
Its only action distribution has $\phi(a)=1$, and its dual coefficient
is $c=(1+k(k-1)\delta)^{1/k}>1$. Thus its dual value is
\[
G_a(t):=\sup_{\theta_a\in\mathbb R}
\left\{\theta_a-c\left(p_{t,s,a}\left[
(\theta_a-z_{s,a})_+^{k^*}\right]\right)^{1/k^*}\right\}.
\]
This is exactly the local dual problem in \SAFkDualReference, with
center $p_{t,s,a}$. The specialization uses a singleton action set,
rather than an action distribution with $\phi(a)=1$ in the original
multi-action $S$-rectangular problem.

Since the samples are drawn from $p_{s,a}$,
$\operatorname{supp}(\widehat p_{s,a})\subseteq\operatorname{supp}(p_{s,a})$
almost surely. On $\Omega_{n,\eta}$, the supports coincide, and for every
$t\in(0,1)$,
\[
\begin{gathered}
p_{t,s,a}=p_{s,a}+t\Delta_{s,a}
\ge(1-\eta)p_{s,a}\ge\tfrac12p_{s,a},\\
\left\|\frac{\Delta_{s,a}}{p_{t,s,a}}\right\|_{L^\infty(p_{t,s,a})}
\le\frac{1}{1-\eta}
\left\|\frac{\Delta_{s,a}}{p_{s,a}}\right\|_{L^\infty(p_{s,a})}
\le2\eta.
\end{gathered}
\]
Hence the conditions of Lemma~\ref{lem:s-rec-fk-envelope-conditions}
hold. Lemmas~\ref{lem:s-rec-fk-envelope-theorem}--\ref{lem:s-rec-fk-perturbation-derivative-bound}
then give
\[
|G_a(1)-G_a(0)|\le2\eta\operatorname{Span}(z_{s,a}).
\]
By \SAFkDualReference\ and its empirical counterpart,
$\mathcal T_\gamma^\pi(v)(s)=\sum_{a\in\mathcal A}\pi(a\mid s)G_a(0)$ and
$\widehat{\mathcal T}_\gamma^\pi(v)(s)=\sum_{a\in\mathcal A}\pi(a\mid s)G_a(1)$.
Since $\sum_a\pi(a\mid s)=1$, we obtain
\[
\begin{aligned}
\left|\widehat{\mathcal T}_\gamma^\pi(v)(s)
-\mathcal T_\gamma^\pi(v)(s)\right|
&\le\sum_{a\in\mathcal A}\pi(a\mid s)|G_a(1)-G_a(0)|\\
&\le2\eta\sum_{a\in\mathcal A}\pi(a\mid s)\operatorname{Span}(z_{s,a})\\
&\le2\eta\bigl(1+\operatorname{Span}(v)\bigr).
\end{aligned}
\]
If $\delta=0$, each local uncertainty set reduces to its center.
On $\Omega_{n,\eta}$, $\|\widehat p_{s,a}-p_{s,a}\|_1\le\eta$, so
Lemma~\ref{lem:tv-span}\ gives
$|\widehat p_{s,a}[z_{s,a}]-p_{s,a}[z_{s,a}]|
\le(\eta/2)\operatorname{Span}(z_{s,a})$.
Weighting by $\pi(a\mid s)$ and summing over actions again yields the
fixed-policy bound.
\end{proof}

We now apply the same lemmas to the full action set in the
$S$-rectangular model, retaining the coupled dual problem for
$\phi=\pi(\cdot\mid s)$.

\begin{corollary}
\label{cor:s-rec-fk-perturbation}
Under the $S$-rectangular $f_k$-divergence uncertainty set, on the good event
$\Omega_{n,\eta}$ with $\eta\le 1/2$, simultaneously for every
$v\in\mathbb R^{|\mathcal S|}$, $\pi\in\Pi_{\mathrm{SR}}$, $s\in\mathcal S$, and $\gamma\in(0,1)$,
\[
\left|
\widehat{\mathcal T}_\gamma^\pi(v)(s)-\mathcal T_\gamma^\pi(v)(s)
\right|
\le2\eta\bigl(1+\operatorname{Span}(v)\bigr).
\]
\end{corollary}

\begin{proof}[Proof of Corollary~\ref{cor:s-rec-fk-perturbation}]
Fix $\pi\in\Pi_{\mathrm{SR}}$ and $s\in\mathcal S$, and set $\phi=\pi(\cdot\mid s)$. The case $\delta=0$ follows from the same relative-error argument used in
Corollary~\ref{cor:sa-rec-fk-perturbation}, weighting the action-wise
estimates by $\pi(a\mid s)$ and summing, so suppose $\delta>0$.
Let $G_\phi$ denote the function $G$ in
Lemma~\ref{lem:s-rec-fk-compact-reduction}. By
\SFkDualReference\ and its
empirical counterpart,
\[
\mathcal T_\gamma^\pi(v)(s)
=G_\phi(0),
\qquad
\widehat{\mathcal T}_\gamma^\pi(v)(s)
=G_\phi(1).
\]
On $\Omega_{n,\eta}$, for every $a\in\mathcal A$ and $t\in(0,1)$,
\[
\begin{gathered}
p_{t,s,a}=p_{s,a}+t\Delta_{s,a}
\ge(1-\eta)p_{s,a}
\ge\frac12p_{s,a},\\
\left\|
\frac{\Delta_{s,a}}{p_{t,s,a}}
\right\|_{L^\infty(p_{t,s,a})}
\le\frac{1}{1-\eta}
\left\|\frac{\Delta_{s,a}}{p_{s,a}}\right\|_{L^\infty(p_{s,a})}
\le2\eta.
\end{gathered}
\]
Thus Lemma~\ref{lem:s-rec-fk-envelope-conditions} applies. For a measurable
selection $\theta^*(t)\in X^*(t)$, Lemmas~\ref{lem:s-rec-fk-envelope-theorem}
and~\ref{lem:s-rec-fk-perturbation-derivative-bound} give
\[
\begin{aligned}
\left|G_\phi(1)-G_\phi(0)\right|
&\le
\int_0^1\left|f_t\bigl(\theta^*(r),r\bigr)\right|\,dr\\
&\le
2\eta\bigl(1+\operatorname{Span}(v)\bigr).
\end{aligned}
\]
This gives the fixed-policy bound.
\end{proof}

\begin{corollary}
\label{cor:sa-rec-tv-perturbation}
Under the SA-rectangular TV uncertainty set, on the good event
$\Omega_{n,\eta}$, simultaneously for every
$v\in\mathbb R^{|\mathcal S|}$, $\pi\in\Pi_{\mathrm{SR}}$, $s\in\mathcal S$, and $\gamma\in(0,1)$,
\[
\left|
\widehat{\mathcal T}_\gamma^\pi(v)(s)-\mathcal T_\gamma^\pi(v)(s)
\right|
\le\frac{\eta}{2}\bigl(1+\operatorname{Span}(v)\bigr).
\]
\end{corollary}

\begin{proof}[Proof of Corollary~\ref{cor:sa-rec-tv-perturbation}]
Fix $\pi\in\Pi_{\mathrm{SR}}$ and $s\in\mathcal S$, and first suppose $\delta>0$. By \SATVDualReference, let
$\widehat\mu_a,\mu_a\ge0$ be optimizers for the empirical and true dual
objectives for $z_{s,a}$, respectively. Then
\[
\begin{aligned}
\widehat p_{s,a}[z_{s,a}-\widehat\mu_a]
-\delta\operatorname{Span}(z_{s,a}-\widehat\mu_a)
&\ge
\widehat p_{s,a}[z_{s,a}]
-\delta\operatorname{Span}(z_{s,a}),\\
\delta\!\left(
\operatorname{Span}(z_{s,a})
-\operatorname{Span}(z_{s,a}-\widehat\mu_a)
\right)
&\ge
\widehat p_{s,a}[\widehat\mu_a]
\ge0,\\[1mm]
p_{s,a}[z_{s,a}-\mu_a]
-\delta\operatorname{Span}(z_{s,a}-\mu_a)
&\ge
p_{s,a}[z_{s,a}]
-\delta\operatorname{Span}(z_{s,a}),\\
\delta\!\left(
\operatorname{Span}(z_{s,a})
-\operatorname{Span}(z_{s,a}-\mu_a)
\right)
&\ge
p_{s,a}[\mu_a]
\ge0.
\end{aligned}
\]
Moreover, on $\Omega_{n,\eta}$,
\[
d_{\mathrm{TV}}\bigl(\widehat p_{s,a},p_{s,a}\bigr)
=\frac12\left\|\widehat p_{s,a}-p_{s,a}\right\|_1
\le\frac{\eta}{2}.
\]
Therefore, by Lemma~\ref{lem:tv-span},
\[
\begin{aligned}
&\left|
\inf_{q_{s,a}\in\widehat{\mathcal P}_{s,a}(d_{\mathrm{TV}},\delta)}q_{s,a}[z_{s,a}]
-\inf_{q_{s,a}\in\mathcal P_{s,a}(d_{\mathrm{TV}},\delta)}q_{s,a}[z_{s,a}]
\right|\\
&\quad\le
\max\!\left\{
\left|\bigl(\widehat p_{s,a}-p_{s,a}\bigr)[z_{s,a}-\widehat\mu_a]\right|,
\left|\bigl(\widehat p_{s,a}-p_{s,a}\bigr)[z_{s,a}-\mu_a]\right|
\right\}\\
&\quad\le
d_{\mathrm{TV}}\bigl(\widehat p_{s,a},p_{s,a}\bigr)
\max\!\left\{
\operatorname{Span}(z_{s,a}-\widehat\mu_a),
\operatorname{Span}(z_{s,a}-\mu_a)
\right\}\\
&\quad\le
d_{\mathrm{TV}}\bigl(\widehat p_{s,a},p_{s,a}\bigr)
\operatorname{Span}(z_{s,a})
\le
\frac{\eta}{2}\operatorname{Span}(z_{s,a}).
\end{aligned}
\]
If $\delta=0$, each uncertainty set reduces to its center, and the same
action-wise bound follows directly from Lemma~\ref{lem:tv-span}.
Weighting these estimates by $\pi(a\mid s)$ and using
$\sum_a\pi(a\mid s)=1$ yields
\[
\begin{aligned}
\left|
\widehat{\mathcal T}_\gamma^\pi(v)(s)
-\mathcal T_\gamma^\pi(v)(s)
\right|
&\le
\frac{\eta}{2}\sum_{a\in\mathcal A}\pi(a\mid s)
\operatorname{Span}\bigl(r(s,a,\cdot)+\gamma v\bigr)\\
&\le
\frac{\eta}{2}\bigl(1+\operatorname{Span}(v)\bigr).
\end{aligned}
\]
\end{proof}

\begin{corollary}
\label{cor:s-rec-tv-perturbation}
Under the $S$-rectangular TV uncertainty set, on the good event
$\Omega_{n,\eta}$, simultaneously for every
$v\in\mathbb R^{|\mathcal S|}$, $\pi\in\Pi_{\mathrm{SR}}$, $s\in\mathcal S$, and $\gamma\in(0,1)$,
\[
\left|
\widehat{\mathcal T}_\gamma^\pi(v)(s)-\mathcal T_\gamma^\pi(v)(s)
\right|
\le\frac{\eta}{2}\bigl(1+\operatorname{Span}(v)\bigr).
\]
\end{corollary}

\begin{proof}[Proof of Corollary~\ref{cor:s-rec-tv-perturbation}]
Fix $\pi\in\Pi_{\mathrm{SR}}$ and $s\in\mathcal S$, and set $\phi=\pi(\cdot\mid s)$. By
\STVDualReference, let
$\widehat\mu_a,\mu_a\ge0$ be optimizers for the empirical and true dual
objectives, respectively. Replace $\mu_a$ by
\[
\min\!\left\{
\mu_a-\min_{y\in\mathcal S}\mu_a(y),\,
\phi(a)z_{s,a}-\min_{y\in\mathcal S}\phi(a)z_{s,a}(y)
\right\}.
\]
Then
\[
\phi(a)z_{s,a}-\mu_a
=
\max\!\left\{
\phi(a)z_{s,a}-\mu_a+\min_{y\in\mathcal S}\mu_a(y),\,
\min_{y\in\mathcal S}\phi(a)z_{s,a}(y)
\right\}.
\]
The resulting vector is pointwise no smaller than the original one, and its
values lie between
\[
\min_{y\in\mathcal S}\phi(a)z_{s,a}(y)
\quad\text{and}\quad
\max_{y\in\mathcal S}\phi(a)z_{s,a}(y).
\]
Thus its expectation does not decrease and its span does not increase.
Consequently, an optimizer can be chosen such that
\[
\operatorname{Span}(\phi(a)z_{s,a}-\mu_a)
\le
\phi(a)\operatorname{Span}(z_{s,a}).
\]
The same argument applies to $\widehat\mu_a$:
\[
\operatorname{Span}(\phi(a)z_{s,a}-\widehat\mu_a)
\le
\phi(a)\operatorname{Span}(z_{s,a}).
\]
On $\Omega_{n,\eta}$, Lemma~\ref{lem:tv-span} gives
\[
\begin{aligned}
&\left|
\widehat{\mathcal T}_\gamma^\pi(v)(s)
-\mathcal T_\gamma^\pi(v)(s)
\right|\\
&\quad\le
\sum_{a\in\mathcal A}
d_{\mathrm{TV}}\bigl(\widehat p_{s,a},p_{s,a}\bigr)
\phi(a)\operatorname{Span}(z_{s,a})\\
&\quad\le
\frac{\eta}{2}
\sum_{a\in\mathcal A}
\phi(a)\operatorname{Span}(z_{s,a})\\
&\quad\le
\frac{\eta}{2}
\max_{a\in\mathcal A}\operatorname{Span}(z_{s,a}).
\end{aligned}
\]
Since $\phi=\pi(\cdot\mid s)$, this gives
\[
\begin{aligned}
&\left|
\widehat{\mathcal T}_\gamma^\pi(v)(s)
-\mathcal T_\gamma^\pi(v)(s)
\right|\\
&\quad\le
\frac{\eta}{2}
\max_{a\in\mathcal A}
\operatorname{Span}\bigl(r(s,a,\cdot)+\gamma v\bigr)\\
&\quad\le
\frac{\eta}{2}
\bigl(1+\operatorname{Span}(v)\bigr).
\end{aligned}
\]
\end{proof}

\begin{corollary}
\label{cor:sa-rec-wasserstein-perturbation}
Under the SA-rectangular Wasserstein uncertainty set, on the good event
$\Omega_{n,\eta}$, simultaneously for every
$v\in\mathbb R^{|\mathcal S|}$, $\pi\in\Pi_{\mathrm{SR}}$, $s\in\mathcal S$, and $\gamma\in(0,1)$,
\[
\left|
\widehat{\mathcal T}_\gamma^\pi(v)(s)-\mathcal T_\gamma^\pi(v)(s)
\right|
\le\frac{\eta}{2}\bigl(1+\operatorname{Span}(v)\bigr).
\]
\end{corollary}

\begin{proof}[Proof of Corollary~\ref{cor:sa-rec-wasserstein-perturbation}]
Fix $\pi\in\Pi_{\mathrm{SR}}$ and $s\in\mathcal S$. By \SAWassersteinDualReference\
and its empirical counterpart, weighting the action-wise dual estimates
by $\pi(a\mid s)$ gives
\[
\begin{aligned}
&\left|
\widehat{\mathcal T}_\gamma^\pi(v)(s)
-\mathcal T_\gamma^\pi(v)(s)
\right|\\
&\quad\le
\sum_{a\in\mathcal A}\pi(a\mid s)\sup_{\lambda_a\ge0}
\left|
\bigl(\widehat p_{s,a}-p_{s,a}\bigr)
\left[
\inf_{y\in\mathcal S}
\left(z_{s,a}(y)+\lambda_a\rho(\cdot,y)^\ell\right)
\right]
\right|.
\end{aligned}
\]
Since $\rho(x,x)=0$, for every $x\in\mathcal S$ and $\lambda_a\ge0$,
\[
\min_{y\in\mathcal S}z_{s,a}(y)
\le
\inf_{y\in\mathcal S}
\left(z_{s,a}(y)+\lambda_a\rho(x,y)^\ell\right)
\le
z_{s,a}(x)
\le
\max_{y\in\mathcal S}z_{s,a}(y).
\]
Hence, by Lemma~\ref{lem:tv-span} and
$d_{\mathrm{TV}}(\widehat p_{s,a},p_{s,a})\le\eta/2$ on
$\Omega_{n,\eta}$, together with $\sum_a\pi(a\mid s)=1$,
\[
\begin{aligned}
&\left|
\widehat{\mathcal T}_\gamma^\pi(v)(s)
-\mathcal T_\gamma^\pi(v)(s)
\right|\\
&\quad\le
\frac{\eta}{2}\sum_{a\in\mathcal A}\pi(a\mid s)\operatorname{Span}(z_{s,a})\\
&\quad\le
\frac{\eta}{2}\sum_{a\in\mathcal A}\pi(a\mid s)
\operatorname{Span}\bigl(r(s,a,\cdot)+\gamma v\bigr)\\
&\quad\le
\frac{\eta}{2}\bigl(1+\operatorname{Span}(v)\bigr).
\end{aligned}
\]
\end{proof}

\begin{corollary}
\label{cor:s-rec-wasserstein-perturbation}
Under the $S$-rectangular Wasserstein uncertainty set, on the good event
$\Omega_{n,\eta}$, simultaneously for every
$v\in\mathbb R^{|\mathcal S|}$, $\pi\in\Pi_{\mathrm{SR}}$, $s\in\mathcal S$, and $\gamma\in(0,1)$,
\[
\left|
\widehat{\mathcal T}_\gamma^\pi(v)(s)-\mathcal T_\gamma^\pi(v)(s)
\right|
\le\frac{\eta}{2}\bigl(1+\operatorname{Span}(v)\bigr).
\]
\end{corollary}

\begin{proof}[Proof of Corollary~\ref{cor:s-rec-wasserstein-perturbation}]
Fix $\pi\in\Pi_{\mathrm{SR}}$ and $s\in\mathcal S$, and set $\phi=\pi(\cdot\mid s)$. By \SWassersteinDualReference\
and its empirical counterpart,
\[
\begin{aligned}
&\left|
\widehat{\mathcal T}_\gamma^\pi(v)(s)
-\mathcal T_\gamma^\pi(v)(s)
\right|\\
&\quad\le
\sup_{\lambda\ge0}
\left|
\sum_{a\in\mathcal A}
\bigl(\widehat p_{s,a}-p_{s,a}\bigr)
\left[
\inf_{y\in\mathcal S}
\left(\phi(a)z_{s,a}(y)+\lambda\rho(\cdot,y)^\ell\right)
\right]
\right|\\
&\quad\le
\sup_{\lambda\ge0}
\sum_{a\in\mathcal A}
d_{\mathrm{TV}}\bigl(\widehat p_{s,a},p_{s,a}\bigr)
\operatorname{Span}\!\left(
\inf_{y\in\mathcal S}
\left(\phi(a)z_{s,a}(y)+\lambda\rho(\cdot,y)^\ell\right)
\right).
\end{aligned}
\]
Since $\rho(x,x)=0$, for every $x\in\mathcal S$,
\[
\min_{y\in\mathcal S}\phi(a)z_{s,a}(y)
\le
\inf_{y\in\mathcal S}
\left(\phi(a)z_{s,a}(y)+\lambda\rho(x,y)^\ell\right)
\le
\phi(a)z_{s,a}(x).
\]
Hence, on $\Omega_{n,\eta}$,
\[
\begin{aligned}
&\left|
\widehat{\mathcal T}_\gamma^\pi(v)(s)
-\mathcal T_\gamma^\pi(v)(s)
\right|\\
&\quad\le
\frac{\eta}{2}
\sum_{a\in\mathcal A}
\phi(a)\operatorname{Span}(z_{s,a})\\
&\quad\le
\frac{\eta}{2}
\max_{a\in\mathcal A}
\operatorname{Span}\bigl(r(s,a,\cdot)+\gamma v\bigr)\\
&\quad\le
\frac{\eta}{2}\bigl(1+\operatorname{Span}(v)\bigr).
\end{aligned}
\]
This proves the fixed-policy bound with the same joint Wasserstein budget.
\end{proof}

\section{Empirical DR-DMDP}
\label{sec:empirical-dr-dmdp}

This appendix develops the bounds required for the empirical DR-DMDP analysis.
We first combine the eight perturbation bounds established in
Appendix~\ref{sec:perturbation-bounds}
into a uniform bound for the fixed-policy Bellman operator. We then control
$\widehat v_\gamma^*-v_\gamma^*$ using the contraction property and relate
$\operatorname{Span}(v_\gamma^*)$ to the robust bias span
$\operatorname{Span}(u_\delta^*)$.

\begin{lemma}
\label{lem:empirical-operator-perturbation}
Suppose that $\mathcal P$ is a KL-divergence, $f_k$-divergence, TV, or
Wasserstein uncertainty set with either SA-rectangular or
$S$-rectangular structure. If $\eta\le 1/2$, then, on
$\Omega_{n,\eta}$, simultaneously for every
$v\in\mathbb R^{|\mathcal S|}$, $\pi\in\Pi_{\mathrm{SR}}$, and $\gamma\in(0,1)$,
\[
\left\|
\widehat{\mathcal T}_\gamma^\pi(v)-\mathcal T_\gamma^\pi(v)
\right\|_\infty
\le
3\eta\bigl(1+\operatorname{Span}(v)\bigr).
\]
\end{lemma}

\newcommand{\SAKLBound}{Corollary~\ref{cor:sa-rec-kl-perturbation}}
\newcommand{\SAFkBound}{Corollary~\ref{cor:sa-rec-fk-perturbation}}
\newcommand{\SATVBound}{Corollary~\ref{cor:sa-rec-tv-perturbation}}
\newcommand{\SAWassersteinBound}{Corollary~\ref{cor:sa-rec-wasserstein-perturbation}}
\newcommand{\SKLBound}{Corollary~\ref{cor:s-rec-kl-perturbation}}
\newcommand{\SFkBound}{Corollary~\ref{cor:s-rec-fk-perturbation}}
\newcommand{\STVBound}{Corollary~\ref{cor:s-rec-tv-perturbation}}
\newcommand{\SWassersteinBound}{Corollary~\ref{cor:s-rec-wasserstein-perturbation}}

\begin{proof}
For the SA-rectangular structure, the KL-, $f_k$-divergence-, TV-, and
Wasserstein-based bounds are given by
\SAKLBound, \SAFkBound, \SATVBound, and \SAWassersteinBound,
respectively. Their conclusions give, for every
$s\in\mathcal S$ and $\pi\in\Pi_{\mathrm{SR}}$,
\[
\left|
\widehat{\mathcal T}_\gamma^\pi(v)(s)-\mathcal T_\gamma^\pi(v)(s)
\right|
\le
\max\left\{3,2,\frac12,\frac12\right\}
\eta\bigl(1+\operatorname{Span}(v)\bigr).
\]
For the $S$-rectangular structure, the corresponding bounds follow from
\SKLBound, \SFkBound, \STVBound, and \SWassersteinBound,
respectively, and yield
\[
\left|
\widehat{\mathcal T}_\gamma^\pi(v)(s)-\mathcal T_\gamma^\pi(v)(s)
\right|
\le
\max\left\{3,2,\frac12,\frac12\right\}
\eta\bigl(1+\operatorname{Span}(v)\bigr).
\]
Both displays are bounded by
$3\eta\bigl(1+\operatorname{Span}(v)\bigr)$. Taking the maximum over
$s\in\mathcal S$ proves the claim, simultaneously over
$v$, $\pi$, and $\gamma$ on $\Omega_{n,\eta}$.
\end{proof}

\begin{lemma}
\label{lem:empirical-discounted-value-perturbation}
Suppose that $\mathcal P$ is a KL-divergence, $f_k$-divergence, TV, or
Wasserstein uncertainty set with either SA-rectangular or
$S$-rectangular structure. If $\eta\le 1/2$, then, on
$\Omega_{n,\eta}$, for every $\gamma\in(0,1)$,
\[
(1-\gamma)
\left\|
\widehat v_\gamma^*-v_\gamma^*
\right\|_\infty
\le
3\eta\bigl(1+\operatorname{Span}(v_\gamma^*)\bigr).
\]
\end{lemma}

\begin{proof}
Using $\mathcal T_\gamma^*=\max_{\pi\in\Pi_{\mathrm{SR}}}\mathcal T_\gamma^\pi$ and its empirical counterpart,
Lemma~\ref{lem:empirical-operator-perturbation} gives
\[
\left\|\widehat{\mathcal T}_\gamma^*(v)-\mathcal T_\gamma^*(v)\right\|_\infty
\le\sup_{\pi\in\Pi_{\mathrm{SR}}}
\left\|\widehat{\mathcal T}_\gamma^\pi(v)-\mathcal T_\gamma^\pi(v)\right\|_\infty
\le3\eta\bigl(1+\operatorname{Span}(v)\bigr).
\]
For the SA-rectangular empirical operator, for $U,v\in\mathbb R^{|\mathcal S|}$,
\[
\begin{aligned}
\left|
\widehat{\mathcal T}_\gamma^*(U)(s)
-\widehat{\mathcal T}_\gamma^*(v)(s)
\right|
&\le
\gamma\max_{a\in\mathcal A}
\sup_{q_{s,a}\in\widehat{\mathcal P}_{s,a}}
\left|q_{s,a}[U-v]\right|\\
&\le \gamma\|U-v\|_\infty.
\end{aligned}
\]
For the $S$-rectangular empirical operator,
\[
\begin{aligned}
&\left|
\widehat{\mathcal T}_\gamma^*(U)(s)
-\widehat{\mathcal T}_\gamma^*(v)(s)
\right|\\
&\quad\le
\max_{\phi\in\Delta(\mathcal A)}
\sup_{q_s\in\widehat{\mathcal P}_s}
\left|
\sum_{a\in\mathcal A}
q_{s,a}[\gamma\phi(a)(U-v)]
\right|\\
&\quad\le
\gamma\|U-v\|_\infty
\max_{\phi\in\Delta(\mathcal A)}
\sup_{q_s\in\widehat{\mathcal P}_s}
\sum_{a\in\mathcal A}\phi(a)q_{s,a}[\mathbf 1]\\
&\quad=
\gamma\|U-v\|_\infty.
\end{aligned}
\]
Thus, under either structure, $\widehat{\mathcal T}_\gamma^*$ is a
$\gamma$-contraction in $\|\cdot\|_\infty$. Since $v_\gamma^*$ and
$\widehat v_\gamma^*$ are the fixed points of $\mathcal T_\gamma^*$\ in \eqref{eq:discounted-bellman-operator}\ and
$\widehat{\mathcal T}_\gamma^*$, respectively,
\[
\begin{aligned}
\left\|
\widehat v_\gamma^*-v_\gamma^*
\right\|_\infty
&\le
\left\|
\widehat{\mathcal T}_\gamma^*(\widehat v_\gamma^*)
-\widehat{\mathcal T}_\gamma^*(v_\gamma^*)
\right\|_\infty\\
&\quad+
\left\|
\widehat{\mathcal T}_\gamma^*(v_\gamma^*)
-\mathcal T_\gamma^*(v_\gamma^*)
\right\|_\infty\\
&\le
\gamma\left\|
\widehat v_\gamma^*-v_\gamma^*
\right\|_\infty
+3\eta\bigl(1+\operatorname{Span}(v_\gamma^*)\bigr),
\end{aligned}
\]
where the last inequality follows from
Lemma~\ref{lem:empirical-operator-perturbation}.
Rearranging proves the claim.
\end{proof}

\begin{lemma}
\label{lem:discounted-value-span-bound}
Suppose that $(u_\delta^*,g_\delta^*)$ solves the optimal average-reward
Bellman equation~\eqref{eq:constant-gain-bellman}\ with constant gain under either the SA-rectangular or
$S$-rectangular structure:
\[
\mathcal T^*(u_\delta^*)
=u_\delta^*+g_\delta^*\mathbf 1.
\]
Then, for every $\gamma\in(0,1)$,
\[
\operatorname{Span}(v_\gamma^*)
\le
2\operatorname{Span}(u_\delta^*).
\]
\end{lemma}

\begin{proof}
This extends the comparison argument in the proof of Theorem~4,
Appendix~B, $(2)\Rightarrow(1)$, of
\citet{wang2025bellmanoptimalityaveragerewardrobust} from $r(s,a)$ to
$r(s,a,s')$, retaining the reward inside the next-state expectation.
By additive homogeneity,
normalize $u_\delta^*$ so that
\[
\min_{s\in\mathcal S}u_\delta^*(s)=0,
\qquad
0\le u_\delta^*(s)\le\operatorname{Span}(u_\delta^*).
\]
Define
\[
\overline v_\gamma
:=
\frac{g_\delta^*}{1-\gamma}\mathbf 1+u_\delta^*,
\qquad
\underline v_\gamma
:=
\overline v_\gamma
-\operatorname{Span}(u_\delta^*)\mathbf 1.
\]
Since $\gamma u_\delta^*\le u_\delta^*$, the operator definitions~\eqref{eq:discounted-bellman-operator}--\eqref{eq:average-bellman-operator}\ and the constant-gain equation~\eqref{eq:constant-gain-bellman}\ give
\[
\begin{aligned}
\mathcal T_\gamma^*(\overline v_\gamma)
&=
\frac{\gamma g_\delta^*}{1-\gamma}\mathbf 1
+\mathcal T_\gamma^*(u_\delta^*)\\
&\le
\frac{\gamma g_\delta^*}{1-\gamma}\mathbf 1
+\mathcal T^*(u_\delta^*)\\
&=
\overline v_\gamma.
\end{aligned}
\]
Also,
\[
\gamma u_\delta^*+(1-\gamma)
\operatorname{Span}(u_\delta^*)\mathbf 1
\ge u_\delta^*,
\]
so
\[
\begin{aligned}
\mathcal T_\gamma^*(\underline v_\gamma)
&=
\left(
\frac{\gamma g_\delta^*}{1-\gamma}
-\gamma\operatorname{Span}(u_\delta^*)
\right)\mathbf 1
+\mathcal T_\gamma^*(u_\delta^*)\\
&\ge
\left(
\frac{\gamma g_\delta^*}{1-\gamma}
-\gamma\operatorname{Span}(u_\delta^*)
\right)\mathbf 1
+\mathcal T^*(u_\delta^*)
-(1-\gamma)\operatorname{Span}(u_\delta^*)\mathbf 1\\
&=
\underline v_\gamma.
\end{aligned}
\]
\[
\begin{aligned}
\underline v_\gamma
&\le
\mathcal T_\gamma^*(\underline v_\gamma)
\le
(\mathcal T_\gamma^*)^2(\underline v_\gamma)
\le\cdots\le
(\mathcal T_\gamma^*)^k(\underline v_\gamma),\\
\overline v_\gamma
&\ge
\mathcal T_\gamma^*(\overline v_\gamma)
\ge
(\mathcal T_\gamma^*)^2(\overline v_\gamma)
\ge\cdots\ge
(\mathcal T_\gamma^*)^k(\overline v_\gamma).
\end{aligned}
\]
The $\gamma$-contraction property of the discounted operator\ in \eqref{eq:discounted-bellman-operator}\ gives
\[
\begin{aligned}
\left\|
(\mathcal T_\gamma^*)^k(\underline v_\gamma)-v_\gamma^*
\right\|_\infty
&\le
\gamma^k
\left\|
\underline v_\gamma-v_\gamma^*
\right\|_\infty
\xrightarrow[k\to\infty]{}0,
\\[2mm]
\left\|
(\mathcal T_\gamma^*)^k(\overline v_\gamma)-v_\gamma^*
\right\|_\infty
&\le
\gamma^k
\left\|
\overline v_\gamma-v_\gamma^*
\right\|_\infty
\xrightarrow[k\to\infty]{}0.
\end{aligned}
\]
Thus,
\[
\underline v_\gamma\le v_\gamma^*\le\overline v_\gamma.
\]
\[
\begin{aligned}
\operatorname{Span}(v_\gamma^*)
&\le
\max_{s\in\mathcal S}\overline v_\gamma(s)
-\min_{s\in\mathcal S}\underline v_\gamma(s)\\
&=
2\operatorname{Span}(u_\delta^*).
\end{aligned}
\]
\end{proof}

\begin{remark}
The constant-gain Bellman solution established in
Section~\ref{subsec:assumptions-and-properties}
allows Lemma~\ref{lem:discounted-value-span-bound} to bound the discounted
value span by twice the robust optimal bias span.
Since $g_\delta^*$ is the same for every state, the common term
$g_\delta^*/(1-\gamma)$ in the upper and lower bounds on $v_\gamma^*$
cancels in the span estimate, yielding a bound uniform over
$\gamma\in(0,1)$. Using this comparison in
Lemma~\ref{lem:empirical-discounted-value-perturbation} gives the dependence
on $\operatorname{Span}(u_\delta^*)$ in the estimation bound of
Lemma~\ref{lem:empirical-dr-dmdp-error},
without requiring uniform mixing. The same comparison is used in
Lemma~\ref{lem:discounted-average-reward-bridge}
to control the discounted-to-average-reward approximation error for
Theorem~\ref{thm:robust-average-reward-error}
in terms of the same robust optimal bias span.
\end{remark}

\section{Empirical DR-AMDP}
\label{sec:empirical-dr-amdp}

This appendix connects the optimal discounted value $v_\gamma^*$ with the
optimal constant gain $g_\delta^*$. The main result is
Lemma~\ref{lem:discounted-average-reward-bridge},
which bounds $\left\|(1-\gamma)v_\gamma^*-g_\delta^*\mathbf 1\right\|_\infty$
in terms of $(1-\gamma)\operatorname{Span}(u_\delta^*)$. For completeness,
we first establish three intermediate results.
Lemma~\ref{lem:average-operator-comparison}
bounds $\mathcal T^*(v_\gamma^*)-v_\gamma^*$ using the minimum and maximum
of $v_\gamma^*$.
Lemma~\ref{lem:iterated-operator-comparison}
shows the property of $\mathcal T^*$\ in \eqref{eq:average-bellman-operator}, and
Lemma~\ref{lem:gain-sandwich}
bounds $g_\delta^*$ by the minimum and maximum of
$\mathcal T^*(v)-v$. Together with
Lemma~\ref{lem:discounted-value-span-bound},
these results yield the bound in
Lemma~\ref{lem:discounted-average-reward-bridge}.

\begin{lemma}
\label{lem:average-operator-comparison}
Let
\[
m:=\min_{s\in\mathcal S}v_\gamma^*(s),
\qquad
M:=\max_{s\in\mathcal S}v_\gamma^*(s).
\]
Under either the SA-rectangular or $S$-rectangular structure,
\[
v_\gamma^*+(1-\gamma)m\mathbf 1
\le
\mathcal T^*(v_\gamma^*)
\le
v_\gamma^*+(1-\gamma)M\mathbf 1.
\]
\end{lemma}

\begin{proof}
For the SA-rectangular structure, for every $s\in\mathcal S$,
$a\in\mathcal A$, and $q_{s,a}\in\mathcal P_{s,a}$,
\[
m
\le q_{s,a}[v_\gamma^*]\le M.
\]
Hence
\[
\begin{aligned}
q_{s,a}[r(s,a,\cdot)+\gamma v_\gamma^*]+(1-\gamma)m
&\le q_{s,a}[r(s,a,\cdot)+v_\gamma^*]\\
&\le q_{s,a}[r(s,a,\cdot)+\gamma v_\gamma^*]+(1-\gamma)M.
\end{aligned}
\]
Taking the infimum over $q_{s,a}\in\mathcal P_{s,a}$ and then the
maximum over $a\in\mathcal A$ yields, by the SA-rectangular reductions in Definition~\ref{defn:bellman-operators},
\[
\mathcal T_\gamma^*(v_\gamma^*)(s)+(1-\gamma)m
\le
\mathcal T^*(v_\gamma^*)(s)
\le
\mathcal T_\gamma^*(v_\gamma^*)(s)+(1-\gamma)M.
\]

For the $S$-rectangular structure, for every
$\phi\in\Delta(\mathcal A)$ and $q_s\in\mathcal P_s$,
\[
m
\le
\sum_{a\in\mathcal A}q_{s,a}[\phi(a)v_\gamma^*]
\le M.
\]
Consequently,
\[
\begin{aligned}
&\sum_{a\in\mathcal A}
q_{s,a}[\phi(a)(r(s,a,\cdot)+\gamma v_\gamma^*)]
+(1-\gamma)m\\
&\quad\le
\sum_{a\in\mathcal A}
q_{s,a}[\phi(a)(r(s,a,\cdot)+v_\gamma^*)]\\
&\quad\le
\sum_{a\in\mathcal A}
q_{s,a}[\phi(a)(r(s,a,\cdot)+\gamma v_\gamma^*)]
+(1-\gamma)M.
\end{aligned}
\]
Taking the infimum over $q_s\in\mathcal P_s$ and the maximum over
$\phi\in\Delta(\mathcal A)$ gives the same display using the operator definitions~\eqref{eq:discounted-bellman-operator} and~\eqref{eq:average-bellman-operator}. Since
$v_\gamma^*=\mathcal T_\gamma^*(v_\gamma^*)$, the claim follows.
\end{proof}

\begin{lemma}
\label{lem:iterated-operator-comparison}
Let $v\in\mathbb R^{|\mathcal S|}$. Under either the SA-rectangular or
$S$-rectangular structure, if
\[
C_1\mathbf 1
\le
\mathcal T^*(v)-v
\le
C_2\mathbf 1,
\]
then, for every $n\in\mathbb N$,
\[
nC_1\mathbf 1
\le
\bigl(\mathcal T^*\bigr)^n(v)-v
\le
nC_2\mathbf 1.
\]
\end{lemma}

\begin{proof}
Under both structures, the operator definition~\eqref{eq:average-bellman-operator}\ gives
\[
U\le W
\Longrightarrow
\mathcal T^*(U)\le\mathcal T^*(W),
\qquad
\mathcal T^*(W+c\mathbf 1)=\mathcal T^*(W)+c\mathbf 1.
\]
The term $r(s,a,\cdot)$ is unchanged in both relations. The assertion holds
for $n=1$. If it holds for $n$, then
\[
\begin{aligned}
\bigl(\mathcal T^*\bigr)^{n+1}(v)
&=\mathcal T^*\!\left(\bigl(\mathcal T^*\bigr)^n(v)\right)\\
&\le\mathcal T^*(v+nC_2\mathbf 1)\\
&=\mathcal T^*(v)+nC_2\mathbf 1\\
&\le v+(n+1)C_2\mathbf 1,
\end{aligned}
\]
and
\[
\begin{aligned}
\bigl(\mathcal T^*\bigr)^{n+1}(v)
&=\mathcal T^*\!\left(\bigl(\mathcal T^*\bigr)^n(v)\right)\\
&\ge\mathcal T^*(v+nC_1\mathbf 1)\\
&=\mathcal T^*(v)+nC_1\mathbf 1\\
&\ge v+(n+1)C_1\mathbf 1.
\end{aligned}
\]
Induction completes the proof.
\end{proof}

\begin{lemma}
\label{lem:gain-sandwich}
Suppose that $(u_\delta^*,g_\delta^*)$ solves the optimal average-reward
Bellman equation~\eqref{eq:constant-gain-bellman}\ with constant gain under either the SA-rectangular or
$S$-rectangular structure:
\[
\mathcal T^*(u_\delta^*)=u_\delta^*+g_\delta^*\mathbf 1.
\]
Then, for every $v\in\mathbb R^{|\mathcal S|}$,
\[
\min_{s\in\mathcal S}
\left\{\mathcal T^*(v)(s)-v(s)\right\}
\le
g_\delta^*
\le
\max_{s\in\mathcal S}
\left\{\mathcal T^*(v)(s)-v(s)\right\}.
\]
\end{lemma}

\begin{proof}
Let
\[
C_1:=\min_{s\in\mathcal S}
\left\{\mathcal T^*(v)(s)-v(s)\right\},
\qquad
C_2:=\max_{s\in\mathcal S}
\left\{\mathcal T^*(v)(s)-v(s)\right\}.
\]
By Lemma~\ref{lem:iterated-operator-comparison},
\[
nC_1\mathbf 1
\le
\bigl(\mathcal T^*\bigr)^n(v)-v
\le
nC_2\mathbf 1.
\]
Set
\[
B_1:=\min_{s\in\mathcal S}\bigl(v(s)-u_\delta^*(s)\bigr),
\qquad
B_2:=\max_{s\in\mathcal S}\bigl(v(s)-u_\delta^*(s)\bigr).
\]
Then, using the constant-gain equation~\eqref{eq:constant-gain-bellman},
\[
u_\delta^*+B_1\mathbf 1
\le v\le
u_\delta^*+B_2\mathbf 1,
\qquad
\bigl(\mathcal T^*\bigr)^n(u_\delta^*)
=u_\delta^*+ng_\delta^*\mathbf 1.
\]
By monotonicity and additive homogeneity,
\[
u_\delta^*+(ng_\delta^*+B_1)\mathbf 1
\le
\bigl(\mathcal T^*\bigr)^n(v)
\le
u_\delta^*+(ng_\delta^*+B_2)\mathbf 1.
\]
Thus,
\[
ng_\delta^*\mathbf 1-(B_2-B_1)\mathbf 1
\le
\bigl(\mathcal T^*\bigr)^n(v)-v
\le
ng_\delta^*\mathbf 1+(B_2-B_1)\mathbf 1.
\]
Combining the two displays gives
\[
C_1\le g_\delta^*+\frac{B_2-B_1}{n},
\qquad
g_\delta^*-\frac{B_2-B_1}{n}\le C_2.
\]
Letting $n\to\infty$ proves the claim.
\end{proof}

\begin{lemma}
\label{lem:discounted-average-reward-bridge}
Suppose that $(u_\delta^*,g_\delta^*)$ solves the optimal average-reward
Bellman equation~\eqref{eq:constant-gain-bellman}\ with constant gain under either the SA-rectangular or
$S$-rectangular structure. Then, for every $\gamma\in(0,1)$,
\[
\left\|
(1-\gamma)v_\gamma^*-g_\delta^*\mathbf 1
\right\|_\infty
\le
2(1-\gamma)\operatorname{Span}(u_\delta^*).
\]
\end{lemma}

\begin{proof}
Let
\[
m:=\min_{s\in\mathcal S}v_\gamma^*(s),
\qquad
M:=\max_{s\in\mathcal S}v_\gamma^*(s).
\]
By Lemma~\ref{lem:average-operator-comparison}
and Lemma~\ref{lem:gain-sandwich},
\[
(1-\gamma)m
\le
\min_{s\in\mathcal S}
\left\{\mathcal T^*(v_\gamma^*)(s)-v_\gamma^*(s)\right\}
\le
g_\delta^*
\le
\max_{s\in\mathcal S}
\left\{\mathcal T^*(v_\gamma^*)(s)-v_\gamma^*(s)\right\}
\le
(1-\gamma)M.
\]
Thus,
\[
m
\le
\frac{g_\delta^*}{1-\gamma}
\le M,
\]
and hence
\[
\begin{aligned}
\left\|
(1-\gamma)v_\gamma^*-g_\delta^*\mathbf 1
\right\|_\infty
&=(1-\gamma)
\max_{s\in\mathcal S}
\left|
v_\gamma^*(s)-\frac{g_\delta^*}{1-\gamma}
\right|\\
&\le
(1-\gamma)(M-m)\\
&=(1-\gamma)\operatorname{Span}(v_\gamma^*).
\end{aligned}
\]
Finally, Lemma~\ref{lem:discounted-value-span-bound} gives
$\operatorname{Span}(v_\gamma^*)\le
2\operatorname{Span}(u_\delta^*)$; combining this with the preceding display
yields the claimed bound.
\end{proof}

The following comparison relates the discounted value to a constant-gain
Bellman solution.
\begin{lemma}
\label{lem:bellman-comparison}
Suppose that $(u_\delta^*,g_\delta^*)$ solves the optimal average-reward
Bellman equation~\eqref{eq:constant-gain-bellman}\
with constant gain under either rectangularity structure. Then, for every
$\gamma\in(0,1)$,
\[
\operatorname{Span}(v_\gamma^*)\le2\operatorname{Span}(u_\delta^*),
\qquad
\left\|(1-\gamma)v_\gamma^*-g_\delta^*\mathbf1\right\|_\infty
\le2(1-\gamma)\operatorname{Span}(u_\delta^*).
\]
\end{lemma}
These bounds follow from the comparison argument in the proof of
Theorem~4, Appendix~B, $(2)\Rightarrow(1)$, of
\citet{wang2025bellmanoptimalityaveragerewardrobust}.
We extend that argument to bounded transition-dependent rewards
$r(s,a,s')$ by retaining the reward inside the next-state expectation.

\begin{proof}
The first inequality is the discounted-value span bound proved in
Lemma~\ref{lem:discounted-value-span-bound};
the second is Lemma~\ref{lem:discounted-average-reward-bridge} above.
Both hold for every $\gamma\in(0,1)$ under either rectangularity structure.
\end{proof}

\section{Main Results}
\label{sec:main-results}

\newcommand{\DMDPErrorLemma}{Lemma~\ref{lem:empirical-dr-dmdp-error}}
\newcommand{\AMDPErrorTheorem}{Theorem~\ref{thm:robust-average-reward-error}}
\newcommand{\GoodEventProposition}{Proposition~\ref{prop:good-event-concentration}}
\newcommand{\DiscountedValueLemma}{Lemma~\ref{lem:empirical-discounted-value-perturbation}}
\newcommand{\DiscountedSpanLemma}{Lemma~\ref{lem:discounted-value-span-bound}}
\newcommand{\AverageBridgeLemma}{Lemma~\ref{lem:discounted-average-reward-bridge}}
\newcommand{\MetricConstantGainProposition}{Proposition~\ref{prop:metric-constant-gain}}
\newcommand{\StructuralResultsReference}{Section~\ref{sec:structural-properties-learning-guarantees}}

We first establish an auxiliary discounted-value bound and use it to prove
\AMDPErrorTheorem\ from \StructuralResultsReference. We then prove the
explicit bias-span bounds in
Theorem~\ref{thm:divergence-bias-span-bounds}.
\begin{lemma}
\label{lem:empirical-dr-dmdp-error}
Suppose either Assumption~\ref{ass:divergence-uncertainty} or
Assumption~\ref{ass:metric-uncertainty} holds, and let
$(u_\delta^*,g_\delta^*)$ be any solution of the corresponding optimal average-reward
Bellman equation~\eqref{eq:constant-gain-bellman}\ with constant gain and
$\operatorname{Span}(u_\delta^*)\ge1$. Fix $\beta\in(0,1)$ and
$\gamma\in(0,1)$. If
\[
n
\ge
\frac{16}{p_{\wedge}}
\log\left(\frac{2|\mathcal{S}|^2|\mathcal{A}|}{\beta}\right),
\]
then the empirical optimal discounted value $\widehat v_\gamma^*$
satisfies, with probability at least
$1-\beta$,
\[
\left\|
\widehat v_\gamma^*
-
v_\gamma^*
\right\|_{\infty}
\le
\frac{14}{1-\gamma}
\sqrt{\frac{1}{np_{\wedge}}
\log\left(\frac{2|\mathcal{S}|^2|\mathcal{A}|}{\beta}\right)}
\operatorname{Span}(u_\delta^*).
\]
\end{lemma}

\begin{proof}[Proof of \DMDPErrorLemma]
Let $\eta$ be chosen as in \GoodEventProposition. The sample-size condition
implies
\[
\sqrt{
\frac{1}{np_{\wedge}}
\log\left(\frac{2|\mathcal S|^2|\mathcal A|}{\beta}\right)}
\le\frac14,
\]
and hence
\[
\begin{aligned}
\eta
&\le
\left(\frac1{12}+\sqrt2\right)
\sqrt{
\frac{1}{np_{\wedge}}
\log\left(\frac{2|\mathcal S|^2|\mathcal A|}{\beta}\right)}
<\frac12.
\end{aligned}
\]
On $\Omega_{n,\eta}$, \DiscountedValueLemma\ and \DiscountedSpanLemma\ give
\[
\begin{aligned}
(1-\gamma)
\left\|
\widehat v_\gamma^*-v_\gamma^*
\right\|_\infty
&\le
3\eta\bigl(1+\operatorname{Span}(v_\gamma^*)\bigr)\\
&\le
3\eta\bigl(1+2\operatorname{Span}(u_\delta^*)\bigr)\\
&\le
9\eta
\operatorname{Span}(u_\delta^*)\\
&\le
14
\sqrt{
\frac{1}{np_{\wedge}}
\log\left(\frac{2|\mathcal S|^2|\mathcal A|}{\beta}\right)}
\operatorname{Span}(u_\delta^*).
\end{aligned}
\]
The penultimate inequality uses $\operatorname{Span}(u_\delta^*)\ge1$,
and the last uses $9(1/12+\sqrt2)<14$.
Since $\mathbb P(\Omega_{n,\eta})\ge1-\beta$ by \GoodEventProposition,
dividing by $1-\gamma$ proves the claim.
\end{proof}

\begin{remark}
Lemma~\ref{lem:empirical-dr-dmdp-error} implies that the
empirical optimal discounted value $\widehat v_\gamma^*$ is a uniform
$\epsilon$-error estimate of $v_\gamma^*$ with high probability using
$\widetilde{\mathcal O}\!\left(
|\mathcal S||\mathcal A|p_{\wedge}^{-1}(1-\gamma)^{-2}
\operatorname{Span}^2(u_\delta^*)\epsilon^{-2}
\right)$ samples. The proof separates the optimal-value error into a
contraction term and an operator-perturbation term. The Bellman fixed-point
identities and the triangle inequality give
$\|\widehat v_\gamma^*-v_\gamma^*\|_\infty
\le
\|
\widehat{\mathcal T}_\gamma^*(\widehat v_\gamma^*)
-\widehat{\mathcal T}_\gamma^*(v_\gamma^*)
\|_\infty
+
\|
\widehat{\mathcal T}_\gamma^*(v_\gamma^*)
-\mathcal T_\gamma^*(v_\gamma^*)
\|_\infty$.
The first term is at most
$\gamma\|\widehat v_\gamma^*-v_\gamma^*\|_\infty$ by contraction;
the second measures the operator perturbation at the population value and is
at most $\sup_{\pi\in\Pi_{\mathrm{SR}}}\|\widehat{\mathcal T}_\gamma^\pi(v_\gamma^*)-\mathcal T_\gamma^\pi(v_\gamma^*)\|_\infty$
by $\mathcal T_\gamma^*=\max_{\pi\in\Pi_{\mathrm{SR}}}\mathcal T_\gamma^\pi$ and its empirical counterpart.
This is controlled by
Corollaries~\ref{cor:sa-rec-kl-perturbation}--\ref{cor:s-rec-wasserstein-perturbation}.
Building on the constant-gain Bellman framework of
\citet{wang2025bellmanoptimalityaveragerewardrobust}, we use
$\operatorname{Span}(v_\gamma^*)\le2\operatorname{Span}(u_\delta^*)$
to express the bound in terms of the robust optimal bias span.
The required Bellman structure follows from nominal weak communication and
the stated radius restrictions for divergence-based sets, or from a positive
radius for metric-based sets. This gives the $(1-\gamma)^{-2}$ sample-complexity
dependence and the
$n^{-1/2}$ robust gain-estimation rate in
Theorem~\ref{thm:robust-average-reward-error}.
\end{remark}

\begin{proof}[Proof of \AMDPErrorTheorem]
Let $\eta$ be chosen as in \GoodEventProposition. The sample-size
condition gives $\eta<1/2$, as shown in the proof of \DMDPErrorLemma.
Since $1-\gamma_n=1/\sqrt n$,
\[
\begin{aligned}
\left\|
\frac{1}{\sqrt n}\widehat v_{\gamma_n}^*
-g_\delta^*\mathbf 1
\right\|_\infty
&\le
(1-\gamma_n)
\left\|
\widehat v_{\gamma_n}^*-v_{\gamma_n}^*
\right\|_\infty\\
&\quad+
\left\|
(1-\gamma_n)v_{\gamma_n}^*-g_\delta^*\mathbf 1
\right\|_\infty.
\end{aligned}
\]
By \DMDPErrorLemma\ and \AverageBridgeLemma, on $\Omega_{n,\eta}$,
\[
\begin{aligned}
\left\|
\frac{1}{\sqrt n}\widehat v_{\gamma_n}^*
-g_\delta^*\mathbf 1
\right\|_\infty
&\le
14
\sqrt{
\frac{1}{np_{\wedge}}
\log\left(\frac{2|\mathcal S|^2|\mathcal A|}{\beta}\right)}
\operatorname{Span}(u_\delta^*)\\
&\quad+
\frac{2}{\sqrt n}\operatorname{Span}(u_\delta^*).
\end{aligned}
\]
The assumption $\operatorname{Span}(u_\delta^*)\ge1$ implies
$|\mathcal S|\ge2$, so
\[
\frac{1}{\sqrt n}
\le
\sqrt{
\frac{1}{np_{\wedge}}
\log\left(\frac{2|\mathcal S|^2|\mathcal A|}{\beta}\right)}.
\]
Therefore,
\[
\frac{2}{\sqrt n}\operatorname{Span}(u_\delta^*)
\le
2
\sqrt{
\frac{1}{np_{\wedge}}
\log\left(\frac{2|\mathcal S|^2|\mathcal A|}{\beta}\right)}
\operatorname{Span}(u_\delta^*).
\]
Thus,
\[
\left\|
\frac{1}{\sqrt n}\widehat v_{\gamma_n}^*
-g_\delta^*\mathbf 1
\right\|_\infty
\le
16
\sqrt{
\frac{1}{np_{\wedge}}
\log\left(\frac{2|\mathcal S|^2|\mathcal A|}{\beta}\right)}
\operatorname{Span}(u_\delta^*).
\]
We now bound the robust average reward of the returned policy on the same
event $\Omega_{n,\eta}$. By construction in
Algorithm~\ref{alg:dr-amdp-reduction}, $\widehat\pi^*$ attains the
maximization in $\widehat{\mathcal T}_{\gamma_n}^*$ at
$\widehat v_{\gamma_n}^*$, so
$\widehat{\mathcal T}_{\gamma_n}^{\widehat\pi^*}(\widehat v_{\gamma_n}^*)
=\widehat{\mathcal T}_{\gamma_n}^*(\widehat v_{\gamma_n}^*)
=\widehat v_{\gamma_n}^*$. Since
$\widehat{\mathcal T}_{\gamma_n}^{\widehat\pi^*}$ is a
$\gamma_n$-contraction, its unique fixed point is the robust discounted
value of $\widehat\pi^*$; hence $\widehat\pi^*$ is robust discount-optimal
in the empirical MDP. This is the standard optimality characterization
under SA- and S-rectangularity; see \citet{iyengar2005robust} and
\citet{wiesemann2013robust}, respectively.
Lemma~\ref{lem:empirical-operator-perturbation}
therefore implies, componentwise,
\[
\begin{aligned}
\mathcal T_{\gamma_n}^{\widehat\pi^*}(\widehat v_{\gamma_n}^*)
&\ge
\widehat{\mathcal T}_{\gamma_n}^{\widehat\pi^*}(\widehat v_{\gamma_n}^*)
-3\eta\bigl(1+\operatorname{Span}(\widehat v_{\gamma_n}^*)\bigr)\mathbf 1\\
&=
\widehat v_{\gamma_n}^*
-3\eta\bigl(1+\operatorname{Span}(\widehat v_{\gamma_n}^*)\bigr)\mathbf 1.
\end{aligned}
\]
The bound is simultaneous over policies and value vectors, so it applies
to the sample-dependent $\widehat\pi^*$ and $\widehat v_{\gamma_n}^*$
without an additional probability event.

Fix the samples in $\Omega_{n,\eta}$, any $\mathbf Q\in(\mathcal P)^{\mathbb N}$,
and an initial state $S_0=s$. Since $Q_t\in\mathcal P$ at each time $t$,
the definition of the fixed-policy operator gives
\[
\begin{aligned}
&{\operatorname E}_{\mathbf Q}^{\widehat\pi^*}
\left[r(S_t,A_t,S_{t+1})
+\gamma_n\widehat v_{\gamma_n}^*(S_{t+1})\mid S_t\right]\\
&\qquad\ge
\mathcal T_{\gamma_n}^{\widehat\pi^*}(\widehat v_{\gamma_n}^*)(S_t)\\
&\qquad\ge
\widehat v_{\gamma_n}^*(S_t)
-3\eta\bigl(1+\operatorname{Span}(\widehat v_{\gamma_n}^*)\bigr).
\end{aligned}
\]
Taking expectations and summing over $t=0,\ldots,T-1$, we use the identity
\[
\begin{aligned}
&\sum_{t=0}^{T-1}
\left[\widehat v_{\gamma_n}^*(S_t)
-\gamma_n\widehat v_{\gamma_n}^*(S_{t+1})\right]\\
&\qquad=
\frac1{\sqrt n}\sum_{t=0}^{T-1}\widehat v_{\gamma_n}^*(S_t)
+\gamma_n\left[\widehat v_{\gamma_n}^*(S_0)
-\widehat v_{\gamma_n}^*(S_T)\right].
\end{aligned}
\]
Dividing by $T$, bounding each value by its minimum, and bounding the
terminal difference by minus the span yields
\[
\begin{aligned}
&\frac1T{\operatorname E}_{\mathbf Q}^{\widehat\pi^*}
\left[\sum_{t=0}^{T-1}r(S_t,A_t,S_{t+1})\,\middle|\,S_0=s\right]\\
&\qquad\ge
\frac1{\sqrt n}\min_{x\in\mathcal S}\widehat v_{\gamma_n}^*(x)
-3\eta\bigl(1+\operatorname{Span}(\widehat v_{\gamma_n}^*)\bigr)
-\frac{\gamma_n}{T}\operatorname{Span}(\widehat v_{\gamma_n}^*).
\end{aligned}
\]
For fixed samples, $\widehat v_{\gamma_n}^*$ is finite. Taking the
$\limsup$ as $T\to\infty$ for each $\mathbf Q$, and then taking the infimum
over $\mathbf Q\in(\mathcal P)^{\mathbb N}$, gives
\[
g_\delta^{\widehat\pi^*}(s)
\ge
\frac1{\sqrt n}\min_{x\in\mathcal S}\widehat v_{\gamma_n}^*(x)
-3\eta\bigl(1+\operatorname{Span}(\widehat v_{\gamma_n}^*)\bigr).
\]
This uses the original long-run criterion without interchanging the
infimum and the limit or requiring a constant gain for $\widehat\pi^*$.
By optimality of $g_\delta^*$, adding and subtracting the same estimate gives
\[
\begin{aligned}
0\le g_\delta^*-g_\delta^{\widehat\pi^*}(s)
={}&g_\delta^*-\frac1{\sqrt n}\min_{x\in\mathcal S}\widehat v_{\gamma_n}^*(x)\\
&+\frac1{\sqrt n}\min_{x\in\mathcal S}\widehat v_{\gamma_n}^*(x)
-g_\delta^{\widehat\pi^*}(s)\\
\le{}&
\left\|g_\delta^*\mathbf 1-\frac{\widehat v_{\gamma_n}^*}{\sqrt n}\right\|_\infty
+3\eta\bigl(1+\operatorname{Span}(\widehat v_{\gamma_n}^*)\bigr).
\end{aligned}
\]

By \DiscountedValueLemma, \DiscountedSpanLemma, and
$\operatorname{Span}(u_\delta^*)\ge1$,
\[
\begin{aligned}
\frac1{\sqrt n}
\|\widehat v_{\gamma_n}^*-v_{\gamma_n}^*\|_\infty
&\le3\eta\bigl(1+\operatorname{Span}(v_{\gamma_n}^*)\bigr)\\
&\le3\eta\bigl(1+2\operatorname{Span}(u_\delta^*)\bigr)\\
&\le9\eta\operatorname{Span}(u_\delta^*).
\end{aligned}
\]
Consequently,
\[
\begin{aligned}
\operatorname{Span}(\widehat v_{\gamma_n}^*)
&\le\operatorname{Span}(v_{\gamma_n}^*)
+2\|\widehat v_{\gamma_n}^*-v_{\gamma_n}^*\|_\infty\\
&\le\bigl(2+18\sqrt n\,\eta\bigr)\operatorname{Span}(u_\delta^*).
\end{aligned}
\]
The triangle inequality and \AverageBridgeLemma\ also give
\[
\left\|g_\delta^*\mathbf 1-\frac{\widehat v_{\gamma_n}^*}{\sqrt n}\right\|_\infty
\le
\left(9\eta+\frac2{\sqrt n}\right)\operatorname{Span}(u_\delta^*).
\]
Combining these bounds and using $\operatorname{Span}(u_\delta^*)\ge1$,
we obtain
\[
\begin{aligned}
\|g_\delta^*\mathbf 1-g_\delta^{\widehat\pi^*}\|_\infty
&\le
\left(9\eta+\frac2{\sqrt n}\right)\operatorname{Span}(u_\delta^*)\\
&\quad+
3\eta\left[1+\bigl(2+18\sqrt n\,\eta\bigr)
\operatorname{Span}(u_\delta^*)\right]\\
&\le
\left(\frac2{\sqrt n}+18\eta+54\sqrt n\,\eta^2\right)
\operatorname{Span}(u_\delta^*).
\end{aligned}
\]
For the choice in \GoodEventProposition, the sample-size condition gives
\[
\begin{aligned}
\eta
&=
\frac1{3np_\wedge}
\log\left(\frac{2|\mathcal S|^2|\mathcal A|}{\beta}\right)
+
\sqrt{\frac2{np_\wedge}
\log\left(\frac{2|\mathcal S|^2|\mathcal A|}{\beta}\right)}\\
&\le
\left(\frac1{12}+\sqrt2\right)
\sqrt{\frac1{np_\wedge}
\log\left(\frac{2|\mathcal S|^2|\mathcal A|}{\beta}\right)}\\
&\le
\frac32
\sqrt{\frac1{np_\wedge}
\log\left(\frac{2|\mathcal S|^2|\mathcal A|}{\beta}\right)}
\le\frac38.
\end{aligned}
\]
Substituting this estimate yields
\[
\begin{aligned}
\|g_\delta^*\mathbf 1-g_\delta^{\widehat\pi^*}\|_\infty
&\le
\frac{\operatorname{Span}(u_\delta^*)}{\sqrt n}
\Bigg[
2+27\sqrt{\frac1{p_\wedge}
\log\left(\frac{2|\mathcal S|^2|\mathcal A|}{\beta}\right)}\\
&\hspace{38mm}
+\frac{243}{2p_\wedge}
\log\left(\frac{2|\mathcal S|^2|\mathcal A|}{\beta}\right)
\Bigg].
\end{aligned}
\]
Since $\operatorname{Span}(u_\delta^*)\ge1$ implies $|\mathcal S|\ge2$,
and $p_\wedge\le1$ and $\beta\in(0,1)$,
\[
\frac1{p_\wedge}
\log\left(\frac{2|\mathcal S|^2|\mathcal A|}{\beta}\right)\ge1.
\]
Thus both the constant and square-root terms are bounded by the same
logarithmic term, and $2+27+243/2=301/2\le151$ gives
\[
\|g_\delta^*\mathbf 1-g_\delta^{\widehat\pi^*}\|_\infty
\le
\frac{151\,\operatorname{Span}(u_\delta^*)}{p_\wedge\sqrt n}
\log\left(\frac{2|\mathcal S|^2|\mathcal A|}{\beta}\right).
\]
Both inequalities hold on the same event $\Omega_{n,\eta}$, whose
probability is at least $1-\beta$ by \GoodEventProposition.
This proves the claim.
\end{proof}

Metric-based and divergence-based uncertainty sets require separate analyses
because they have different support properties. A positive-radius TV or
Wasserstein uncertainty set can move probability mass outside the nominal
support. We first prove the metric bounds.

\begin{proof}[Proof of Theorem~\ref{thm:divergence-bias-span-bounds}]
Fix $\gamma\in(0,1)$, write $v=v_\gamma^*$, and let
\[
s_+\in\arg\max_{s\in\mathcal S}v(s),
\qquad
s_-\in\arg\min_{s\in\mathcal S}v(s).
\]
Since $r(s,a,s')\in[0,1]$, the SA-rectangular and $S$-rectangular
operators in~\eqref{eq:discounted-bellman-operator}\ both give
\[
v(s_-)=\mathcal T_\gamma^*(v)(s_-)\ge\gamma v(s_-),
\qquad
(1-\gamma)v(s_-)\ge0.
\]

\textit{Case 1: TV uncertainty sets.}
The uncertainty sets are compact and the Bellman objectives in Definition~\ref{defn:bellman-operators}\ are continuous;
hence the outer maximum and inner infimum are attained. Choose
$a_+\in\mathcal A$ and $\phi_+\in\Delta(\mathcal A)$ attaining the respective
SA-rectangular and $S$-rectangular maxima at $s_+$. Set
\[
q_{s_+,a}
=(1-\delta)p_{s_+,a}+\delta\mathbf 1\{\cdot=s_-\}.
\]
Then, for every $a\in\mathcal A$,
\[
\begin{aligned}
d_{\mathrm{TV}}(q_{s_+,a},p_{s_+,a})
&=\frac{\delta}{2}
\left\|\mathbf 1\{\cdot=s_-\}-p_{s_+,a}\right\|_1\\
&=\delta\bigl(1-p_{s_+,a}(s_-)\bigr)
\le\delta,\\
q_{s_+,a}[v]
&=(1-\delta)p_{s_+,a}[v]+\delta v(s_-)\\
&\le(1-\delta)v(s_+)+\delta v(s_-).
\end{aligned}
\]
Since $q_{s_+,a_+}\in\mathcal P_{s_+,a_+}$, the SA-rectangular reduction in Definition~\ref{defn:bellman-operators}\ gives
\[
\begin{aligned}
v(s_+)
&=\inf_{q_{s_+,a_+}\in\mathcal P_{s_+,a_+}}
q_{s_+,a_+}[r(s_+,a_+,\cdot)+\gamma v]\\
&\le q_{s_+,a_+}[r(s_+,a_+,\cdot)+\gamma v]\\
&\le1+\gamma\bigl((1-\delta)v(s_+)+\delta v(s_-)\bigr).
\end{aligned}
\]
For the $S$-rectangular structure,
\[
\sum_{a\in\mathcal A}
d_{\mathrm{TV}}(q_{s_+,a},p_{s_+,a})
\le|\mathcal A|\delta,
\]
so $(q_{s_+,a})_{a\in\mathcal A}\in\mathcal P_{s_+}$. Therefore, the operator definition~\eqref{eq:discounted-bellman-operator}\ gives
\[
\begin{aligned}
v(s_+)
&=\inf_{q_s\in\mathcal P_{s_+}}
\sum_{a\in\mathcal A}
q_{s,a}\left[\phi_+(a)(r(s_+,a,\cdot)+\gamma v)\right]\\
&\le
\sum_{a\in\mathcal A}
q_{s_+,a}\left[\phi_+(a)(r(s_+,a,\cdot)+\gamma v)\right]\\
&\le1+\gamma\sum_{a\in\mathcal A}\phi_+(a)q_{s_+,a}[v]\\
&\le1+\gamma\bigl((1-\delta)v(s_+)+\delta v(s_-)\bigr).
\end{aligned}
\]
Thus, under either structure,
\[
\begin{aligned}
\operatorname{Span}(v_\gamma^*)
&=v(s_+)-v(s_-)\\
&\le1+\gamma(1-\delta)v(s_+)+(\gamma\delta-1)v(s_-)\\
&=1+\gamma(1-\delta)\operatorname{Span}(v_\gamma^*)
-(1-\gamma)v(s_-)\\
&\le1+\gamma(1-\delta)\operatorname{Span}(v_\gamma^*),
\end{aligned}
\]
and hence
\[
\operatorname{Span}(v_\gamma^*)
\le\frac{1}{1-\gamma(1-\delta)}
\le\frac{1}{\delta}.
\]

\textit{Case 2: Wasserstein uncertainty sets.}
Again, the uncertainty sets are compact and the Bellman objectives in Definition~\ref{defn:bellman-operators}\ are
continuous, so the outer maximum and inner infimum are attained. Choose
$a_+$ and $\phi_+$ attaining the respective maxima at $s_+$. Let
\[
\alpha=\left(\frac{c}{\delta}\right)^{-\ell},
\qquad
q_{s_+,a}
=(1-\alpha)p_{s_+,a}+\alpha\mathbf 1\{\cdot=s_-\}.
\]
Coupling the mass $(1-\alpha)p_{s_+,a}$ with itself and transporting the
remaining mass to $s_-$ gives
\[
W_\ell(q_{s_+,a},p_{s_+,a})^\ell
\le\alpha c^\ell=\delta^\ell,
\qquad
q_{s_+,a}[v]
\le(1-\alpha)v(s_+)+\alpha v(s_-).
\]
Thus $q_{s_+,a_+}\in\mathcal P_{s_+,a_+}$ in the SA-rectangular case and
$(q_{s_+,a})_{a\in\mathcal A}\in\mathcal P_{s_+}$ in the $S$-rectangular
case, since
\[
\sum_{a\in\mathcal A}
W_\ell(q_{s_+,a},p_{s_+,a})^\ell
\le|\mathcal A|\delta^\ell.
\]
The same SA- and $S$-rectangular Bellman arguments as in Case 1, with
$\alpha$ in place of $\delta$, give
\[
v(s_+)\le1+\gamma\bigl((1-\alpha)v(s_+)+\alpha v(s_-)\bigr).
\]
Consequently,
\[
\operatorname{Span}(v_\gamma^*)
\le1+\gamma(1-\alpha)\operatorname{Span}(v_\gamma^*).
\]
Hence
\[
\operatorname{Span}(v_\gamma^*)
\le\frac{1}{1-\gamma(1-\alpha)}
\le\left(\frac{c}{\delta}\right)^\ell.
\]

Applying the normalization and convergent-subsequence argument from the proof
of \MetricConstantGainProposition\ to these uniform bounds, the bias can be
chosen so that
\[
\operatorname{Span}(u_\delta^*)\le\frac{1}{\delta}
\quad\text{for TV},
\qquad
\operatorname{Span}(u_\delta^*)
\le\left(\frac{c}{\delta}\right)^\ell
\quad\text{for Wasserstein}.
\]

\textit{Case 3: KL and $f_k$ uncertainty sets.}
For SA-rectangular uncertainty, the coordinatewise arguments of
\citet[Proposition~B.1]{chen2025sample} for KL divergence and
\citet[Lemma~E.1]{chen2025sample} for $f_k$ divergence
show that the stated radius conditions imply
\[
p_{s,a}(s')>0
\quad\Longrightarrow\quad
q_{s,a}(s')
\ge\frac12p_{s,a}(s')
\ge\frac12p_{\wedge}.
\]
For $S$-rectangular uncertainty, each nonnegative action-wise divergence is
at most $|\mathcal A|\delta$, so the same argument applies under the
corresponding stated radius. Since $q_{s,a}\ll p_{s,a}$, every feasible row
has the same support as its nominal row. Thus every nominal communication
path is retained.

Fix an arbitrary solution $(u_\delta^*,g_\delta^*)$\ of \eqref{eq:constant-gain-bellman}. For SA-rectangular
uncertainty, choose a minimizing row $q_{s,a}$ in the SA-rectangular reduction in Definition~\ref{defn:bellman-operators}\ for every $(s,a)$. For
$S$-rectangular uncertainty, compactness and
Lemma~\ref{lem:inf-sup-to-optimal-bellman}
give, for every $s\in\mathcal S$,
\[
q_s\in\arg\min_{q_s\in\mathcal P_s}
\max_{\phi\in\Delta(\mathcal A)}
\sum_{a\in\mathcal A}
q_{s,a}\!\left[
\phi(a)\bigl(r(s,a,\cdot)+u_\delta^*\bigr)
\right].
\]
Since the objective is linear in $\phi$,
\[
\max_{\phi\in\Delta(\mathcal A)}
\sum_{a\in\mathcal A}
q_{s,a}\!\left[
\phi(a)\bigl(r(s,a,\cdot)+u_\delta^*\bigr)
\right]
=
\max_{a\in\mathcal A}
q_{s,a}\!\left[r(s,a,\cdot)+u_\delta^*\right].
\]
Thus, by the average-reward operator definition~\eqref{eq:average-bellman-operator}, under either structure, the worst-case transition kernel
$Q\in\mathcal P$ satisfies
\[
u_\delta^*(s)+g_\delta^*
=
\max_{a\in\mathcal A}
q_{s,a}\!\left[r(s,a,\cdot)+u_\delta^*\right],
\qquad s\in\mathcal S.
\]
By
Lemma~\ref{prop:adversarial-weak-communication},
$Q$ is weakly communicating; let $C$ denote its communicating set. Every
positive coordinate of $Q$ is at least $p_{\wedge}/2$.

Let
\[
s^+\in\arg\max_{s\in\mathcal S}u_\delta^*(s),
\qquad
s^-\in\arg\min_{s\in\mathcal S}u_\delta^*(s).
\]
The fixed-kernel Bellman equation gives
\[
u_\delta^*(s^-)+g_\delta^*
\ge u_\delta^*(s^-),
\qquad
u_\delta^*(s^+)+g_\delta^*
\le u_\delta^*(s^+)+1,
\]
so $0\le g_\delta^*\le1$.

Let $y\in C$. Weak communication allows us to choose $\pi\in\Pi_{\mathrm{SD}}$ that puts
probability one on one action at each state and reaches $y$ almost surely
from every state: choosing actions along shortest positive-probability paths
to $y$ leaves no closed class outside $y$. Define
\[
\tau_y:=\inf\{t\ge0:S_t=y\},
\qquad
M:=\max_{s\in\mathcal S}\mathbb E_s^{\pi,Q}[\tau_y].
\]
Choose $s_0$ attaining $M$ and a simple positive-probability path under
$\pi$,
\[
s_0,s_1,\ldots,s_m=y,
\qquad m\le|\mathcal S|-1.
\]
For $0\le i<m$, take $a\in\mathcal A$ with $\pi(a\mid s_i)=1$. Then
\[
\begin{aligned}
\mathbb E_{s_i}^{\pi,Q}[\tau_y]
&\le
1+q_{s_i,a}(s_{i+1})
\mathbb E_{s_{i+1}}^{\pi,Q}[\tau_y]\\
&\quad+
\bigl(1-q_{s_i,a}(s_{i+1})\bigr)M,
\end{aligned}
\]
where $q_{s_i,a}(s_{i+1})\ge p_{\wedge}/2$. Rearranging the case
$i=0$ and applying induction along the path gives
\[
M
\le
\sum_{j=1}^{i}
\left(\frac{2}{p_{\wedge}}\right)^j
+\mathbb E_{s_i}^{\pi,Q}[\tau_y],
\qquad 1\le i\le m.
\]
Setting $i=m$ yields
\[
\begin{aligned}
\max_{s\in\mathcal S}\mathbb E_s^{\pi,Q}[\tau_y]
&\le
\sum_{j=1}^{|\mathcal S|-1}
\left(\frac{2}{p_{\wedge}}\right)^j\\
&=
\frac{2}{2-p_{\wedge}}
\left[
\left(\frac{2}{p_{\wedge}}\right)^{|\mathcal S|-1}-1
\right].
\end{aligned}
\]
The same argument bounds the first entrance time of $C$ under any stationary
deterministic policy, since states outside $C$ are transient under every such
policy; replace the terminal state by the target set $C$.

Let
\[
y^+\in\arg\max_{y\in C}u_\delta^*(y),
\qquad
\tau_C:=\inf\{t\ge0:S_t\in C\},
\qquad
\tau_{y^+}:=\inf\{t\ge0:S_t=y^+\}.
\]
Let $\pi^*\in\Pi_{\mathrm{SD}}$ put probability one on a maximizing action in the
fixed-kernel Bellman equation at each state, and choose $\pi\in\Pi_{\mathrm{SD}}$ as
above with target $y^+$. Applying the Bellman
equality under $\pi^*$ up to $\tau_C$ and the Bellman inequality under $\pi$
up to $\tau_{y^+}$ gives
\[
\begin{aligned}
u_\delta^*(s^+)
&=
\mathbb E_{s^+}^{\pi^*,Q}
\left[
\sum_{t=0}^{\tau_C-1}
\bigl(r(S_t,A_t,S_{t+1})-g_\delta^*\bigr)
+u_\delta^*(S_{\tau_C})
\right]\\
&\le
u_\delta^*(y^+)
+(1-g_\delta^*)
\mathbb E_{s^+}^{\pi^*,Q}[\tau_C],\\
u_\delta^*(s^-)
&\ge
\mathbb E_{s^-}^{\pi,Q}
\left[
\sum_{t=0}^{\tau_{y^+}-1}
\bigl(r(S_t,A_t,S_{t+1})-g_\delta^*\bigr)
+u_\delta^*(y^+)
\right]\\
&\ge
u_\delta^*(y^+)
-g_\delta^*
\mathbb E_{s^-}^{\pi,Q}[\tau_{y^+}].
\end{aligned}
\]
The stopped relations follow by truncation, since both hitting times have
finite expectation. Therefore,
\[
\begin{aligned}
\operatorname{Span}(u_\delta^*)
&\le
(1-g_\delta^*)
\mathbb E_{s^+}^{\pi^*,Q}[\tau_C]
+g_\delta^*
\mathbb E_{s^-}^{\pi,Q}[\tau_{y^+}]\\
&\le
\frac{2}{2-p_{\wedge}}
\left[
\left(\frac{2}{p_{\wedge}}\right)^{|\mathcal S|-1}-1
\right].
\end{aligned}
\]
\end{proof}

\begin{remark}
\label{remark:attained}
The TV and Wasserstein span bounds are attained. Let
$\mathcal S=\{s_+,s_-\}$, $\mathcal A=\{a\}$, and
$r(s_+,a,\cdot)=1$, $r(s_-,a,\cdot)=0$. For TV uncertainty, take
\[
p_{s_+,a}=\mathbf 1\{\cdot=s_+\},
\qquad
p_{s_-,a}=\delta\mathbf 1\{\cdot=s_+\}
+(1-\delta)\mathbf 1\{\cdot=s_-\}.
\]
At radius $\delta$, the minimizing kernels are
\[
q_{s_+,a}=(1-\delta)\mathbf 1\{\cdot=s_+\}
+\delta\mathbf 1\{\cdot=s_-\},
\qquad
q_{s_-,a}=\mathbf 1\{\cdot=s_-\},
\]
which yield $g_\delta^*=0$, $u_\delta^*(s_-)=0$, and
$u_\delta^*(s_+)=1/\delta$.

For Wasserstein uncertainty, let
$\rho(s_+,s_-)=c$, $\alpha=(\delta/c)^\ell$, and take
\[
p_{s_+,a}=\mathbf 1\{\cdot=s_+\},
\qquad
p_{s_-,a}=\alpha\mathbf 1\{\cdot=s_+\}
+(1-\alpha)\mathbf 1\{\cdot=s_-\}.
\]
The minimizing kernels are
\[
q_{s_+,a}=(1-\alpha)\mathbf 1\{\cdot=s_+\}
+\alpha\mathbf 1\{\cdot=s_-\},
\qquad
q_{s_-,a}=\mathbf 1\{\cdot=s_-\},
\]
and satisfy $W_\ell(q_{s,a},p_{s,a})^\ell=\alpha c^\ell=\delta^\ell$.
Thus $g_\delta^*=0$, $u_\delta^*(s_-)=0$, and
\[
u_\delta^*(s_+)=\frac1\alpha
=\left(\frac c\delta\right)^\ell.
\]
\end{remark}

\begin{proof}[Proof of Theorem~\ref{thm:divergence-bias-span-lower-bound}]
Fix
$0<p_{\wedge}\le\frac12$ and let
\[
\mathcal S=\{S_1,S_2,\ldots,S_{|\mathcal S|}\},
\qquad
\mathcal A=\{a\},
\]
and define
\[
p_{S_1,a}=\mathbf1\{\cdot=S_1\},
\qquad
p_{S_i,a}
=p_{\wedge}\mathbf1\{\cdot=S_{i-1}\}
+(1-p_{\wedge})\mathbf1\{\cdot=S_{|\mathcal S|}\},
\quad 2\le i\le|\mathcal S|,
\]
with
\[
r(S_1,a,\cdot)=1,
\qquad
r(S_i,a,\cdot)=0,
\quad 2\le i\le|\mathcal S|.
\]
The unique communicating class is $\{S_1\}$, while
$S_2,\ldots,S_{|\mathcal S|}$ are transient under the unique policy.
With one action, the SA- and $S$-rectangular uncertainty sets coincide.

For every sufficiently small positive $\delta$, choose each $q_{S_i,a}$ to
minimize the feasible probability assigned to $S_{i-1}$, for
$2\le i\le|\mathcal S|$. The nominal rows have the same two probabilities,
and KL and $f_k$ divergences are invariant under relabeling, so their minimum
feasible predecessor probabilities agree. The support bound from Case 3 and
feasibility of the nominal rows give
\[
\frac12p_{\wedge}
\le q_{S_2,a}(S_1)
\le p_{\wedge},
\qquad
q_{S_i,a}(S_{i-1})=q_{S_2,a}(S_1),
\quad 2\le i\le|\mathcal S|.
\]

Normalize $u_\delta^*(S_{|\mathcal S|})=0$. The constant-gain equation~\eqref{eq:constant-gain-bellman}\ at $S_1$ gives
$g_\delta^*=1$. Starting from $i=|\mathcal S|$ and proceeding backward,
$u_\delta^*(S_i)\ge0$ implies
\[
1+u_\delta^*(S_i)
=
\inf_{q_{S_i,a}\in\mathcal P_{S_i,a}}
q_{S_i,a}(S_{i-1})u_\delta^*(S_{i-1})>0.
\]
Thus $u_\delta^*(S_{i-1})>0$, and the infimum is attained by the
chosen $q_{S_i,a}$. Hence
\[
u_\delta^*(S_{i-1})
=
\frac{1+u_\delta^*(S_i)}{q_{S_i,a}(S_{i-1})},
\qquad 2\le i\le|\mathcal S|.
\]
Consequently,
\[
u_\delta^*(S_i)
=
\sum_{j=1}^{|\mathcal S|-i}
\frac{1}{\bigl(q_{S_2,a}(S_1)\bigr)^j},
\qquad 1\le i\le|\mathcal S|-1.
\]
The constant-gain equation~\eqref{eq:constant-gain-bellman}\ at $S_1$ uniquely determines $g_\delta^*$, and the displayed
backward recursion uniquely determines the normalized bias. Thus all
solutions differ only by an additive constant and have the same span.
Therefore,
\[
\begin{aligned}
\operatorname{Span}(u_\delta^*)
&=
\sum_{j=1}^{|\mathcal S|-1}
\frac{1}{\bigl(q_{S_2,a}(S_1)\bigr)^j}\\
&\ge
\sum_{j=1}^{|\mathcal S|-1}
\left(\frac{1}{p_{\wedge}}\right)^j\\
&=
\frac{p_{\wedge}^{-(|\mathcal S|-1)}-1}{1-p_{\wedge}}.
\end{aligned}
\]
\end{proof}

\graphicspath{{manuscript/figures/}{figures/}{../figures/}}
\section{Additional Experiments}
\label{sec:additional-experiments}

This appendix presents additional large-scale experiments on a fixed nominal MDP with $20$ states and $30$ actions. We evaluate all eight combinations of SA- and S-rectangularity with KL-divergence, $\chi^2$-divergence, TV, and Wasserstein uncertainty sets.

We construct a random nominal kernel $P$ by partitioning the states into $T=\{0,1,2,3\}$ and $C=C_1\cup C_2$, where $C_1=\{4,\ldots,11\}$ and $C_2=\{12,\ldots,19\}$. At every state in $T$, each action assigns total probability $1/2$ to $T$ and $1/2$ to $C$. At a state in $C_i$, actions $a_1,\ldots,a_{15}$ remain in $C_i$, while actions $a_{16},\ldots,a_{30}$ can reach every state in $C$. Each conditional distribution on an allowed set of $m$ states is drawn independently as $0.8z+0.2\mathbf 1/m$, with $z\sim\mathrm{Dirichlet}(1,\ldots,1)$. All other transition probabilities are zero. We generate rewards as $\sigma_{s,a}Z_{s,a,s'}$, where $\sigma_{s,a}\sim\mathrm{Uniform}(0,1)$ and $Z_{s,a,s'}\sim\mathcal N(0,1)$, and rescale the full reward array to $[0,1]$. The nominal kernel and rewards remain fixed across all trials.

Starting in $T$, the probability of remaining there after $t$ transitions is $2^{-t}$ under every policy, and no action returns from $C$ to $T$. Thus every state in $T$ is transient under every policy. The stationary policy that always chooses $a_{16}$ makes all states in $C$ communicate, so the nominal MDP is weakly communicating in the sense of Definition~\ref{def:weakly-communicating-mdp}. The stationary policy that always chooses $a_1$ has two disjoint recurrent classes, $C_1$ and $C_2$, so the nominal MDP is not unichain. Its induced Markov chain cannot converge to a single stationary distribution independent of the initial state. Since $P\in\mathcal P$ for each of the eight uncertainty sets, all eight configurations violate unichain and uniform ergodicity assumptions over all feasible kernels and stationary policies.

\begin{figure}[!htbp]
\centering
\begin{subfigure}[t]{0.49\linewidth}
\centering
\includegraphics[width=\linewidth]{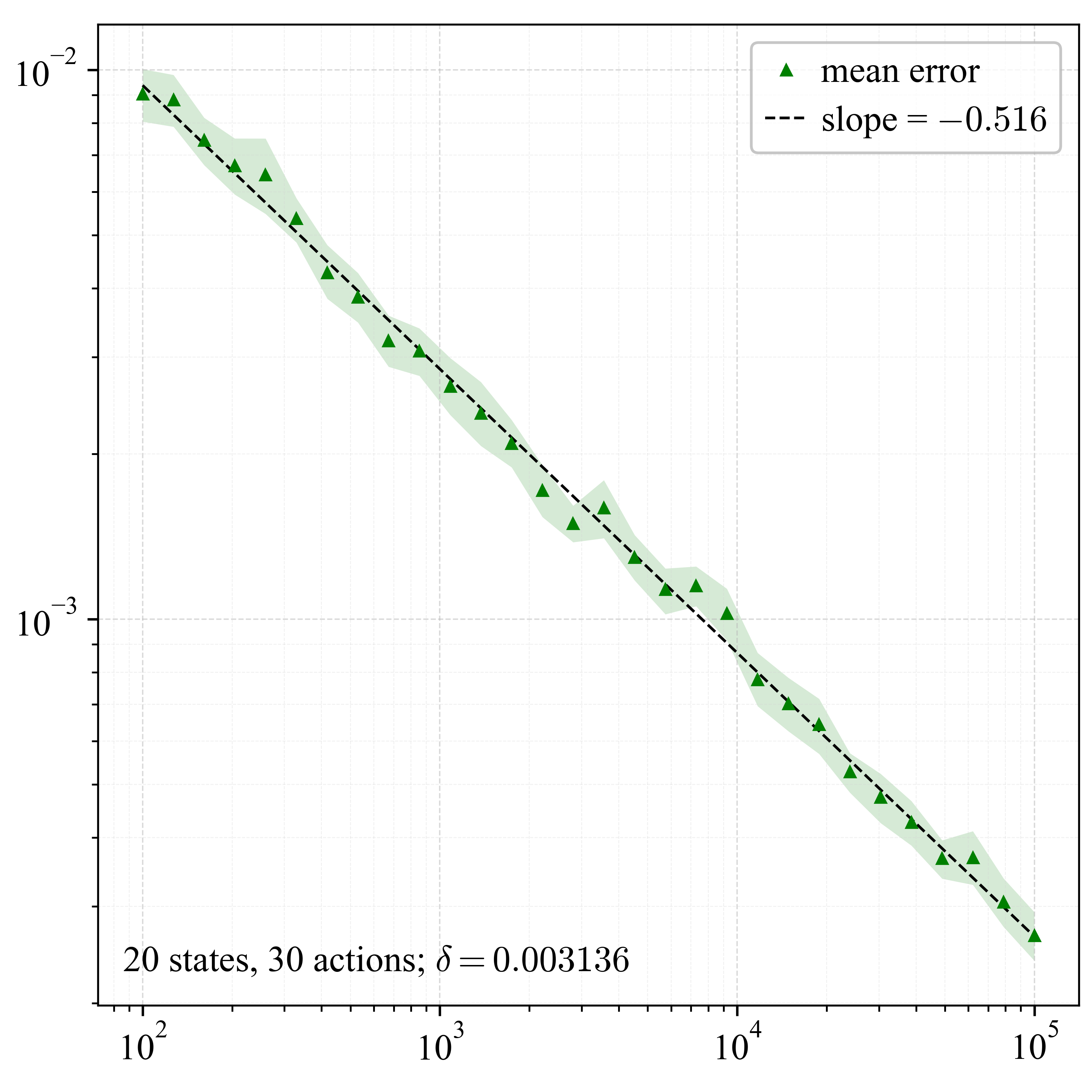}
\caption{SA-rectangular KL-divergence uncertainty.}
\end{subfigure}\hfill
\begin{subfigure}[t]{0.49\linewidth}
\centering
\includegraphics[width=\linewidth]{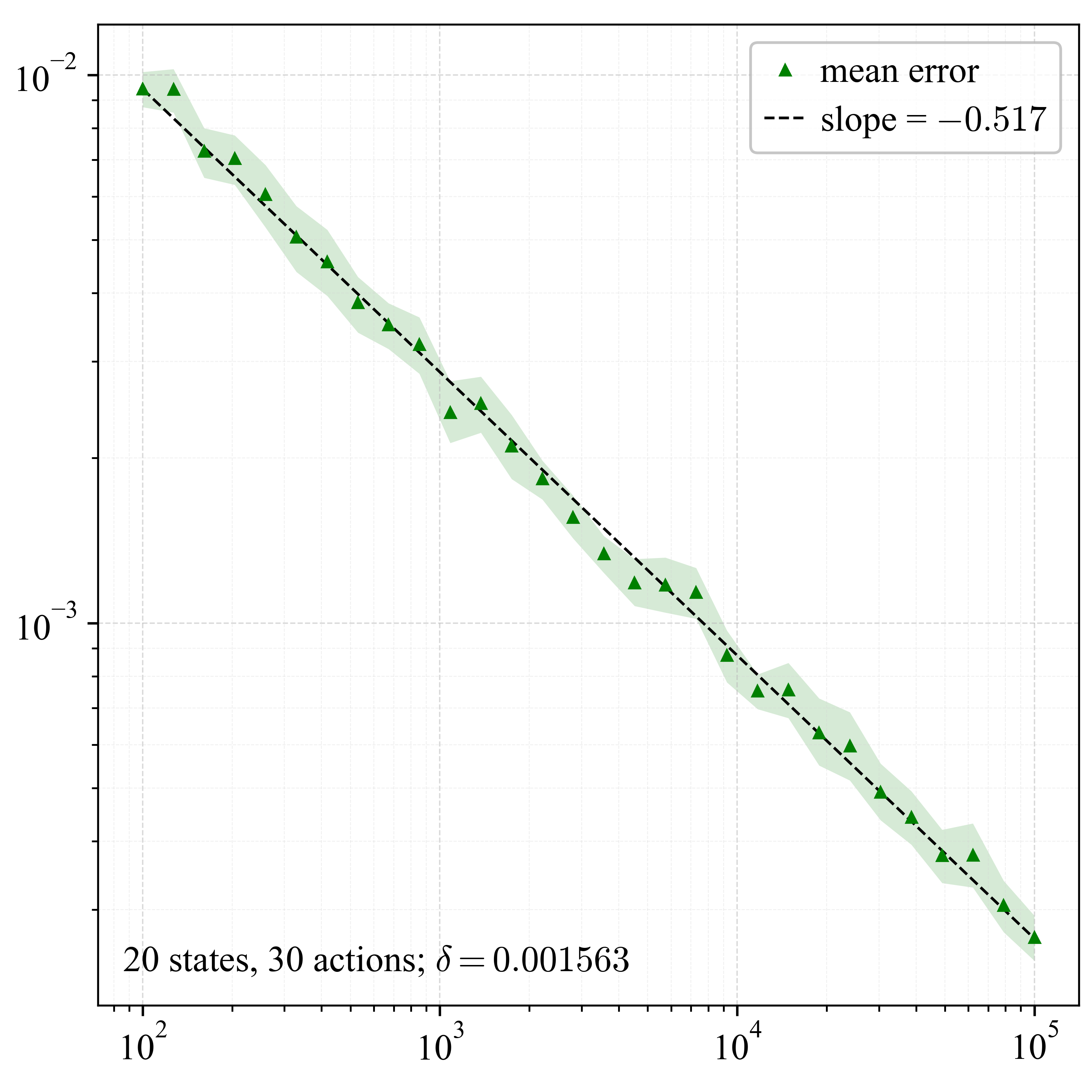}
\caption{SA-rectangular $\chi^2$-divergence uncertainty.}
\end{subfigure}

\begin{subfigure}[t]{0.49\linewidth}
\centering
\includegraphics[width=\linewidth]{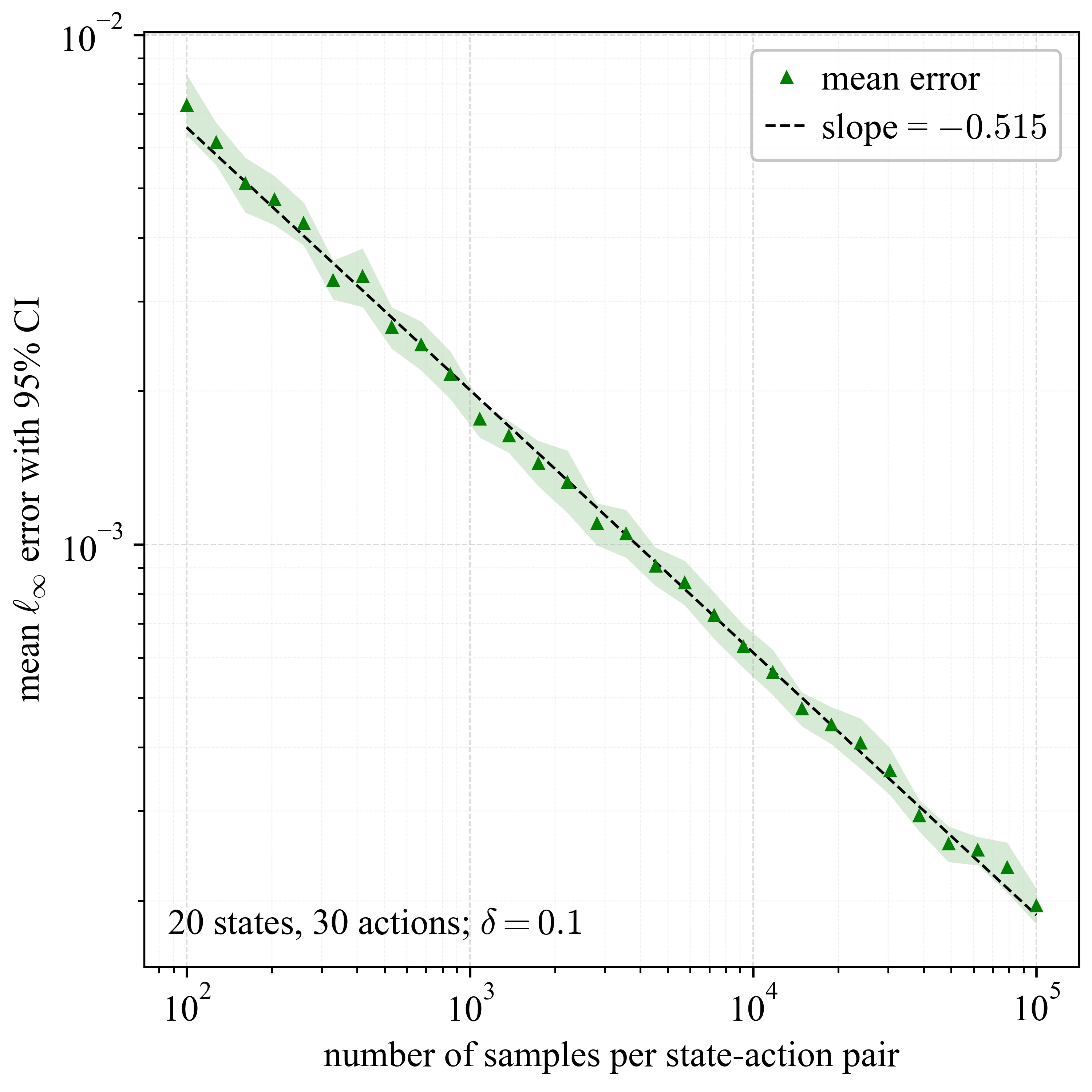}
\caption{SA-rectangular TV uncertainty.}
\end{subfigure}\hfill
\begin{subfigure}[t]{0.49\linewidth}
\centering
\includegraphics[width=\linewidth]{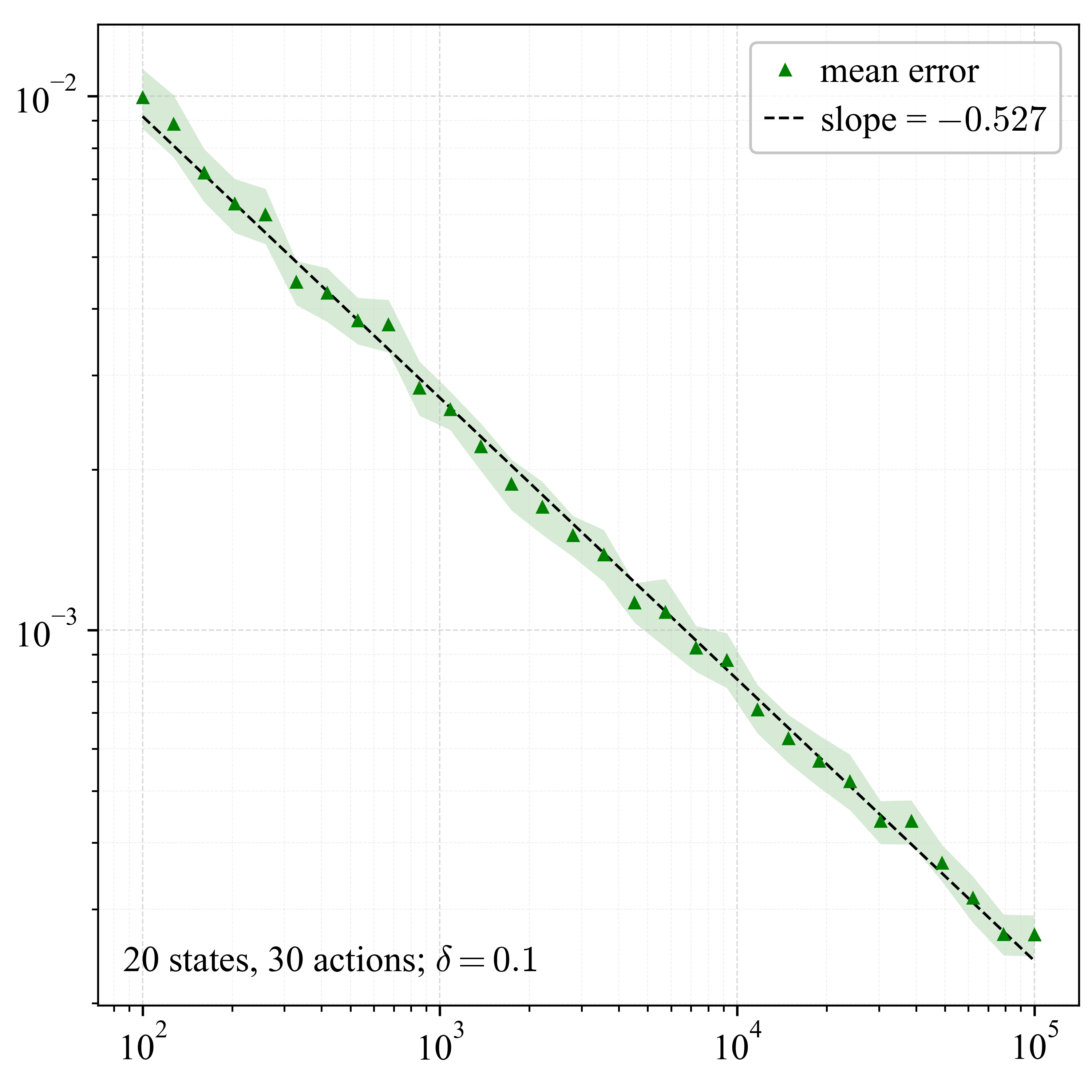}
\caption{SA-rectangular Wasserstein uncertainty.}
\end{subfigure}
\caption{Large-scale experiments under SA-rectangular uncertainty. Each panel plots the mean $\ell_\infty$ gain-estimation error against the number of samples per state-action pair. Shaded bands indicate 95\% confidence intervals, and dashed lines show fitted slopes on logarithmic axes.}
\label{fig:large-scale-sa-rectangular}
\end{figure}

For KL and $\chi^2$ uncertainty, we set $\delta$ to half the corresponding upper bound in Assumption~\ref{ass:divergence-uncertainty}, using the minimum positive transition probability of $P$. These radii preserve the nominal transition supports. For TV and Wasserstein uncertainty, we use $\delta=0.1$ and allow transitions outside the nominal support. The Wasserstein order is $\ell=2$, with ground metric $\rho(s,s')=\mathbbm{1}\{s\ne s'\}$.

For each of $30$ sample sizes $n$ from $10^2$ to $10^5$, we independently draw $n$ next-state samples from each $p_{s,a}$ and construct $\widehat p_{s,a}$ from empirical frequencies, without smoothing. We repeat this procedure over $20$ independent trials. With $\gamma_n=1-n^{-1/2}$, we report the mean of $\|(1-\gamma_n)\widehat v_{\gamma_n}^*-g_\delta^*\mathbf 1\|_\infty$ across trials, with percentile-bootstrap 95\% confidence intervals. Figures~\ref{fig:large-scale-sa-rectangular} and~\ref{fig:large-scale-s-rectangular} show the results. The fitted slopes, computed using all $30$ sample sizes, range from $-0.527$ to $-0.490$, consistent with an error decay of approximately $n^{-1/2}$.

\begin{figure}[!htbp]
\centering
\begin{subfigure}[t]{0.49\linewidth}
\centering
\includegraphics[width=\linewidth]{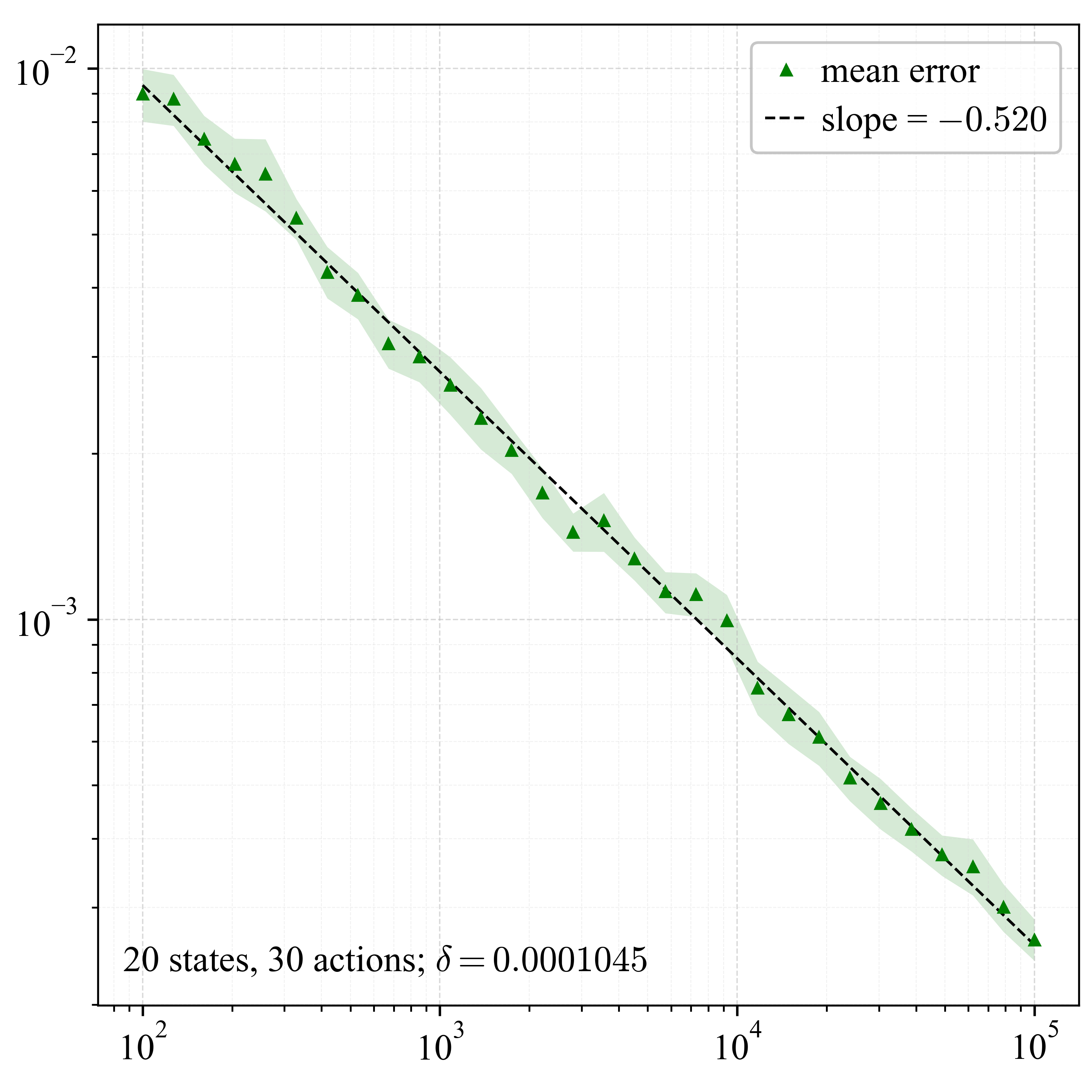}
\caption{S-rectangular KL-divergence uncertainty.}
\end{subfigure}\hfill
\begin{subfigure}[t]{0.49\linewidth}
\centering
\includegraphics[width=\linewidth]{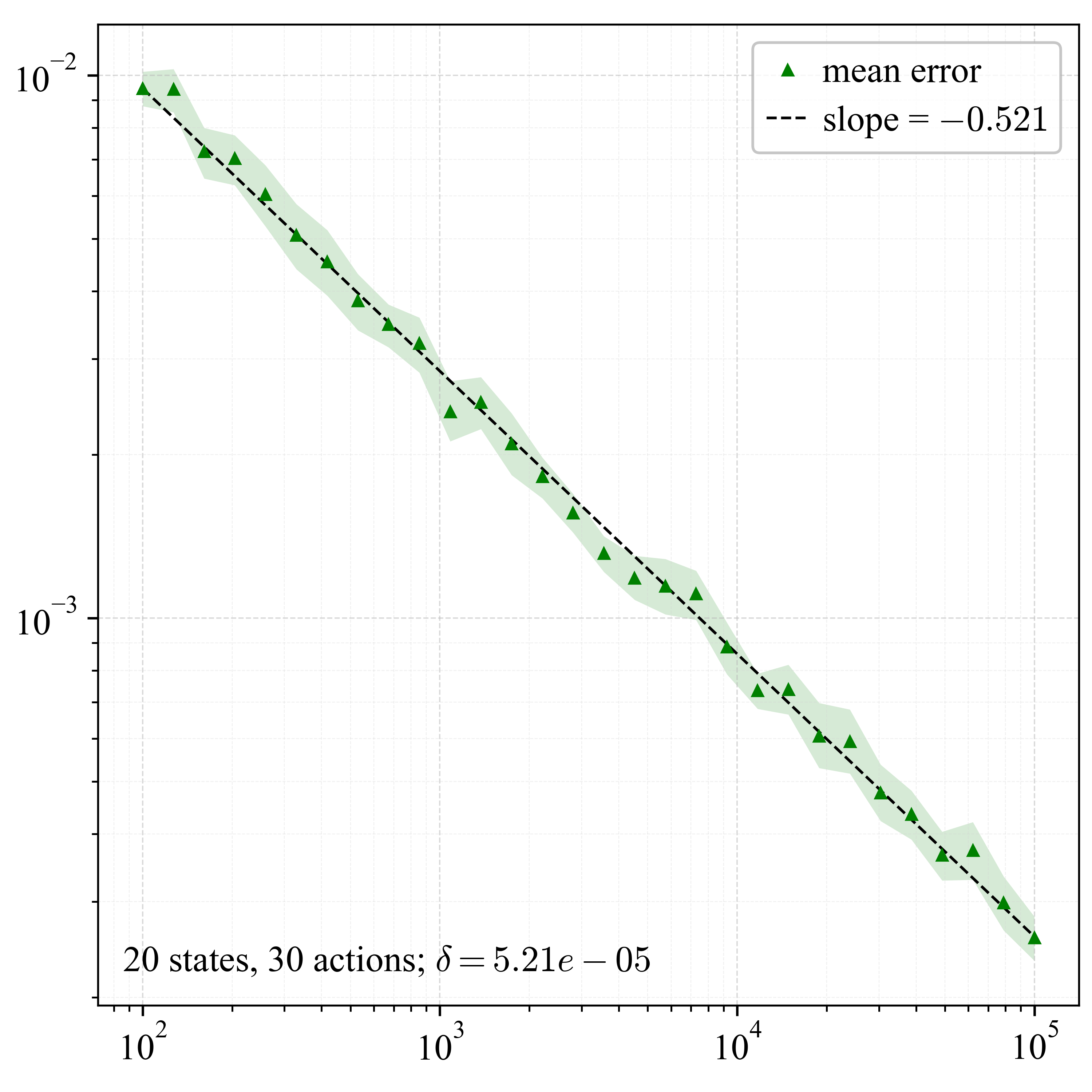}
\caption{S-rectangular $\chi^2$-divergence uncertainty.}
\end{subfigure}

\begin{subfigure}[t]{0.49\linewidth}
\centering
\includegraphics[width=\linewidth]{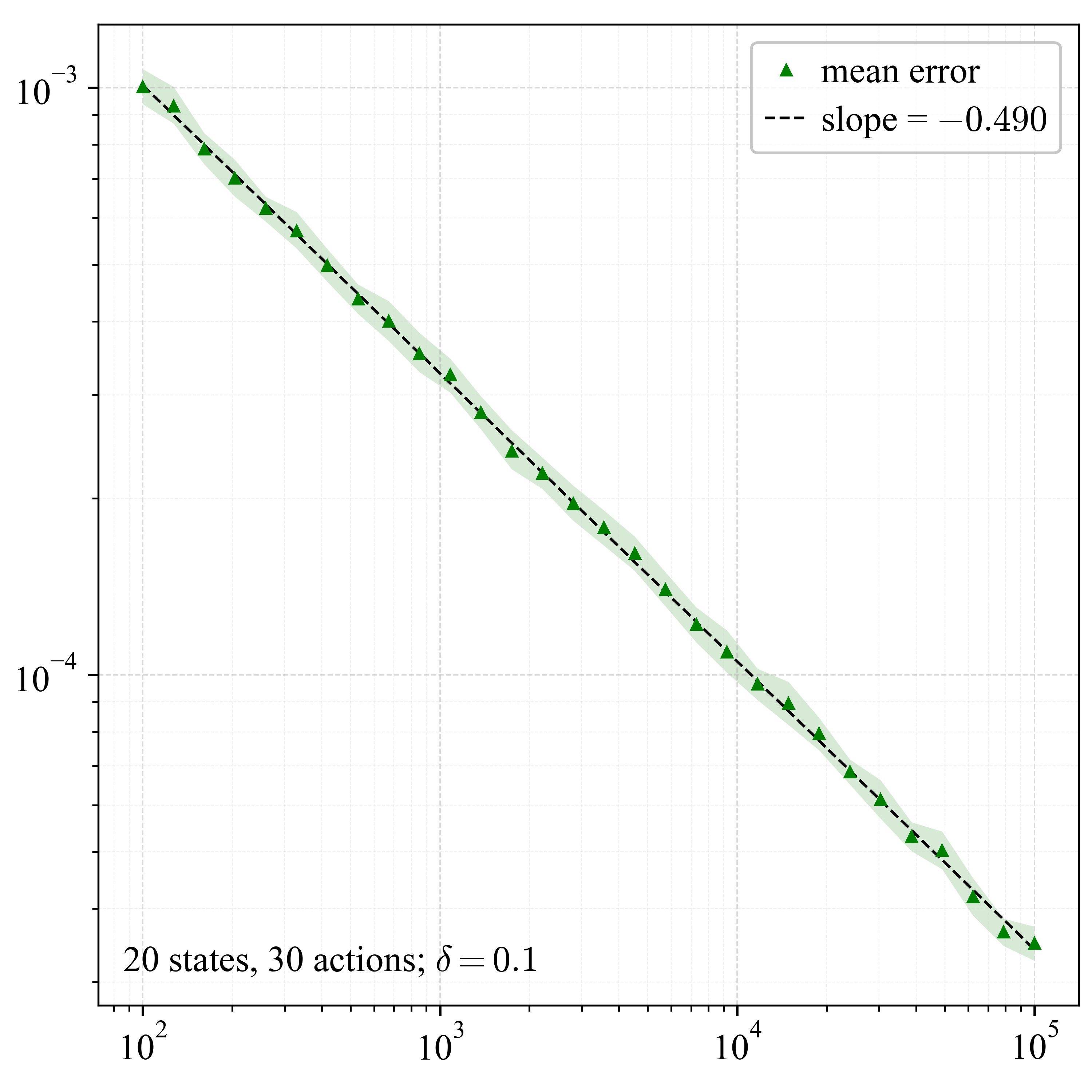}
\caption{S-rectangular TV uncertainty.}
\end{subfigure}\hfill
\begin{subfigure}[t]{0.49\linewidth}
\centering
\includegraphics[width=\linewidth]{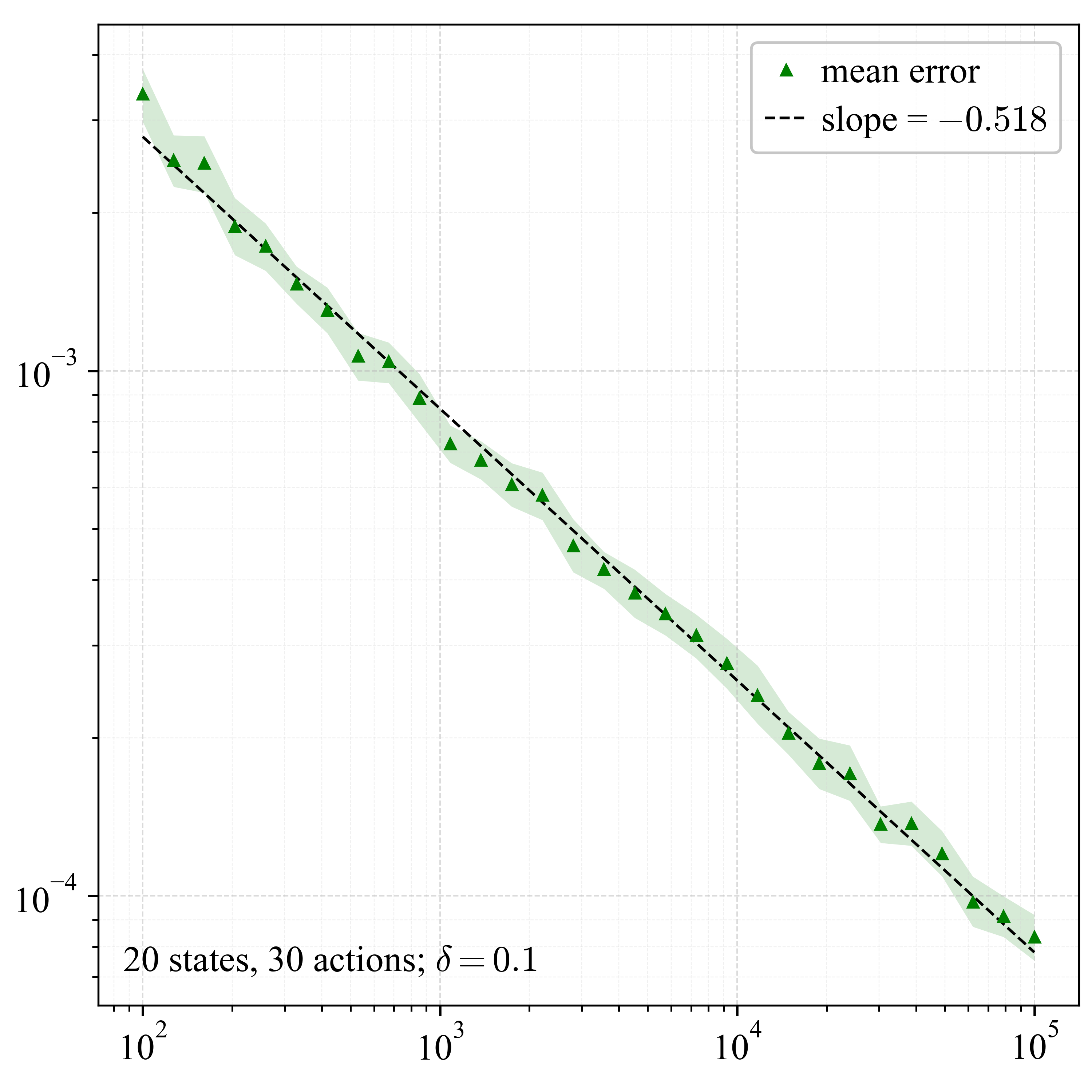}
\caption{S-rectangular Wasserstein uncertainty.}
\end{subfigure}
\caption{Large-scale experiments under S-rectangular uncertainty, with the same error measure and plotting conventions as Figure~\ref{fig:large-scale-sa-rectangular}.}
\label{fig:large-scale-s-rectangular}
\end{figure}

\end{document}